\documentclass[conference,compsocconf,10pt]{IEEEtran}
\IEEEoverridecommandlockouts
\usepackage{cite}
\usepackage{amsmath,amssymb,amsfonts}
\usepackage{amsbsy}
\usepackage{mathbbol}
\usepackage{algorithmic}
\usepackage{graphicx}
\usepackage{textcomp}
\usepackage{xcolor}
\usepackage{tcolorbox}
\usepackage{pbox}
\usepackage{url}
\usepackage{hyperref}
\hypersetup{colorlinks=true,urlcolor=blue,citecolor=teal,linkcolor=violet}
\usepackage{balance} 
\usepackage{subcaption} 

\def\BibTeX{{\rm B\kern-.05em{\sc i\kern-.025em b}\kern-.08em
    T\kern-.1667em\lower.7ex\hbox{E}\kern-.125emX}}

\usepackage{amsthm}   % for proofs
\usepackage{thm-restate}
\newtheorem{definition}{Definition}

\newtheorem{example}{Example}

\newcommand{\changed}[1]{{#1}}

\newcommand{\fix}[1]{\textcolor{red}{#1}}

\newcommand{\mods}[1]{{#1}}
\newcommand{\hide}[1]{{}}

\newcommand{\calx}{{\cal X}}
\newcommand{\caly}{{\cal Y}}
\newcommand{\calz}{{\cal Z}}
\newcommand{\calw}{{\cal W}}
\newcommand{\call}{{\cal L}}
\newcommand{\expect}{{\mathbb{E}}}

\newcommand{\loss}{\mathbb{L}}
\newcommand{\CE}{\mathit{CE}}

\DeclareMathOperator*{\argmax}{argmax}
\DeclareMathOperator*{\argmin}{argmin}

\usepackage{mathtools}
\usepackage{blkarray, bigstrut}
\usepackage{multirow}
\usepackage{rotating}
\usepackage{dcolumn}
\newcolumntype{L}{@{\hspace{1.5em}}D{.}{.}{1.1}@{\hspace{1.5em}}}
\newcolumntype{V}{>{\centering\arraybackslash}m{.1em}}
\newcolumntype{N}{@{}m{0pt}@{}}
\usepackage{tikz}
\usepackage{comment}

\usepackage{algorithmic}
\usepackage[]{algorithm2e}
\SetAlgoCaptionLayout{centerline}%
\SetKwInput{KwModels}{Models}
\usepackage[section]{placeins}

\usepackage{diagbox}
\usepackage{xcolor,colortbl}
\usepackage{hhline}
\definecolor{lightblue}{rgb}{0.8, 0.9, 1}
\definecolor{lightpurple}{rgb}{0.95, 0.9, 1}
\newcolumntype{a}{>{\columncolor{lightblue}}c}
\definecolor{lightgray}{rgb}{0.85, 0.9, 0.9}
\newcolumntype{b}{>{\columncolor{lightgray}}c}

\usepackage{subcaption}
\usepackage{nicefrac}
\usepackage{stmaryrd}
\usepackage{booktabs}
\usepackage{balance}
\usepackage{cleveref}

\providecommand{\UsePackageFor}[2]{ \ifx#2\undefined\usepackage{#1}\fi }

\UsePackageFor{amssymb}{\boxtimes}
\usepackage[normalem]{ulem}		% For strikethroughs.
\usepackage{bbding}				% For the small pencil.
\usepackage{endnotes}			% For endnotes.
\usepackage{hyperref}			% For hyperlinks to and from endnotes.
\usepackage{pdfcolfoot}			% Fixes issues with coloured footnotes that break across pages.

\ifdefined\OrigFootnote\relax\else% might be already applied by MacrosLabels
	\newenvironment{FootnoteContent}{}{}
	\let\OrigFootnote\footnote
	\let\OrigFootnoteText\footnotetext
	\renewcommand{\footnotetext}[1]{\OrigFootnoteText{\begin{FootnoteContent}#1\end{FootnoteContent}}}
	\renewcommand{\footnote    }[1]{\OrigFootnote    {\begin{FootnoteContent}#1\end{FootnoteContent}}}
\fi

\definecolor{PurplePlum}{rgb}{0.1,0,0.55} 
\definecolor{Brown}{rgb}{0.5,.25,0}
\definecolor{Orange}{rgb}{1,.3,0}
\definecolor{Gray}{rgb}{.7,.7,.7}
\definecolor{DarkGreen}{rgb}{.1,.41,.1}
\newif\ifBleck
\newcommand\Bleck {\Blecktrue} % Carry out all deletes and replacements; remove all footnotes/endnotes; make everything black.

\newcommand\Colour[1] {\color{#1}}

\newcommand\PrintToCLinks{	% This prints the actual links.
  {\Colour{blue}\mbox{
    \hyperlink{w1619}{\sf$\rightarrow$~top}\quad
    \hyperlink{w1031}{\sf$\rightarrow$~toc}\quad
    \hyperlink{w1148}{\sf$\rightarrow$~lof}\quad
    \hyperlink{GreenRoom}{\sf$\rightarrow$~gr}\quad
    \hyperlink{EndNotes}{\sf$\rightarrow$~en}\quad
    \hyperlink{Sargasso}{\sf$\rightarrow$~sg}\quad
    \hyperlink{Index}{\sf$\rightarrow$~idx}
  }}
}
\makeatletter % must be outside \newcommand
\newcommand\ToCLinks{
  \ifx\@onlypreamble\@notprerr		% If in document...
    \hypertarget{w1619}{}			% ...do it now.
  \else
    \AtBeginDocument{\hypertarget{w1619}{}}	% But in the preamble, delay.
  \fi

  \ifBleck\else	% 1031: Put hyperlink to ToC  and GR at bottom of every page.
    \ifdefined\cofoot
      \cofoot{\PrintToCLinks}
      \cefoot{\PrintToCLinks}
    \else
      \def\@oddfoot{\PrintToCLinks}
      \def\@evenfoot{\PrintToCLinks}
    \fi
 \fi
}
\makeatother

\newif\ifEndNotes % 0831: Workaround for endnotes-bug: missing ".ent" file when there are no endnotes.

\newcommand\FnSym{{\scriptsize\PencilLeftDown\kern.1em}}		% Needs bbding.
\newcommand\EnSym {{$\bigtriangledown$}}
\def\MarkupsHowto{} % How-to content, we add in every \MakeMarkups, use \MarkupHowto to print.
\newcommand{\MarkupsHowtoAdd}[1]{\expandafter\def\expandafter\MarkupsHowto\expandafter{\MarkupsHowto{}#1}} % adds to \MarkupsHowto

\newif\ifMarkupsHowtoPrinted %1121: Print help text only once. (Note: this is actually set to true when the how-to is _printed_.)
\newif\ifSuppress % Used to decide whether to suppress individual markups.

\newcommand\MakeMarkups[3][.]{%1119

     \Suppressfalse
     \ifBleck\Suppresstrue\fi
     \ifx0#1\Suppresstrue\fi
     \ifx1#1\Suppressfalse\fi
     
     \expandafter\providecommand\csname#2x\endcsname {} % 1613: Using providecommand this way allows MakeMarkups to be called more than once for a particular person.
     \ifSuppress\expandafter\renewcommand\csname#2x\endcsname{\relax}\else
                       \expandafter\renewcommand\csname#2x\endcsname{#3}\fi
                       
     \expandafter\providecommand\csname#2\endcsname {} % 1613.
     \ifSuppress\expandafter\renewcommand\csname#2\endcsname[1]{##1}\else
                       \expandafter\renewcommand\csname#2\endcsname[1]{{\csname#2x\endcsname##1}}\fi

     \expandafter\providecommand\csname#2d\endcsname {} % 1613.
     \ifSuppress\expandafter\renewcommand\csname#2d\endcsname[1]{\relax}\else
                       \expandafter\renewcommand\csname#2d\endcsname[1]{{\csname#2x\endcsname\sout{##1}}}\fi
                       
     \expandafter\providecommand\csname#2r\endcsname {} % 1613.
     \ifSuppress\expandafter\renewcommand\csname#2r\endcsname[2]{{##2}}\else
                       \expandafter\renewcommand\csname#2r\endcsname[2]{\csname#2d\endcsname{##1} \csname#2\endcsname{##2}}\fi

     \expandafter\providecommand\csname#2i\endcsname {} % 1613.
     \ifSuppress\expandafter\renewcommand\csname#2i\endcsname[1]{\relax}\else
                       \expandafter\renewcommand\csname#2i\endcsname[1]{\csname#2\endcsname{##1}}\fi

     \expandafter\providecommand\csname#2t\endcsname {} % 1613.
     \ifSuppress\expandafter\renewcommand\csname#2t\endcsname[1]{\relax}\else
                       \expandafter\renewcommand\csname#2t\endcsname[1]{{\csname#2x\endcsname{\mbox{$\langle\!\langle$}##1{\csname#2x\endcsname\mbox{$\rangle\!\rangle$}}}}}\fi 

     \expandafter\providecommand\csname#2b\endcsname {} % 1613.
     \ifSuppress\expandafter\renewcommand\csname#2b\endcsname[1][empty]{\relax}\else 
                       \expandafter\renewcommand\csname#2b\endcsname[1][\empty]{\ifx\empty##1\empty
                       	\label{#2-bookmark} % With hyperlinks you can jump directly to your bookmark by clicking on an appropriate \pageref.
                              \marginpar [\raggedleft\csname#2\endcsname{{\footnotesize\fbox{#2 working here}}~$\Longrightarrow$}]
                                                {\csname#2\endcsname{$\Longleftarrow$~{\footnotesize\fbox{#2 working here}}}}
                       \else % No label if there's an explicit parameter.
                       	\marginpar [\raggedleft\csname#2\endcsname{\ifx\empty##1\empty\else\fbox{\tiny\parbox{8em}{\raggedright##1}}~\fi$\Longrightarrow$}]
                                                {\csname#2\endcsname{$\Longleftarrow$\ifx\empty##1\empty\else~{\tiny\fbox{\parbox{8em}{\raggedright##1}}}\fi}}\fi}\fi

     \expandafter\providecommand\csname#2TD\endcsname {} % 1613.
     \ifSuppress\expandafter\renewcommand\csname#2TD\endcsname{\relax}\else
                       \expandafter\renewcommand\csname#2TD\endcsname{\csname#2\endcsname{\fbox{#2 to do}}}\fi

     \expandafter\providecommand\csname#2Bar\endcsname {} % 1613.
     \ifSuppress\expandafter\renewcommand\csname#2Bar\endcsname{\relax}\else
                       \expandafter\renewcommand\csname#2Bar\endcsname{\csname#2\endcsname{\scriptsize\XSolidBrush}}\fi

     \expandafter\providecommand\csname#2f\endcsname {} % 1613.
     \ifSuppress\expandafter\renewcommand\csname#2f\endcsname[2][]{\relax}\else
      \expandafter\renewcommand\csname#2f\endcsname[2][\empty]{ % Have forgotten what the optional parameter was for.
        {\mbox{\csname#2x\endcsname\tiny$\boxtimes$}\marginpar{\hsize1cm\csname#2x\endcsname\fbox{\FnSym\footnotemark}}\relax 
        \footnotetext{\csname#2x\endcsname##2}}}\fi

     \expandafter\providecommand\csname#2e\endcsname {}% 1613.
     \ifSuppress\expandafter\renewcommand\csname#2e\endcsname[1]{\relax}\else%
      \expandafter\renewcommand\csname#2e\endcsname[1]{%
       \global\EndNotestrue%0831
       \mbox{\scriptsize\csname#2x\endcsname$\boxtimes$}\relax%
       \marginpar{\hsize1cm\csname#2x\endcsname\fbox{\EnSym\endnotemark%
                          \hypertarget{ENmark\thepage-\theendnote}{}~\hyperlink{ENtext\thepage-\theendnote}{{\Colour{blue}$\downarrow$}}}%
       }%
       {% 1641 Make \z,\zz local.
        \def\zz{\noexpand#3}%
        \edef\z{~{[Endnote \theendnote\ %
        on p.\noexpand\hypertarget{ENtext\thepage-\theendnote}{}\thepage%
                    ~\noexpand\hyperlink{ENmark\thepage-\theendnote}{{\noexpand\Colour{blue}$\uparrow$}}]}%
        }%
        \expandafter\endnotetext\expandafter{\z\vspace{2ex}\\ ##1\newpage}%
       }%1641
      }\fi

     \expandafter\providecommand\csname#2n\endcsname {}% 1613.
     \ifSuppress\expandafter\renewcommand\csname#2n\endcsname[1]{\relax}\else%
      \expandafter\renewcommand\csname#2n\endcsname[1]{%
       \global\EndNotestrue%0831
    \marginpar{{\tiny\endnotemark}\hypertarget{ENmark\thepage-\theendnote}{}~\hyperlink{ENtext\thepage-\theendnote}{}}
       {% 1641 Make \z,\zz local.
        \def\zz{\noexpand#3}%
        \edef\z{~{\zz[Endnote (deferred) % \theendnote\ 
        from p.\noexpand\hypertarget{ENtext\thepage-\theendnote}{}\thepage%
        ]}%
        }%
        \expandafter\endnotetext\expandafter{\z\vspace{2ex}\\ ##1\newpage}%
       }%1641
      }\fi

     \expandafter\providecommand\csname#2fe\endcsname {} % 1613.
     \ifSuppress\expandafter\renewcommand\csname#2fe\endcsname[2][]{\relax}\else %1538
      \expandafter\renewcommand\csname#2fe\endcsname[2][]{ % 1540: \File is local.
       \def\File{##1}\relax
       \ifx\File\empty\csname#2f\endcsname{##2}\else % 1536: Without file, behave like \Xf.
        \global\EndNotestrue %0831
        \mbox{\scriptsize\csname#2x\endcsname$\boxtimes$}
        \marginpar{\csname#2x\endcsname\fbox{\FnSym\footnotemark}}\relax
        \footnotetext{~\csname#2x\endcsname##2\
                             --- See [\EnSym\endnotemark\hypertarget{ENmark\thepage-\theendnote}{}
                             \kern-.2em\hyperlink{ENtext\thepage-\theendnote}{{\Colour{blue}$\downarrow$}}].}\relax
       { % 1641 Make \z,\zz local.
         \def\zz{\noexpand#3}
         \edef\z{~{\zz[Endnote~\thefootnote~on~p.\noexpand\hypertarget{ENtext\thepage-\theendnote}{}\thepage
                     ~\noexpand\hyperlink{ENmark\thepage-\theendnote}
                     {{\noexpand\Colour{blue}\kern-0.1em$\uparrow$}]}}
                     {\noexpand\footnotesize\noexpand\newline\noexpand\hspace*{2em} (~from file {\noexpand\tt\File.tex}~)}
         }    
         \expandafter\endnotetext\expandafter{\z~\par\input{##1}\newpage}
        } %1641
       \fi % 1536
      } % 1540
     \fi % 1538

     \ifSuppress\relax\else\ifBleck\relax\else%0749
      \MarkupsHowtoAdd{\par\csname#2t\endcsname{%1009
       $\backslash$\texttt{#2}$\cdots$\ markups are in \textbf{this} colour\ifx#1..\else\ifx1#1.\else, e.g.\ for #1.\fi\fi
       \ifMarkupsHowtoPrinted\relax\else %1121
        \global\MarkupsHowtoPrintedtrue %1121
        \begin{quote}\begin{tabular}{l@{\hspace{2em}}p{.7\linewidth}}
         \multicolumn{2}{l}{\texttt{$\backslash$MakeMarkups\ifx#1.\relax\else[#1]\fi\{#2\}\{{\it$\langle$colour command\/$\rangle$}\}}
         				 --- Defines the macros below:}\\
             & see comments at \texttt{$\backslash$MakeMarkups} definition. \\[1ex]
         \texttt{$\backslash$#2\{$\langle$text$\rangle$\}} & Sets \texttt{$\langle$text$\rangle$} in \texttt{#2}'s colour. \\
         \texttt{$\backslash$#2x} & Changes to \texttt{#2}'s colour (until end of context). \\
         \texttt{$\backslash$#2d\{$\langle$text$\rangle$\}} & Sets \texttt{$\langle$text$\rangle$} in \texttt{#2}'s colour with a strikethrough (i.e.\ delete). \\
         \texttt{$\backslash$#2r\{$\langle$this$\rangle$\}\{$\langle$that$\rangle$\}} &
          Strikes through \texttt{$\langle$this$\rangle$} and inserts \texttt{$\langle$that$\rangle$} (i.e.\ replace). \\
         \texttt{$\backslash$#2f\{$\langle$text$\rangle$\}} & Meta-comment: puts \texttt{$\langle$text$\rangle$} in a \texttt{#2}-footnote with a {\tiny$\boxtimes$} in the main text. \\
         \texttt{$\backslash$#2t\{$\langle$text$\rangle$\}} & Use for meta when  \texttt{$\backslash$#2f} isn't allowed (``Not in outer-par mode.'') \\
         \texttt{$\backslash$#2b[$\langle$optional$\rangle$]} & Marginal pointer, with label for hyper-linking directly there. \\
         \texttt{$\backslash$#2e\{$\langle$text$\rangle$\}} & Puts \texttt{$\langle$text$\rangle$} in a \texttt{#2}-endnote with a (big) $\boxtimes$ in the main text. \\[.5ex]
         \texttt{$\backslash$#2n\{$\langle$text$\rangle$\}} & Like \texttt{$\backslash$#2e}
         except there's no reference from the main text. Good for ``decluttering''
         when you still want to have the footnote- or endnote texts as reminders. \\[.5ex]
         \texttt{$\backslash$#2fe[$\langle$this$\rangle$]\{$\langle$that$\rangle$\}} & Makes a \texttt{$\backslash$#2f\{$\langle$that$\rangle$\}} that refers to a \\
           & \texttt{$\backslash$#2e\{$\langle$contents of file this.tex$\rangle$\}}. \\ 
           & Without the optional argument, acts as \texttt{$\backslash$#2f\{$\langle$that$\rangle$\}}. \\[.5ex]
         \texttt{$\backslash$#2Bar} & Inserts ``burn after reading'' symbol \csname#2Bar\endcsname, meaning
          \begin{quote}\begin{itemize}\setlength\itemsep{0pt}
           \item If yours is the only \csname#2Bar\endcsname\ in this (presumably someone else's) footnote, and you are happy that the footnote has been addressed,
           go ahead and comment-out the whole footnote. (The \csname#2Bar\endcsname\ is their request for you to ``approve and remove''.)
           \item If you are not happy, delete only your \csname#2Bar\endcsname\ and follow-on in the footnote
            (in your colour, i.e.\ with \texttt{$\backslash$#2x}) saying why you are not happy.
           \item If you are happy, but there are others' burn-after-reading symbols as well as yours, just delete yours; the other people have not yet responded.
          \end{itemize}
          \end{quote}
          The idea is that when everyone's happy, the last person will comment-out the meta-text. \\[0.5ex]
         \texttt{$\backslash$#2TD} & Inserts {\csname#2TD\endcsname}\ . \\
        \end{tabular}\end{quote}
       \fi
      }}%1009
     \fi\fi%0749
}%1119

\newif\ifNoGreenRoom
\newcommand\MakeGreenRoom {\ifBleck\relax\else\ifNoGreenRoom\relax\else
\newcommand\NewGRLabel[1] {\OldGRLabel{GreenRoom-##1}} % All Green-Room labels "branded".
 \newcommand\NewGRRef[1] % All Green-Room references take branded version if possible.
 {\expandafter\ifx\csname r@GreenRoom-##1\endcsname\relax\OldGRRef{##1}\else\OldGRRef{GreenRoom-##1}\fi}
 \let\OldGRLabel\label \let\label\NewGRLabel
 \let\OldGRRef\ref \let\ref\NewGRRef
 \hrule
 ~\\\begin{center}\Huge \hypertarget{GreenRoom}{Green Room}
 \end{center}~\\
 \hrule
\fi\fi}

\newcommand\EndGreenRoom  {\ifBleck\relax\else\ifNoGreenRoom\relax\else
\let\label\OldGRLabel
\let\ref\OldGRRef
\fi\fi}

\newif\ifNoEndNotes
\newcommand\MakeEndNotes {\ifBleck\relax\else\ifNoEndNotes\relax\else
 \hrule\vspace{20pt}\begin{center}\ifEndNotes\Huge Endnotes\else\textit{No endnote macros were called in this run.}\fi\end{center}\vspace{20pt}\hrule
 {\parindent 0pt \parskip 2ex \def\enotesize{\normalsize} \def\notesname{} % Supress EndNotes title.
 \ifEndNotes\theendnotes\fi} %0831
\fi\fi}

\newif\ifNoSargasso
\newcommand\MakeSargasso {
 \hypertarget{Sargasso}{}
 \newcommand\NewLabel[1] {\OldLabel{Sargasso-##1}} % All Sargasso labels "branded".
 \newcommand\NewRef[1] % All Sargasso references take branded version if possible.
 {\expandafter\ifx\csname r@Sargasso-##1\endcsname\relax\OldRef{##1}\else\OldRef{Sargasso-##1}\fi}
 \let\OldLabel\label \let\label\NewLabel
 \let\OldRef\ref \let\ref\NewRef
\ifBleck\end{document}\else\ifNoSargasso
\relax% \begin{center}\huge \textit{Sargasso suppressed in this run}\end{center}
\else
  \hrule
  ~\\\begin{center}\Huge Sargasso
  \end{center}~\\
  \hrule
 \fi\fi
}

\newcommand\EndSargasso  {\ifBleck\relax\else\ifNoSargasso\relax\else
\let\label\OldLabel
\let\ref\OldRef
\fi\fi}

\newcommand\EndDocument {\ifBleck\end{document}\fi} % Cut off the Green Room and the Sargasso when Bleck.

\newcommand\Cite[2][\empty] {{\Colour{red}\ifx#1\empty[#2]\else[#2,~#1]\fi}}

\makeatletter
\newif\ifenotelinks
\newcounter{Hendnote}
\let\savedhref\href
\let\savedurl\url
\def\endnotemark{%
\@ifnextchar[\@xendnotemark{%
\stepcounter{endnote}%
\protected@xdef\@theenmark{\theendnote}%
\protected@xdef\@theenvalue{\number\c@endnote}%
\@endnotemark
}%
}%
\def\@xendnotemark[#1]{%
\begingroup\c@endnote#1\relax
\unrestored@protected@xdef\@theenmark{\theendnote}%
\unrestored@protected@xdef\@theenvalue{\number\c@endnote}%
\endgroup
\@endnotemark
}%
\def\endnotetext{%
\@ifnextchar[\@xendnotenext{%
\protected@xdef\@theenmark{\theendnote}%
\protected@xdef\@theenvalue{\number\c@endnote}%
\@endnotetext
}%
}%
\def\@xendnotenext[#1]{%
\begingroup
\c@endnote=#1\relax
\unrestored@protected@xdef\@theenmark{\theendnote}%
\unrestored@protected@xdef\@theenvalue{\number\c@endnote}%
\endgroup
\@endnotetext
}%
\def\endnote{%
\@ifnextchar[\@xendnote{%
\stepcounter{endnote}%
\protected@xdef\@theenmark{\theendnote}%
\protected@xdef\@theenvalue{\number\c@endnote}%
\@endnotemark\@endnotetext
}%
}%
\def\@xendnote[#1]{%
\begingroup
\c@endnote=#1\relax
\unrestored@protected@xdef\@theenmark{\theendnote}%
\unrestored@protected@xdef\@theenvalue{\number\c@endnote}%
\show\@theenvalue
\endgroup
\@endnotemark\@endnotetext
}%
\def\@endnotemark{%
\leavevmode
\ifhmode
\edef\@x@sf{\the\spacefactor}\nobreak
\fi
\ifenotelinks
\expandafter\@firstofone
\else
\expandafter\@gobble
\fi
{%
\Hy@raisedlink{%
\hyper@@anchor{Hendnotepage.\@theenvalue}{\empty}%
}%
}%
\hyper@linkstart{link}{Hendnote.\@theenvalue}%
\makeenmark
\hyper@linkend
\ifhmode
\spacefactor\@x@sf
\fi
\relax
}%
\long\def\@endnotetext#1{%
\if@enotesopen
\else
\@openenotes
\fi
\immediate\write\@enotes{%
\@doanenote{\@theenmark}{\@theenvalue}%
}%
\begingroup
\def\next{#1}%
\newlinechar='40
\immediate\write\@enotes{\meaning\next}%
\endgroup
\immediate\write\@enotes{%
\@endanenote
}%
}%
\def\theendnotes{%
\immediate\closeout\@enotes
\global\@enotesopenfalse
\begingroup
\makeatletter
\edef\@tempa{`\string>}%
\ifnum\catcode\@tempa=12
\let\@ResetGT\relax
\else
\edef\@ResetGT{\noexpand\catcode\@tempa=\the\catcode\@tempa}%
\@makeother\>%
\fi
\def\@doanenote##1##2##3>{%
\def\@theenmark{##1}%
\def\@theenvalue{##2}%
\par
\smallskip %<-small vertical gap between endnotes
\begingroup
\def\href{\expandafter\savedhref}%
\def\url{\expandafter\savedurl}%
\@ResetGT
\edef\@currentlabel{\csname p@endnote\endcsname\@theenmark}%
\enoteformat
}%
\def\@endanenote{%
\par\endgroup
}%
\renewcommand*\@makeenmark{%
\hbox{\normalfont\@theenmark~}%
}%
\enoteheading
\enotesize
\input{\jobname.ent}%
\endgroup
}%
\def\enoteformat{%
\rightskip\z@
\leftskip1.8em
\parindent\z@
\leavevmode\llap{%
\setcounter{Hendnote}{\@theenvalue}%
\addtocounter{Hendnote}{-1}%
\refstepcounter{Hendnote}%
\ifenotelinks
\expandafter\@secondoftwo
\else
\expandafter\@firstoftwo
\fi
{\@firstofone}%
{\hyperlink{Hendnotepage.\@theenvalue}}%
{\makeenmark}%
}%
}%
\makeatother
\enotelinkstrue
\Bleck		% to produce a clean version
\MakeMarkups[Kostas]{K}{\Colour{Brown}}
\MakeMarkups[Catuscia]{U}{\Colour{PurplePlum}}
\MakeMarkups[Marco]{M}{\Colour{Orange}}

\begin{document}

\title{Optimal Obfuscation Mechanisms via Machine Learning 
%\thanks{This research was supported by DATAIA Convergence Institute as part of the ``Programme d’Investissement d’Avenir'' (ANR-17-CONV-0003), operated by Inria and CentraleSupélec. The work of Catuscia Palamidessi was supported by the European Research Council (ERC) under the European Union’s Horizon 2020 research and innovation programme. Grant agreement № 835294.}
}

%\author{
%\IEEEauthorblockN{Marco Romanelli\IEEEauthorrefmark{1},
%			       Kostantinos Chatzikokolakis\IEEEauthorrefmark{2},
%                                Catuscia Palamidessi\IEEEauthorrefmark{1}
%}
%\IEEEauthorblockA{  \IEEEauthorrefmark{1}Inria and LIX, \'{E}cole Polytechnique, France \\
%                                  \IEEEauthorrefmark{2}University of Athens, Greece}
%}

%\author{\IEEEauthorblockN{1\textsuperscript{st} Marco Romanelli}
\author{\IEEEauthorblockN{Marco Romanelli}
\IEEEauthorblockA{\textit{Inria} \\
\textit{LIX, Ecole Polytechnique, IPP}%\\
%email address
}
\and
%\IEEEauthorblockN{2\textsuperscript{nd} Kostantinos Chatzikokolakis}
\IEEEauthorblockN{Konstantinos Chatzikokolakis}
\IEEEauthorblockA{\textit{University of Athens, Greece} %\\
%email address
}
\and
%\IEEEauthorblockN{3\textsuperscript{rd} Catuscia Palamidessi}
\IEEEauthorblockN{Catuscia Palamidessi}
\IEEEauthorblockA{\textit{Inria} \\
\textit{LIX, Ecole Polytechnique, IPP}%\\
%email address
}
%\and
%\and
%\and
}
\maketitle

%%%%%%%%%%%%%%%%%%%%%%%%%%%%%%%%
%%%%%%%%%%%%%%%%%%%%%%%%%%%%%%%%
%%%%%%%%%%%%%%%%%%%%%%%%%%%%%%%%

%\TODO{Submitted papers may include up to 13 pages of text and up to 
%5 pages for references and appendices, totalling no more than 18 pages. 
%The same applies to camera-ready papers, although, at the 
%PC chairs’ discretion, additional pages may be allowed for references 
%and appendices. Reviewers are not required to read appendices.
%Papers must be formatted for US letter (not A4) size paper. 
%The text must be formatted in a two-column layout, with columns no more 
%than 9.5 in. tall and 3.5 in. wide. The text must be in Times font, 10-point or
% larger, with 11-point or larger line spacing. Authors are encouraged to use 
% the IEEE conference proceedings templates. LaTeX submissions should 
% use IEEEtran.cls version 1.8. All submissions will be automatically checked 
% for conformance to these requirements. Failure to adhere to the page limit 
% and formatting requirements are grounds for rejection without review.
%\begin{enumerate}
%\item check the name of the networks and always call them with the same name;
%\item Remove page number;
%\item Manually equalize the lengths of two columns on the last page of your paper;
%\item Ensure that any PostScript and/or PDF output post-processing
% uses only Type 1 fonts and that every step in the generation
% process uses the appropriate paper size.
% \item Correct the terminology:  Laplace mechanism, Laplacian noise (note the majuscule)
% \end{enumerate}
%}

\begin{abstract}
We consider the problem of obfuscating sensitive information while preserving utility, and
we propose a  machine-learning approach    inspired by the generative adversarial networks  paradigm. 
The idea is to set up two nets: the generator, that tries to produce an optimal obfuscation mechanism to protect the data,  and the classifier, that tries to de-obfuscate the data. By letting the two nets compete against  each other, the mechanism improves its degree of protection, until an equilibrium is reached. 
We apply our method to the case of location privacy, and we perform experiments on synthetic data and on real data from the Gowalla dataset.  We evaluate the privacy of the  mechanism not only by its capacity to defeat the classifier, but also in terms of 
the Bayes error, which represents the strongest possible adversary.  We compare the privacy-utility tradeoff  of our method with that of the planar Laplace mechanism used in geo-indistinguishability, showing favorable results.  Like the  Laplace mechanism,  our system can be deployed at the user end for protecting his location. 

%
%We consider the problem of obfuscating location traces in order to prevent re-identification, while preserving utility. Assuming that the adversary has access to location 
%data collected from a community of users, and can use them to learn how to infer their identities on new locations, we formulate the attack 
%as a classification problem. The goal is then to devise an optimal noise distribution to perturbate the locations so to minimize the attacker's probability of
%success, while avoiding excessive degradation of the utility. Given that an analytical
%solution to this problem does not scale for large datasets, we propose an approach based on machine learning, inspired by the GANs
%(Generative Adversarial Networks) paradigm. However, since in our case the target distribution has to be ``invented" rather than ``imitated", our resulting method 
%departs significantly from that of the GANs. 
%We perform experiments on a synthetic dataset to illustrate our approach. Then we evaluate it on the \emph{Gowalla} dataset, and we compare the privacy-utility tradeoff of our mechanism with the Laplace one (which is  state-of-the-art in location privacy), showing favorable results.  With  a loose enough utility constraint, in fact, our mechanism achieves the maximum privacy expressed in terms of Bayes error, namely it induces maximum confusion in an (ideal) Bayesian adversary.  
%%\TODO{Add a sentence saying that our noise is better than Laplacian noise, measured analytically as Bayes error}
\end{abstract}
% !TEX root = main.tex

\section{Introduction}
\label{section_1}
Data analytics are crucial for modern companies and, consequently, there is an enormous interest in collecting and processing 
all sort of personal information. Individuals, on the other hand, are often willing to provide their data  in exchange of improved services and experiences.
%In principle, the big-data technologies brings all sorts of benefits at many different levels. 
%Over the last few years, social medias, mobile applications and cloud services have turned big data into an integral part of global economy. 
%All kinds of personal data  are collected everyday to provide services and better meet customers' need. 
However there is the risk that such disclosure of personal information could be used against them.
%: It could affect everything from relationships to getting a job, 
%or qualifying for a loan, or worse. 
The rise of machine learning, with its capability of performing powerful  analytics on massive amounts of data, has further exacerbated the risks. % for privacy. 
Several researchers have   pointed out possible threats such as the  model inversion attacks~\cite{Fredrikson:15:CCS} and the membership inference attacks~\cite{Shokri:17:SP,Pyrgelis:18:NDSS,Hayes:19:PETS,Melis:19:SP}.

%To illustrate the point, consider the following scenario: some  users  send their identities and coordinates to a Location Based Service (LBS)  
%to obtain some kind of assistance, for instance the points of interest (POIs) near them. 
%%Such locations could be used by an attacker to extract information using machine learning techniques. For instance, the 
%An attacker that has access to the LBS could collect the traces of these users for a while, and use machine learning techniques  and some background knowledge (for instance the home address of the users) to infer information from them. For example, he could train a machine to classify traces, to associate a class  to a home address and therefore to a user, and also to connect a  trace outset to its possible continuations. 
%%{reconstruct users' tracks over time and foretell where a certain user is likely to be next.
%Later on,  the attacker could use the machine to identify the user from a new trace (even if the trace does not contain the home address), or to predict the likely next location the user will visit.  

Nonetheless, if machine learning can be a threat, it can also be a powerful means to build good privacy protection mechanisms, as we will  demonstrate in this paper. 
We focus on mechanisms that obfuscate data  by adding  controlled noise.
Usually the quality of service (QoS) that the user receives in exchange of his obfuscated data degrades with the amount of obfuscation, hence the challenge is to find a  
good trade-off between privacy and utility. Following the approach of~\cite{Shokri:17:TPS},  we aim at maximizing the privacy protection while preserving the desired QoS\footnote{Other approaches  take the opposite view, and aim at maximizing utility while achieving the desired amount of privacy, see for instance~\cite{Bordenabe:14:CCS}.}.
We consider the case of \emph{ location privacy } and in particular  the  \emph{re-identification} of the user from his location, but the framework 
that we develop is general and can be applied to any situation in which an attacker might infer sensitive  information from accessible  correlated  data. 

Utility is typically expressed as a bound on the expected distance between the real location and the obfuscated one\footnote{This notion is known as \emph{distortion} in information theory~\cite{Cover:06:BOOK}.}  \cite{Shokri:17:TPS,Andres:13:CCS,Bordenabe:14:CCS,Chatzikokolakis:17:POPETS},  capturing the fact that location based services usually offer a better QoS when they receive a more accurate location. 
If also privacy is expressed as a linear function, then the optimal trade-off  can in principle be achieved with linear programming \cite{Shokri:17:TPS,Bordenabe:14:CCS,Shokri:15:PETS,Oya:17:CCS}.
The limitation of this approach, however,  is that it does not scale to large datasets.
The problem is that the linear program needs one variable for every pair $(w,z)$ of real and obfuscated locations. 
Such variables represent the probability of producing the obfuscated location $z$ when the real one is $w$.
For a $50\times 50$ grid this is more than six million variables, which is already at the limit
of what modern solvers can do. For a $260\times 260$ grid, the program has 4.5 billion variables, making
it completely intractable (we could not even launch such a program due to the huge memory requirements).
%This problem can in fact be thought of as a \emph{linear optimization problem}  where the privacy is the objective function and 
%the  threshold on utility is the constraint. The limitation of this approach, however,  is that it does not scale to large datasets:   
%the existing tools cannot handle more than a few  hundreds possible locations. 
Furthermore, the background knowledge and the correlation between data points affect privacy and  are usually  difficult to determine and express formally.  
\Kt{Write about ``analytic'' methods in the sense of Planar Laplace}

Our position is that \emph{machine learning can help to solve this problem}. Inspired by the GANs paradigm~\cite{Goodfellow:14:NIPS}, 
 we propose a {system} consisting of two adversarial neural networks, $\mathit{G}$ (\emph{generator}) and $\mathit{C}$ (\emph{classifier}). 
 The idea is that   $\mathit{G}$  generates noise so to confuse the adversary as much as possible, within the boundaries of the utility constraints, 
while  $\mathit{C}$ inputs the noisy locations produced by  $\mathit{G}$ and tries to  re-identify  (classify) the corresponding user.  
%In other words, $\mathit{C}$ tries to build a classification function, where the obfuscated locations are the features, and the users' ids are the labels. The classification produced by $\mathit{C}$ is then fed back to $\mathit{G}$, which uses it to regulate the noise injection. 
While fighting against $\mathit{C}$, $\mathit{G}$ refines its 
strategy, until  a point where it cannot improve any longer. 
Note that a  significant difference from the standard GANs is that, in the latter, the generator has to learn to reproduce an existing distribution from samples. In our case, instead, 
the generator has to ``invent'' a distribution from scratch. 
%A main difference between our approach and the GANs   is that  the  GANs have access to a dataset  sampled from the distribution that $\mathit{G}$ should learn to reproduce. The adversary,  which  in that case is called $\mathit{D}$ (\emph{discriminator}), tries  to  distinguish between the real data (from the dataset) and those generated by  $\mathit{G}$. The net $\mathit{G}$  then ``learns''  the target distribution from the feedback provided  by $\mathit{D}$. 
%In our case, on the contrary,  there is  no dataset of samples that can ``direct'' $\mathit{G}$ towards an optimal noise distribution. In a sense, our $\mathit{G}$ has to be more ``creative'' and ``invent'' a good distribution from scratch. 

The interplay between $\mathit{G}$ and $\mathit{C}$  can be seen as an instance of a
zero-sum  Stackelberg game~\cite{Shokri:17:TPS}, where $\mathit{G}$ is the \emph{leader}, 
and $\mathit{C}$ is the \emph{follower},  and the 
payoff function $f$ is the privacy loss. Finding the optimal point of equilibrium between $\mathit{G}$ and 
$\mathit{C}$ corresponds to solving a minimax problem on $f$ with $\mathit{G}$ being the minimizer and  $\mathit{C}$ the maximizer. 

A major challenge in our setting  is represented by the choice of $f$. 
A first idea would be to measure it in terms of  $\mathit{C}$'s capability to re-associate a location to the right user. 
Hence we could define  $f$ as 
the expected success probability of $\mathit{C}$'s classification.
%, or, alternatively, as 
%the similarity between  $\mathit{C}$'s classification and the real one, which can be formalized in terms of \emph{cross entropy}. 
%As notion of similarity we could use the converse of the \emph{cross entropy}, which is typical in machine learning.      
Such function $f$ would be convex/concave with respect to the strategies of $\mathit{G}$  and $\mathit{C}$   respectively, so from game theory we would derive the existence   of a saddle point corresponding to  the optimal obfuscation-re-identification pair. 
The problem, however, is that it is difficult to reach the  saddle point via the typical  alternation between the two nets. 
Let us clarify this point with a simple example\footnote{A similar example was independently pointed out  in \cite{Abadi:16:CoRR}.}: 
%\begin{table}
%\center
\begin{figure*}%\label{tab:ideal_game_cross_entropy_and_MI}
\centering
		 {\footnotesize		
		\begin{subfigure}{1\columnwidth}
		\centering
		%\begin{table}[h]
	        %\begin{center}
\begin{tabular}{|b|c|c|c|c|c|}
	\cline{3-6}
	\hhline{*2{>{\arrayrulecolor{white}}-}>{\arrayrulecolor{black}}|*4{-}|} 
	\multicolumn{2}{c|}{} & \multicolumn{4}{c|}{\parbox[c]{2.4mm}{\cellcolor{lightgray}\vspace{1mm}$\mathit{C}$\vspace{1mm}}} \\ 
	\cline{3-6}
	\hhline{*2{>{\arrayrulecolor{white}}-}>{\arrayrulecolor{black}}|*4{-}|}
	\multicolumn{2}{c|}{} & \multicolumn{1}{c|}{\cellcolor{lightgray}$ \begin{array}{c} \scriptstyle  \mathit{a\shortrightarrow A}\\ \scriptstyle \mathit{b\shortrightarrow B}\end{array}$} & \multicolumn{1}{c|}{\cellcolor{lightgray}$   \begin{array}{c}\scriptstyle \mathit{a\shortrightarrow A}\\\scriptstyle  \mathit{b\shortrightarrow A}\end{array}$} & \multicolumn{1}{c|}{\cellcolor{lightgray}$  \begin{array}{c}\scriptstyle \mathit{a\shortrightarrow B}\\ \scriptstyle \mathit{b\shortrightarrow B}\end{array}$} & \multicolumn{1}{c|}{\cellcolor{lightgray}$  \begin{array}{c}\scriptstyle \mathit{a\shortrightarrow B}\\ \scriptstyle \mathit{b\shortrightarrow A}\end{array}$}\\
	\hline
	\multirow{7}{*}{\begin{sideways} $\color{red}\mathit{G}$~ \end{sideways}} 
	&\cellcolor{lightgray} $\begin{array}{c}\scriptstyle  A\shortrightarrow a\\ \scriptstyle  B\shortrightarrow b\end{array}$ & ${1}$  & ${0.5}$ & ${0.5}$  &${0}$ \\  %[5pt]
	\cline{2-6}
	\hhline{|*1{>{\arrayrulecolor{lightgray}}-}>{\arrayrulecolor{black}}|*5{-}|}
	&\cellcolor{lightgray} ... & ...  &... & ... & ... \\  %[5pt]
	\hhline{|*1{>{\arrayrulecolor{lightgray}}-}>{\arrayrulecolor{black}}|*5{-}|}
	$G$&\cellcolor{lightgray}$\begin{array}{c}\scriptstyle A{\shortrightarrow}{ a} \\ \scriptstyle B{\shortrightarrow}{a}\end{array}$  &  ${0.5}$  & ${0.5}$ &  ${0.5}$ &  ${0.5}$ \\  %[5pt]
%	$G$&\cellcolor{lightgray}$\begin{array}{c}\scriptstyle A{\shortrightarrow}{  0.5\, a, 0.5\, b} \\ \scriptstyle B{\shortrightarrow}{\scriptstyle 0.5\, a, 0.5\, b}\end{array}$  &  $0.5$  &  $0.5$&   $0.5$&  $0.5$ \\  %[5pt]
	\hhline{|*1{>{\arrayrulecolor{lightgray}}-}>{\arrayrulecolor{black}}|*5{-}|}
	&\cellcolor{lightgray} ... & ...  &... &...  &...  \\  %[5pt]
	\hhline{|*1{>{\arrayrulecolor{lightgray}}-}>{\arrayrulecolor{black}}|*5{-}|}
	&\cellcolor{lightgray} $\begin{array}{c}\scriptstyle A\shortrightarrow b\\ \scriptstyle  B\shortrightarrow a\end{array}$ & ${0}$  & ${0.5}$ &  ${0.5}$ &  ${1}$ \\  %[5pt]
	\hline 
\end{tabular}
\subcaption{$f=$ Expected success probability of the  classification. % task.
}\label{subtb:classification}
\end{subfigure}
%\hspace{.4 cm}
%\hfill%
%\medskip
%\subcaption{Payoff:  success probability of the classification task.}\label{subtb:classification}
%\bigskip
\begin{subfigure}{1\columnwidth}
\centering
\begin{tabular}{|b|c|c|c|c|c|}
	\hhline{*2{>{\arrayrulecolor{white}}-}>{\arrayrulecolor{black}}|*4{-}|}
	\multicolumn{2}{c|}{} & \multicolumn{4}{c|}{\parbox[c]{2.4mm}{\cellcolor{lightgray}\vspace{1mm}$\mathit{C}$\vspace{1mm}}} \\ 
	\hhline{*2{>{\arrayrulecolor{white}}-}>{\arrayrulecolor{black}}|*4{-}|}
	\multicolumn{2}{c|}{} & \multicolumn{1}{c|}{\cellcolor{lightgray}$\begin{array}{c}\scriptstyle \mathit{a\shortrightarrow A}\\ \scriptstyle  \mathit{b\shortrightarrow B}\end{array}$} & \multicolumn{1}{c|}{\cellcolor{lightgray}$\begin{array}{c}\scriptstyle \mathit{a\shortrightarrow A}\\ \scriptstyle \mathit{b\shortrightarrow A}\end{array}$} & \multicolumn{1}{c|}{\cellcolor{lightgray}$\begin{array}{c}\scriptstyle \mathit{a\shortrightarrow B}\\ \scriptstyle \mathit{b\shortrightarrow B}\end{array}$} & \multicolumn{1}{c|}{\cellcolor{lightgray}$\begin{array}{c}\scriptstyle \mathit{a \shortrightarrow B}\\ \scriptstyle \mathit{b\shortrightarrow A}\end{array}$}\\
	\hline
	\multirow{7}{*}{\begin{sideways} $\mathit{G}$~ \end{sideways}} 
	& \cellcolor{lightgray}$\begin{array}{c}\scriptstyle  A\shortrightarrow a\\ \scriptstyle  B\shortrightarrow b\end{array}$ & $\mathbf{1}$ \;\;\; $\mathbb{1}$  & $\mathbf{0}$ \;\;\; $\mathbb{0.5}$ & $\mathbf{0}$ \;\;\; $\mathbb{0.5}$  &$\color{red}{\mathbf{1}}$\;\;\; $\color{red}{\mathbb{1}}$ \\  %[5pt]
	\hhline{|*1{>{\arrayrulecolor{lightgray}}-}>{\arrayrulecolor{black}}|*5{-}|}
	&\cellcolor{lightgray} ... & ...  &... & ... & ... \\  %[5pt]
	\hhline{|*1{>{\arrayrulecolor{lightgray}}-}>{\arrayrulecolor{black}}|*5{-}|}
	$G$&\cellcolor{lightgray}$\begin{array}{c}\scriptstyle A{\shortrightarrow}{a} \\ \scriptstyle B{\shortrightarrow}{a}\end{array}$  &  $\mathbf{0}$ \;\;\; $\mathbb{0.5}$  &  $\mathbf{0}$ \;\;\; $\mathbb{0.5}$ & $\mathbf{0}$ \;\;\; $\mathbb{0.5}$ &  $\mathbf{0}$ \;\;\; $\mathbb{0.5}$ \\  %[5pt]
%	$G$&\cellcolor{lightgray}$\begin{array}{c}\scriptstyle A{\shortrightarrow}{  0.5\, a, 0.5\, b} \\ \scriptstyle B{\shortrightarrow}{\scriptstyle 0.5\, a, 0.5\, b}\end{array}$  &  $0.5$  &  $0.5$&   $0.5$&  $0.5$ \\  %[5pt]
	\hhline{|*1{>{\arrayrulecolor{lightgray}}-}>{\arrayrulecolor{black}}|*5{-}|}
	& \cellcolor{lightgray}... & ...  &... &...  &...  \\  %[5pt]
	\hhline{|*1{>{\arrayrulecolor{lightgray}}-}>{\arrayrulecolor{black}}|*5{-}|}
	&\cellcolor{lightgray} $\begin{array}{c}\scriptstyle  A\shortrightarrow b\\ \scriptstyle B\shortrightarrow a\end{array}$ & $\color{red}{\mathbf{1}}$ \;\;\; $\color{red}{\mathbb{1}} $  & $\mathbf{0}$ \;\;\; $\mathbb{0.5}$ & $\mathbf{0}$ \;\;\; $\mathbb{0.5}$  &${\mathbf{1}}$\;\;\; $\mathbb{1} $ \\  %[5pt]
	\hline 
\end{tabular}
%\medskip
\subcaption{{\bf Bold:} $f = I(X;Y)$. \;\;
$\mathbb{Hollow}$: $f=  1-B(X|Y)$.}\label{subtb:Bayes}
\end{subfigure}%
\caption{Payoff tables of the games  in Example~\ref{exa:AliceandBob}, for various payoff functions $f$.  
   $\mathit{A}$ stands for $\mathit{Alice}$  and $\mathit{B}$ for $\mathit{Bob}$.}\label{tab:ideal_game_cross_entropy_and_MI}
}	
\end{figure*}
%}
%\end{table}
\begin{example}\label{exa:AliceandBob}
Consider two users,  Alice and Bob,  in locations $a$ and $b$ respectively. Assume that at first  $\mathit{G}$ reports their true locations  (no noise). Then $\mathit{C}$   learns that $a$ corresponds to $\mathit{Alice}$ and $b$ to $\mathit{Bob}$. At the next round, $\mathit{G}$ will figure that  to maximize the misclassification error (given the prediction of $\mathit{C}$) it should swap  the locations, i.e., report $a$ for  $\mathit{Alice}$ and $b$ for  $\mathit{Bob}$. Then, on its turn, $\mathit{C}$ will  have to  ``unlearn'' the previous classification  and learn the   new one. But then, at the next round, $\mathit{G}$ will again swap the locations, and bring the situation   back to the starting point, and so on, without ever reaching an equilibrium. Note that a possible equilibrium point   for  $\mathit{G}$ would be  the mixed strategy that reports $a$  for both $\mathit{Alice}$ and 
$\mathit{Bob}$\footnote{There are two more equilibrium points: one is when both $\mathit{Alice}$ and  $\mathit{Bob}$  report $a$ or $b$ with uniform probability, the other is when they both report $b$. All the three strategies 
%give the same payoff and we could choose any of them to illustrate the issue. 
are equivalent.}  (so that   $\mathit{C}$ could only make a bling guess), but $\mathit{G}$ may not stop there. The problem is that it is difficult to calibrate 
the training of $\mathit{G}$ so that it stops in proximity of the saddle point   rather than continuing all the way to reach its relative  optimum. The situation is illustrated in Fig.\ref{subtb:classification}. 
%The payoff function considered in this figure is the success probability of the classification, but it would be analogous if we considered, for instance, the cross   entropy between the true ids and $\mathit{C}$'s prediction.                                                                                             
\end{example}

In order to address this issue we  adopt a different target function,  less sensitive to the particular labeling strategy of $\mathit{C}$. 
The idea is to consider not just the \emph{precision} of the classification, but, rather, the \emph{information} contained in it.
%,   which represents the  \emph{potential precision} of an ideal classifier that uses that information in the optimal way. 
There are two main ways of formalizing this intuition: the \emph{mutual information} $I(X;Y)$ and the  \emph{Bayes error} $B(X|Y)$, where
$X, Y$ are respectively the random variable associated to the \emph{true ids},  and to the ids resulting from the classification (\emph{predicted ids}).
We recall that $I(X;Y)=H(X)-H(X|Y)$, where $H(X)$ is the entropy of $X$ and 
$H(X|Y)$ is the residual entropy of $X$ given $Y$, while  
 $B(X|Y)$ is the  probability  of error when we select the value of X with \emph{maximum aposteriori probability}, given $Y$. 
% Its complement $1-B(X|Y)$ is also known as   \emph{posterior}   \emph{Bayes vulnerability}~\cite{Alvim:16:CSF}.
% , represents   the precision of an ideal classifier that makes the best possible guess.
 Mutual information and Bayes error are  related  by the  
Santhi-Vardy bound~\cite{Santhi:06:ACCCC}:
%\begin{equation}\label{eqn:SV}
$B(X|Y) \leq 1-2^{-H(X|Y)}.$

If we set $f$ to be $I(X;Y)$ or $1-B(X|Y)$, we obtain the  payoff table illustrated in Fig.\ref{subtb:Bayes}. 
Note that the mimimum $f$ in the first and last columns   corresponds  now to a  point of equilibrium for any choice of $\mathit{C}$. 
This is not always the case, but in general it is closer to the equilibrium and  makes the training of  
$\mathit{G}$ more stable:  training $\mathit{G}$ for a longer time does not risk to increase the distance from the equilibrium point. 

In this paper we   use the mutual information to generate the noise, but we   evaluate the level of privacy  also in terms of the Bayes error, which 
 represents the probability of error of the strongest possible adversary. 
Both notions have been used in the literature as privacy measures,  
for instance mutual information has been applied to quantify anonymity \cite{Zhu:05:ICDCS,Chatzikokolakis:08:JCS}. 
The   Bayes error 
%and the  Bayes vulnerability have 
has been considered  in \cite{Chatzikokolakis:08:JCS,McIver:10:ICALP,Cherubin:17:POPETS1, Alvim:16:CSF}, and indirectly as \emph{min-entropy leakage} in \cite{Smith:09:FOSSACS}. 
\changed{Oya et al. advocate in \cite{Oya:17:CCS} that to guarantee a good level of location privacy a mechanism should measure well in terms of both  the Bayes error and the residual entropy (which is strictly related to mutual  information)}.  
Fig.~\ref{tab:anticipate_results} anticipates some of the experimental results of Sections~\ref{synthetic_experiments} and~\ref{Govalla_experiments}. 
We note that the performance of our mechanism is much better than the planar  Laplace, and comparable to that of the optimal solution in  all the three cases in which we can determine the latter. 
\changed{Of course, this comparison is not completely fair, because the planar Laplace was designed to satisfies a different notion of privacy, called \emph{geo-indistinguishability}~\cite{Andres:13:CCS} (see next paragraph). Our mechanism on the contrary does not satisfy this notion.}
\renewcommand{\arraystretch}{1.5}
 \begin{figure*}[t]\centering
{\footnotesize
\begin{subfigure}{.45\columnwidth}\centering
\begin{tabular}{|c|c|c|}
\hline
 \multicolumn{3}{|c|}{\cellcolor{lightblue} {Synthetic data, low utility}} \\ 
\hline
Laplace & Ours &  Optimal\\
\hline
$0.39$ & $0.74$ &  $0.75$\\
\hline
\end{tabular}
\end{subfigure}\hfill
\begin{subfigure}{.45\columnwidth}\centering
\begin{tabular}{|c|c|c|}
\hline
 \multicolumn{3}{|c|}{\cellcolor{lightblue}  {Synthetic data, high utility}} \\ 
\hline
Laplace & Ours &  Optimal\\
\hline
$0.23$ & $0.42$ &  $0.50$\\
\hline
\end{tabular}
\end{subfigure}\hfill%
\begin{subfigure}{.45\columnwidth}\centering
\begin{tabular}{|c|c|c|}
\hline
 \multicolumn{3}{|c|}{ \cellcolor{lightpurple} {Gowalla data, low utility}} \\ 
\hline
Laplace & Ours & Optimal \\
\hline
$0.33$ & $0.80$ &  $0.83$\\
\hline
\end{tabular}
\end{subfigure}\hfill
\begin{subfigure}{.45\columnwidth}\centering
\begin{tabular}{|c|c|c|}
\hline
 \multicolumn{3}{|c|}{ \cellcolor{lightpurple} {Gowalla data, high utility}} \\ 
\hline
Laplace & Ours & Optimal\\
\hline
$0.28$ &  $0.38$ & ?\\
\hline
\end{tabular}
\end{subfigure}\hfill
}
\caption{\footnotesize Bayes error on synthetic and Gowalla data, for the Laplace mechanism, our mechanism, and the optimal one, on a grid of $260\times 260$ cells. 
In the last table the Bayes error of the optimal mechanism is unknown: the linear program contains 4.5 billion variables, making it intractable in practice. }
\label{tab:anticipate_results}
\end{figure*}
\renewcommand{\arraystretch}{1}

Other popular privacy metrics are  \emph{differential privacy} (DP) \cite{Dwork:06:TCC}, \emph{local differential privacy} (LPD)~\cite{Duchi:13:FOCS}, and $d$-\emph{privacy}~\cite{Chatzikokolakis:13:PETS}, of which geo-indistinguishability is an instance. 
%Its relation with mutual information has been explored in \cite{Mir:13:FPS,Cuff:16:CCS}, while its relation with the Bayes vulnerability has been investigated in  \cite{Alvim:15:JCS}. 
The main difference between these and the notions used in this paper is that they are worst-case measures, while ours are average. 
\changed{In other words, ours refer to the \emph{expected} level of privacy over all sensitive data, while the others are concerned with the protection of \emph{each} individual datum. 
Clearly, the latter is stronger, as proved in \cite{Alvim:11:ICALP} and \cite{De:12:TCC}, although \cite{Cuff:16:CCS} has proved that a conditional version of mutual information correspond to a relaxed form of differential privacy called $(\varepsilon,\delta)$-differential privacy.
We regard the individual protection as an important issue, and we plan to investigate the possibility of generating worst-case mechanisms via ML in future work}.  This paper 
is a preliminary exploration of the applicability of ML to  privacy, and as a starting point we focus on the average notions that have been considered in location privacy~\cite{Shokri:17:TPS,Shokri:15:PETS,Oya:17:CCS}.    
%\end{equation}
%Furthermore, this bound is tight when the classifier is totally confused, namely when $H(X|Y)  = H(X) = \log {n}$, where $n$ is the cardinality of the ids.
%This corresponds to introducing enough noise so that the optimal  classifier   cannot do better that  random guessing,  which is the ideal goal of $\mathit{G}$. 
%Our target function, therefore, will be the  negative  residual entropy   $-H(X|Y)$, or, equivalently, the mutual information $I(X;Y)=H(X)-H(X|Y)$, since  we assume $X$ to be fixed.  
%It is important to note that also the game with $I(X;Y)$ as payoff has a table similar to that of Fig.~\ref{subtb:Bayes}. 

%We are interested in comparing the noise function computed by our method to the Laplacian noise function which is the most widely used. 
%First we run tests on a synthetic dataset. It represents a particular scenario where the limits of the Laplacian mechanism clearly appear. 
%We then move along running experiments on a subset of the \emph{Gowalla} dataset centered in  in 5, Boulevard de S\'ebastopol, Paris, France and considering the six most frequent users in that area.

From a practical point of view  our method belongs to the    \emph{local privacy} category, like LDP and geo-indistinguishability, in the sense that it can be deployed at the user's end, with no   need of a trusted third party. Once the training is done the system can be used as a personal device   that, each time the user needs to report his location to a LBS, generates a sanitized   version of it by adding noise to the real location.   
%The training phase however must be done by a trusted party, who needs to collect samples from a population of users.  But since the method is probabilistic, all what the trusted party will know, at the end of the training,  is the prior distribution of the users and the noise mechanism produced, which is usually assumed to be the case (white-box security).
%\Kt{\KTD Mention private learning}

\subsection{Contribution}
The contributions of the paper are the following:
\begin{itemize}
\item We propose an approach based on adversarial nets to generate obfuscation mechanisms with a good  privacy-utility tradeoff. 
The advantage of our method  is twofold:
\begin{itemize}
\item wrt linear programming methods, we can work on a continuous domain instead of a small  grid;
\item wrt analytic methods (such as the Planar Laplace mechanism) our approach is data-driven, taking
	into account prior knowledge about the users.
\end{itemize}
\item Although our approach is inspired by the GANs paradigm, it departs significantly from it: In our case, the distribution has to be ``invented'' rather than ``imitated''. Hence we need different techniques for evaluating a distribution. To achieve our goal, we propose a new method based on the mutual information
between the supervised and the predicted class labels. 

\item We show that the use of the use of mutual information (instead of the cross entropy) for the generator  is crucial for convergence. On the other hand for the classifier it is possible to use cross entropy and it is more efficient. 
%  replacing mutual information with crossentropy (which is commonly done for
%efficiency) is possible for the classifier, but \emph{not for the generator}, demonstrated by
%a case-study on simple synthetic data. For the generator
%the use of mutual information is crucial for convergence.

% \item We apply our method to the case of location privacy. We craft some experiments to show in detail how our approach works and how close it is to achieving optimality. 
%Depending on the distribution of data and on the utility constraint, the resulting noise function may achieve the maximal privacy, i.e., equivalent to that of random guessing. Trivially, in this case it also achieves the optimal  tradeoff between privacy and utility. 
\item We evaluate the obfuscation  mechanism produced by our method on real location data from the  Gowalla dataset.
\item We compare our mechanism with the  planar  Laplace~\cite{Andres:13:CCS} and with the optimal one, when it is possible to compute or determine theoretically the latter. 
We show that the performance of our mechanism is much better than Laplace, and not so far from the optimal.
%Then, we confirm these results (and hence the advantages of our method) by means of  the (ideal) optimal Bayesian classifier, which represents the strongest possible adversary. 
\item We have made publicly available the implementation and the experiments at  \url{https://gitlab.com/MIPAN/mipan}.
\end{itemize}

\subsection{Related work}
Optimal mechanisms, namely mechanisms providing an optimal compromise between  utility and privacy, have attracted the interest  of many researchers.
Many of the studies so far have focused on optimization methods based on linear programming \cite{Shokri:17:TPS,Bordenabe:14:CCS,Shokri:15:PETS,Oya:17:CCS}. Although they can provide exact solutions, the huge size of the corresponding linear programs limits the scalability of these methods. 
Our approach, in contrast, using the efficient optimization process of neural networks (the gradient descent), does not suffer from this drawback. All the experiments were done 
on grid sizes for which linear programming is completely intractable. 

Adversarial  networks to construct privacy-protection mechanisms have been also proposed by  \cite{Abadi:16:CoRR,Tripathy:17:ArXiv,Huang:17:Entropy}, 
with applications on image data (the MNIST and the GENKI datasets).
The   authors of  \cite{Tripathy:17:ArXiv,Huang:17:Entropy} have also developed a theoretical framework similar to ours. 
From the methodological point of view the  main difference is that in the implementation
 they use  as target function  the cross entropy rather than the mutual information. 
 Hence in our setting the convergence of their method may be problematic,  due to the ``swapping effect'' described in Example~\ref{exa:AliceandBob}. 
 We have actually experimented \mods{the use of cross entropy} as target function on our examples in Section ~\ref{synthetic_experiments}, 
 and we could not achieve convergence. The intermediate mechanisms were unstable and the level of privacy  was poor. 
Another related paper is \cite{Jia:18:USENIX}, which uses  an adversarial  network to produce mechanisms against attribute inference attacks. The target function is the Kullback-Liebler divergence, which, in this particular context where the distribution of the secrets is fixed, reduces to cross entropy. Hence in our setting we would get the  same swapping effect explained above. 

\changed{Other works that have proposed the use of minimax learning to preserve privacy are \cite{Hamm:17:JMLR,Hayes:18:NIPS,Edwards:16:ICLR,Ren:18:ECCV}. 
The author of  \cite{Hamm:17:JMLR}  introduces the notion of minimax filter as a solution to the optimization problem between privacy as expected risk and utility as distortion, and propose various learning-based methods to approximate such solution. 
The authors of \cite{Hayes:18:NIPS} consider multi-party machine learning, and use  adversarial training  to mitigate privacy-related attacks such as party membership inference of individual
records. 
The authors of \cite{Edwards:16:ICLR} propose the minimax technique to  remove private information from personal images. Their approach is to use a stochastic gradient alternate min-max optimizer, but since they express the objective in terms of cross entropy, they may incur in the same problem as described above, i.e., they cannot guarantee convergence.  
The authors of \cite{Ren:18:ECCV}   consider personal images, and in particular the problem of preventing their re-identification  while preserving their utility, such as the the discernibility of the actions in the images. They use  the angular softmax loss as objective function, and do not analyze the problem of convergence, but  their experimental results are impressive.  }
%Their application to images privacy seems to confirm this expectation: their goal is to obfuscate the images of some digits  
%the image of two figures but do not seem to reduce the mutual information. 

\changed{Another related line of work is the generation of synthetic data via machine learning. An example is \cite{Beaulieu-Jones:19:CIRCOUTCOMES}, where the authors use an adversarial network to generate artificial medical records that  closely resemble participants of the Systolic Blood Pressure Trial dataset. In this case, the paradigm they use is the same as the original GAN: the discriminator takes in input both the records produced by the generator and samples from the  original dataset, and tries to distinguish them. The  original dataset is also obfuscated with differential privacy techniques to prevent membership attacks.}

One of the side contributions of our paper is a method to compute  mutual information in neural network (cfr. Section~\ref{implementation_in_nn}). 
Recently, Belghazi et al. have proposed MINE, an efficient method to neural  estimation of mutual information~\cite{Belghazi:18:ICML}, inspired  by the framework of \cite{Nowozin:16:NIPS} for the estimation of a general class of functions representable as  $f$-divergencies. These methods work also in the continuous case and for high-dimensional data. 
In our case, however, we are dealing with a discrete domain, and we can compute directly and \emph{exactly} the mutual information. 
Another reason for developing our own method is that we need to deal with \mods{a loss function} that contains not only the mutual information, but also a component representing utility, and depending on the notion of utility  the result may not be an $f$-divergence. 

Our paradigm has been inspired by the \emph{GANs}~\cite{Goodfellow:14:NIPS}, but it comes with some fundamental differences:
	\begin{itemize}
	\item $\mathit{C}$ is a classifier performing re-identification while in the \emph{GANs} there is a discriminator able to distinguish a real data distribution from a generated one;
	\item in the \emph{GANs} paradigm the generator network \mods{tries to reproduce} the original data distribution to fool the discriminator. 
	A huge difference is that, in our adversarial scenario, $\mathit{G}$ does not have a model distribution to refer to. The final data distribution only depends on the evolution of the two networks over time and it is driven by the constraints imposed in the loss functions that rule the learning process.
%	 and that will be described in Section~\ref{implementation_in_nn}. 
	\item We still adopt a training algorithm which alternates the training of $\mathit{G}$ and of $\mathit{C}$, but as we will show in Section~\ref{implementation_in_nn}, it is different from the one adopted for \emph{GANs}.
	\end{itemize}

% !TEX root = main.tex

\section{Our setting}\label{sec:oursetting}

\begin{table}[t]%[htbp]
\begin{center}
\begin{tabular}{|c|l|}
\hline
\textbf{Symbol}&\textbf{Description}\\
\hline
%\cline{1-3} 
 $\mathit{C}$ &  \pbox{15cm}{Classifier network (attacker).} \\%& \textbf{\textit{Table column subhead}}& \textbf{\textit{Subhead}}& 
\hline
$\mathit{G}$ &  \pbox{15cm}{Generator network.} \\
\hline
$X, \calx$ &  \pbox{15cm}{\vspace{0.5mm} Sensitive information. (Random var. and domain.)\vspace{0.5mm}} \\
\hline
$W, \calw$ & \pbox{15cm}{\vspace{0.5mm}Useful information with respect to\\  the intended notion of utility.\vspace{0.5mm}} \\
\hline
$Z, \calz$ &  \pbox{15cm}{\vspace{0.5mm}Obfuscated information accessible \\to the service provider   and to the  attacker.\vspace{0.5mm}}\\
\hline
$Y, \caly$ &  \pbox{15cm}{Information inferred by the attacker. }\\
\hline
$P_{\cdot ,\cdot}$ & \pbox{15cm}{Joint probability of two random variables.}\\
\hline
$P_{\cdot|\cdot}$ & \pbox{15cm}{\vspace{0.5mm}Conditional probability.\vspace{0.5mm}}\\
\hline
$P_{Z|W}$ & \pbox{15cm}{\vspace{0.5mm}Obfuscation mechanism.\vspace{0.5mm}}\\
\hline
$B(\cdot\mid \cdot)$ & \pbox{15cm}{\vspace{0.5mm}Bayes error.\vspace{0.5mm}}\\
\hline
$\loss[Z\mid W]$ & \pbox{15cm}{\vspace{0.5mm}Utility loss induced by the obfuscation mechanism.\vspace{0.5mm}}\\
\hline
$L$ & \pbox{15cm}{\vspace{0.5mm}Threshold on the utility loss.\vspace{0.5mm}}\\
\hline
$H(\cdot)$& \pbox{15cm}{Entropy of a random variable.}\\
\hline
$H(\cdot|\cdot)$& \pbox{15cm}{\vspace{0.5mm}Conditional entropy.\vspace{0.5mm}}\\
\hline
$I(\cdot;\cdot)$& \pbox{15cm}{Mutual information between two random variables.}\\
\hline
\end{tabular}
\caption{Table of symbols}
\label{tab_synthetic_laplace_results}
\end{center}
\end{table}

We formulate the privacy-utility optimization problem using a framework similar to that of \cite{Basciftci:16:ITAW}. 
We consider four random variables,  $X,Y,Z,W$, ranging over the sets ${\calx}, {\caly}, {\calz}$ and ${\calw}$ respectively, with the following meaning:
\begin{itemize}
\item $X$: the sensitive information that the users wishes to conceal,
\item $W$: the useful information with respect to some service provider and  the intended notion of utility,
\item $Z$: the information made visible to the service provider,  which may be intercepted by some  attacker, and 
\item $Y$: the information inferred by the attacker.
\end{itemize}
We assume a fixed joint distribution (\emph{data model}) $P_{X,W}$ over the users' data  $\calx\times\calw$. 
We present our framework assuming that the variables are discrete,  but all results and definitions can be transferred to the continuous case, 
by replacing the distributions with probability density functions, and the summations with integrals. 
For the initial definitions and results of this section ${\calx}$ and ${\caly}$ may be different sets.  Starting from Section~\ref{implementation_in_nn} we will  assume that ${\calx}={\caly}$.

An obfuscation mechanism can be represented as  a conditional probability distribution $P_{Z|W}$, 
where $P_{Z|W}(z|w)$ indicates the  probability that the mechanism transform the data point  $w$ into the noisy data point $z$.
We assume that $Z$ are the only attributes visible to the attacker and to the service provider. 
The goal of the defender $G$ is to  optimize the data release mechanism $P_{Z|W}$ so to achieve  a desired level of utility while minimizing the leakage of the sensitive attributes $X$. 
The goal of the attacker $C$ is to retrieve $X$ from $Z$ as precisely as possible. In doing so, it produces a classification $P_{Y|Z}$ (\emph{prediction}). 

Note that the four random variables form a Markov chain:
\begin{equation}\label{eq:markov}
X \leftrightarrow W \leftrightarrow Z \leftrightarrow Y.
\end{equation}
Their joint distribution is completely determined by the data model, the obfuscation mechanism and the classification:
\[
 \mbox{\small 
$P_{X, W, Z, Y}(x,w,z,y)\;=\; P_{X,W}(x,w)  P_{Z|W}(z\mid w) P_{Y|Z} (y\mid z)$}. 
\]
\noindent From $P_{X, W, Z, Y}$ we can derive the marginals, the conditional probabilities of any  two variables, etc. 
For instance:
\begin{eqnarray}
P_{X}(x)&=& \sum_w P_{X,W}(x,w).\\
P_{Z}(z)&=& \sum_{xw} P_{X,W}(x,w)   P_{Z|W}(z\mid w).\\
P_{Z|X}(z|x)&=& \frac{\sum_w P_{X,W}(x,w)   P_{Z|W}(z\mid w)}{P_{X}(x)}.\\
P_{X|Z}(x|z)&=& \frac{P_{Z|X}(z|x)   P_{X}(x)}{P_{Z}(z)}.
\end{eqnarray}
The latter distribution, $P_{X|Z}$, is the \emph{posterior distribution} of $X$ given $Z$, and plays an important role in the following sections.

\subsection{Quantifying utility}
Concerning the utility, we consider a loss function  $\ell : W \times Z \rightarrow [0, \infty)$, where $\ell(w,z)$ represents the utility loss caused by reporting $z$ when the true value is $w$. 
\begin{definition}[Utility loss]\label{def:UtilityLoss}
The utility loss from the original data $W$ to the noisy data $Z$, given the loss function  $\ell$, is defined as the expectation of 
$\ell$:
\begin{equation} \label{def:UtilityLoss}
{\loss}[Z\mid W, \ell]  \;=\; \expect[\ell \mid W,Z]  \;=\;\sum_{w z} P_{W,Z}(w,z) \ell(w,z).
\end{equation}
\end{definition}

We will omit $\ell$ when it is clear from the context. 
Note that, given a data model $P_{X,W}$, the utility loss can be expressed in terms of the mechanism $P_{Z|W}$:
\begin{equation}\label{eqn:concreteutilityloss}
{\loss}[Z\mid W]  \;=\; \sum_{x w z} P_{X,W}(x,w)  P_{Z|W}(z|w) \ell(w,z).
\end{equation}
Our goal is to build a privacy-protection mechanism  that keeps the loss below a certain threshold $L$. We denote by 
$M_L$ the set of such mechanisms, namely:
\begin{equation}\label{eq:UtilityConstraint}
M_L   \;\stackrel{\rm def}{=}\;  \{ P_{Z|W} \mid {\loss}[Z\mid W]  \leq L\}.
\end{equation}
 
The following property is immediate: 
\begin{restatable}[Convexity of $M_L$]{proposition}{ConvexityM_L}
\label{prop:ConvexityM_L}
The set $M_L$ is convex and closed.
\end{restatable}

\subsection{\changed{Quantifying privacy as mutual information}}\label{sec:mi}

\changed{We recall the basic information-theoretic definitions that will be used in the paper:\\
Entropy of $X$:
\begin{equation}\label{eq:H} H(X) = - \sum_{x} P_{X}(x) \log  P_{X}(x).\end{equation}
Residual Entropy of $X$ given $Y$:
\begin{equation}\label{eq:RH}  H(X|Y) = - \sum_{xy} P_{X,Y}(x,y) \log  P_{X|Y}(x|y).\end{equation}
Mutual Information between $X$ and $Y$:
\begin{equation}\label{eq:MI}  I(X;Y) = H(X) - H(X|Y).\end{equation}
Cross entropy between the posterior and the prediction:
\begin{equation}\label{eq:CE} 
\small \mathit{CE}(X,Y) = - \sum_z P_{Z}(z) \sum_{x} P_{X\mid Z}(x|z) \log  P_{Y\mid Z}(y|z).
\end{equation}
}

\changed{We recall that the more correlated $X$ and $Y$ are, the larger is $I(X;Y)$, and viceversa. The minimum $I(X;Y)=0$ is when $X$ and $Y$ 
are independent; the maximum is when the value of $X$ determines uniquely the value of $Y$ and viceversa. 
In contrast, $\CE(X,Y)$, that represents the precision loss in the classification prediction, is not related to the correlation between $X$ and $Y$, but rather to the similarity between 
$P_{X\mid Z}$ and $P_{Y\mid Z}$: the more similar they are, the smaller is  $\CE(X,Y)$. In particular,  the minimum $\CE(X,Y)$ is   when  $P_{X\mid Z}=P_{Y\mid Z}$. 
}

\changed{The privacy leakage of a mechanism $P_{Z\mid W}$
with respect to an attacker  $C$, characterized by  the prediction  
 $P_{Y\mid Z}$,  will be 
quantified by the mutual information $I(X;Y)$. This notion of privacy will be used as objective function,  
rather than the more typical cross entropy  $\CE(X,Y)$. 
As explained in the introduction, this choice makes the training of $G$ more stable because,  
in order to reduce   $I(X;Y)$, $G$ cannot simply  swap  around the 
 labels of the classification learned by $C$, it must reduce the correlation between $X$ and $Z$ (via suitable modifications of $P_{Z\mid W}$),  and in doing so it limits the amount of information that \emph{any} adversary can infer about $X$ from $Z$.
We will come back on this point in more detail in \autoref{sec:mi-vs-ce}.}

% $X$ and the classification $Y$  produced by $C$ \changed{by observing $Z$, and}
%which  the obfuscation mechanism created by $G$ tries to minimize.

%This is the reason why we need to consider  $I(X;Y)$ instead of $\CE(X,Y)$: in order to minimize the former, $G$ needs to minimize the correlation  
%between $X$ and $Z$ (because the correlation between $X$ and $Y$ is via $Z$),  
%and in doing so it limits the amount of information that any adversary can infer about $X$ from $Z$, and not just the precision of the prediction  
% $P_{Y\mid Z}$ of a particular adversary. 
%As pointed out in the introduction, this choice makes the training of $G$ more stable: if $G$ tried to maximize  $\CE(X,Y)$ instead, 
%it could just swap  around the 
%labels of the classification learned by $C$, but at the next round $C$ would learn the new classification, and reduce $\CE(X,Y)$ to its minimum again. 
%We will come back to this point in \autoref{sec:mi-vs-ce}.}
%

 \subsection{Formulation of the game}\label{sec:game}

The game that $\mathit{G}$ and $\mathit{C}$   play corresponds to the following minimax formulation:
\begin{equation}
	\label{eq:minimax}
	\min_\mathit{G} \;\max_\mathit{C} \; I(X;Y)
\end{equation}
where the minimization by $\mathit{G}$ is on the mechanisms $P_{Z\mid W}$ ranging over $M_L$, while the maximization by $\mathit{C}$ is on the classifications  $P_{Y\mid Z}$. 

Note that $P_{Z|W}$ can be seen as a stochastic matrix and therefore as an element of a vector space. An important property for our purposes is that the mutual information  is convex  with respect to $P_{Z|W}$:

\begin{restatable}[Convexity of $I$]{proposition}{ConvexityI}
\label{prop:ConvexityI}
Given  $P_{X,W}$ and $P_{Y|Z}$, let $f(P_{Z|W}) = I(X;Y)$. Then $f$  is convex.
% Namely, given a pair of convex coefficients $c_1$ and $c_2$, and 
% two mechanisms  {\footnotesize $P^{(1)}_{Z|W}$} and  {\footnotesize  $P^{(2)}_{Z|W}$}, we have:
% \begin{equation}
% f(c_1 P^{(1)}_{Z|W} + c_2 P^{(2)}_{Z|W}) \;\leq\;  c_1  f(P^{(1)}_{Z|W} )+ c_2 f(P^{(2)}_{Z|W}).
% \end{equation}
\end{restatable}

Proposition~\ref{prop:ConvexityM_L} and \ref{prop:ConvexityI}  show that this problem is well defined: for any choice of $C$, 
$ I(X;Y)$ has a global minimum in $M_L$, and no strictly-local minima.

\subsubsection*{On the use of the the classifier}
We note that, in principle, one could avoid using the GAN paradigm, and try to achieve the optimal mechanism by  solving, instead, the following minimization problem: 

\begin{equation}\label{eq:max}
\min_{G}\;   I(X;Z)
\end{equation}
where $\min_{G}\; I(X;Z)$ is meant, as before,  as a minimization over the mechanisms $P_{Z\mid W}$ ranging over $M_L$.
This approach would have the advantage that it is independent from the attacker, so one would need to reason only about $G$ (and there would be no need for a GAN). 

The main difference between  $I(X;Y)$ and  $I(X;Z)$ is that the latter  represents the information about $X$ available to any adversary, not only those that are trying to retrieve  $X$ by building a classifier. This fact reflects in the following relation between the two formulations: 

\begin{restatable}{proposition}{equilibriumI}
\label{prop:equilibriumI}
\[
\min_G \;\max_C \; I(X;Y)\;\leq \; \min_{G}\;   I(X;Z)
\]
\end{restatable}
Note that, since $\min_{G}\;   I(X;Z)$ is an upper bound of our target, it imposes a limit on $\max_C \; I(X;Y)$. 

On the other hand, there are some advantages in considering $\min_G \;\max_C \; I(X;Y)$ instead than $\min_{G}\;   I(X;Z)$: first of all, $Z$ may have a much larger and more complicated domain than $Y$, so performing the gradient descent on $I(X;Z)$  could be infeasible. 
Second, if we are interested in considering only  classification-based attacks, then $\min_G \;\max_C \; I(X;Y)$ should give a better result than $\min_{G}\;   I(X;Z)$. 
In this paper we focus on the former, and leave the exploration of an approach based on $\min_{G}\;   I(X;Z)$ as future work. 

\subsection{Measuring privacy as Bayes error} 
As explained in the introduction, we intend to evaluate the resulting mechanism also in terms of Bayes error. Here we give the relevant definitions and properties. 

\begin{definition}[Bayes error]\label{def:BayesError}
The  Bayes error of $X$ given $Y$ is: 
\[
B(X\mid Y)\;=\;\sum_y P_Y(y) (1-\max_x P_{X|Y} (x\mid y)).
\]
\end{definition}
Namely, the Bayes error is the expected probability of ``guessing the wrong id''  of an adversary that, when he sees that $C$ produces the id $y$, 
it guesses the id $x$ that has the highest posterior probability given $y$.

The definition of $B(X\mid Z) $ is analogous. Given a mechanism $P_{Z|W}$, we   regard $B(X\mid Y)$ as a measure of the privacy of $P_{Z|W}$ w.r.t. one-try~\cite{Smith:09:FOSSACS}  classification-based attacks, whereas $B(X\mid Z)$ is w.r.t.  any one-try attack.
The following proposition shows the relation between the two notions. 

\begin{restatable}{proposition}{LowerBound}
\label{prop:LowerBound}
$
B(X\mid Z)   \;\leq \; B(X\mid Y) 
$
\end{restatable}

%Next, we propose an implementation via GANs of the method illustrated above.

% !TEX root = main.tex

\section{Implementation in Neural Networks}\label{implementation_in_nn}

In this section we describe the implementation of our adversarial game
between $\mathit{G}$ and $\mathit{C}$ in terms of alternate training of neural networks.
%%%%%%%%%%%%%%%%%%%%%%%%%%%%%%%%
The scheme of our game is illustrated in Fig. \ref{fig:adversarial_scheme}, where: 
%\vspace{-1mm}
\begin{centering}
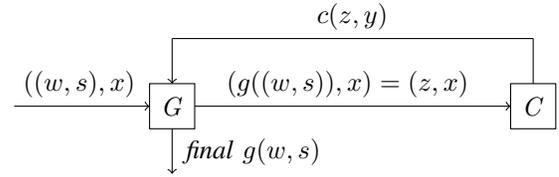
\begin{figure}[t]
\centering
\begin{tikzpicture}[scale=0.6]
\draw [black] (1,0.5) rectangle (2,1.5);
\node at (1.5,1) {$\mathit{G}$};
\draw [black] (9,0.5) rectangle (10,1.5);
\node at (9.5,1) {$\mathit{C}$};
\draw [black] [->] (-2, 1) -- (1,1);
\draw [black] [->] (2, 1) -- (9,1);
%\draw [black] [->] (10, 1) -- (13,1);
\draw [black] [-] (9.5, 2.5) -- (1.5,2.5);
\draw [black] [-] (9.5, 1.5) -- (9.5,2.5);
\draw [black] [->] (1.5, 2.5) -- (1.5,1.5);
\node [above right, black] at (4.5,2.5) {${c({z},y)}$};
\node [above right, black] at (-2,1) {${(({w}, {s}),x)}$};
\node [above right, black] at (2.5,1) {${(g(({w}, {s})),x)=({z},x)}$};
\draw [black] [->] (1.5, 0.5) -- (1.5,-0.5);
\node at (3.3,0) {{\it final} $g({w}, {s})$};
\node at (3.3,-1) {};
%\node [above right, black] at (10.2,1) {${c(\underline{z})=y}$};
\end{tikzpicture}
\caption{Scheme of the adversarial nets for our setting.}
\label{fig:adversarial_scheme}
\end{figure}
\end{centering}
%\newline
\begin{itemize}
%\item [~]%to leave an empty line before the list
\item $x,y,z$ and $w$ are instances of the random variables $X$,$Y$, $Z$ and $W$ respectively, whose meaning is described in previous section. We assume that the domains of $X$ and $Y$ coincide. 
\item $ {s}$ (seed)  is a randomly-generated number in $[0,1)$. 
\item $\mathit{g}$ is the function learnt by $\mathit{G}$, and it represents an obfuscation mechanism $P_{Z\mid W}$. The input $s$ provides the randomness needed to generate random noise. It is necessary because a neural network in itself is deterministic. 
\item $\mathit{c}$ is the classification learnt by    $\mathit{C}$,  corresponding to $P_{Y\mid Z}$.
%\item [~]%to leave an empty line after the list
\end{itemize}

%%%%%%%%%%%%%%%%%%%%%%%%%%%%%%%%%%%%%%%%%%%%%%%%%%%%%%%%%%%%%%%%%%%%%
%%%%%%%%%%%%%%%%%%%%%%%%%%%%%%%%%%%%%%%%%%%%%%%%%%%%%%%%%%%%%%%%%%%%%
	\makeatletter
	\newcommand{\removelatexerror}{\let\@latex@error\@gobble}
	\makeatother
	\newcommand{\myalgorithm}{%
	\begingroup
	\removelatexerror% Nullify \@latex@error
	\begin{algorithm}[t]
	\begin{tcolorbox}[width=0.45\textwidth]
	{\small
	\KwData{$\mathit{train\_data}$ \tcp{\footnotesize Training data}}
	\KwModels{$\mathit{G}_i$ generator evolution at the $i$--th step; \newline
                           $\mathit{C}_i$ classifier evolution at the $i$--th step. \newline }
   	 $\mathit{train(n, d)}$ trains the  network $n$ on the data $d$. 
	
	$\mathit{G_i(data)}$ outputs a noisy version of $\mathit{data}$.\newline
 	
	$C_0$ = base classifier model % $\mathit{C}_0$ = $\mathit{C}_0$ %
 
	$G_0$ = base generator model % $G_0$ = $G_0$ %
 
 	i = 0 %-1
 
 	\While{True}{
   	% \tcp{\footnotesize Let $P_{Z|W}$ be the mechanism produced  by the current generator.} 
	%and let  $P_{X|Z}$ be the corresponding posterior on $X$ given $Z$.}
	
	i += 1

	\tcp{Train class. from scratch}

  	$C_i$ = $train(C_{0}, G_{i-1}(\mathit{train\_data}))$\\% = $C_i$\\ %\tcp{\footnotesize Train the classifier  to produce  a $P_{Y|Z}$ that  approximates  $P_{X|Z}$.} 
	% \tcp{this update results in an update of the classifier in the adversarial network as well}

 	$A = G_{i-1}$ and $C_i$ in cascade
  
  	$A$ = $train(A, \mathit{train\_data})$\\
	% \tcp{\footnotesize Train the network to produce a new $P_{Z|W}$ by reducing $\mathit{Loss}_{\mathit{G}}$.
	% using  both the outputs from the generator and the classifier. 
	% The weights of the classifier are frozen during this phase.}
  
  	$G_i$ = generator layer in $A$ %\tcp{ \footnotesize Save generator evolution for next iteration.} 
 	}
	}
	 \end{tcolorbox}
	
	\caption{\small  Adversarial algorithm with classifier reset.}\label{alg:game}
	\end{algorithm}
	\endgroup}
%%%%%%%%%%%%%%%%%%%%%%%%%%%%%%%%%%%%%%%%%%%%%%%%%%%%%%%%%%%%%%%%%%%%%
%%%%%%%%%%%%%%%%%%%%%%%%%%%%%%%%%%%%%%%%%%%%%%%%%%%%%%%%%%%%%%%%%%%%%
%	\newline\newline
	\myalgorithm
%	\phantom{\\}
	The evolution of the adversarial network  is described in Algorithm~\ref{alg:game}.
	 $\mathit{C}$ and $\mathit{G}$ are trained at two different moments within the same adversarial training iteration. In particular $\mathit{C_i}$ is obtained by training the network 
	 $\mathit{C}$ against the noise generated by $\mathit{G_{i-1}}$ and $\mathit{G_i}$ is obtained by fighting against $\mathit{C_i}$. 
	 
	Note that in our method each $\mathit{C_i}$ is trained   on the output of  
	$\mathit{G_{i-1}}$. 
	This is a main difference with respect to the GANs paradigm, where the discriminator is  trained both \changed{on the output of the generator and on samples from 
	the target distribution generated by an external source}. 	 
	Another particularity of our method is that at the end of the $i$-th iteration, while $\mathit{G_i}$ is retained for the next iteration, $\mathit{C_i}$ is discarded and the classifier
	for iteration $i+1$ is reinitialized to the base one $\mathit{C}_0$. 
	\changed{The reason is that restarting  from  $\mathit{C}_0$ is more efficient than starting 
	 from the last trained classifier $\mathit{C_i}$. This is  because $\mathit{G_i}$ may have changed at step $i$ the noise mechanism $P_{Z\mid W}$ and therefore the 
	association between $X$ and $Z$ expressed by $P_{X\mid Z}$. The  predictions $P_{Y\mid Z}(x\mid z)$ that $\mathit{C_i}$ had produced during its training 
	(trying to match the  $P_{X\mid Z}(x\mid z)$ previously produced by $\mathit{G_{i-1}}$  as closely as possible), not only is not optimal  anymore: for some  $z$'s 
	it may have become completely wrong, and starting from a 
	wrong prediction is a drawback that slows down the learning of the new prediction. There may be several  $z$'s for which the old prediction is a good approximation of the 
	new one to be learned, but according to our experiments the net effect is negative:  the training of the new classifier is usually faster if we restart from scratch. 
	It is worth noting that this is only a matter of efficiency though: eventually, even if we started from $\mathit{C_i}$, the new classifier would ``unlearn'' the old, wrong 
	predictions and learn the correct new ones.  }	
%if $\mathit{C_{i}}$ is not reset, it might take a long time before it becomes able to contrast the noise injection at iteration $i+1$. 
%In fact $\mathit{C_{i}}$ has been trained on data generated by $\mathit{G_{i-1}}$ and it has consistently updated its weights according to that. 
%It could take a long time for $\mathit{C_{i}}$ to ``forget" what it has learned and move on to learning to beat the new noise distribution. 
%On the contrary, when the classifier is reset to $C_0$, it quickly adapts to the data produced by the new generator. 

At the end of each training iteration we evaluate the quality of the produced noise by checking the performance of the $\mathit{C}$ network. In particular we make sure that the noise produced by the $\mathit{G}$ network affects the training, validation and test data in a similar way. In fact, in case the performances were good on the training data but not on the the other data, this would be a result of overfitting rather than of a quality indicator of the injected noise.

We describe now in more detail some key implementation choices of our proposal.

\subsection{Mutual information vs cross entropy}\label{sec:mi-vs-ce}

Based on the formulation of our game \eqref{eq:minimax}, the alternate training of
both $G$ and $C$ is performed using the \emph{mutual information}
$I(X;Y)$ as the loss function.
\changed{The goal of $\mathit{G}$ is to minimize $I(X;Y)$ by refining the mechanism
 $P_{Z\mid W}$, while $\mathit{C}$ aims at maximizing it by refining the 
 classifier $P_{Y\mid Z}$.} 

\changed{
We remark that the use of mutual information as loss function is not standard.
A more typical  function  for training a classifier is the cross entropy $\CE(X,Y)$, 
which is more efficient to implement. 
$\CE(X,Y)$ is minimized when $P_{X\mid Z}$ and   $P_{Y\mid Z}$ coincide. 
Such outcome would correspond to the 
perfect classifier, that predicts the exact probability $P_{X\mid Z}(x|z)$ that a given sample 
$z$ belongs to the class $x$.
One could then think of reformulating the game in terms of the cross entropy $\CE(X,Y)$, 
where  $\mathit{C}$ would be the minimizer (trying to infer probabilistic information about 
the secret $x$ from a given observation $z$) and $\mathit{G}$ the maximizer 
(trying to prevent  the adversary $\mathit{C}$ from achieving this knowledge). 
However, as already 
observed in  Example~\ref{exa:AliceandBob} in the introduction, training $G$  
via $\CE(X,Y)$ does not allow to reach an equilibrium, because it  takes  into account
only one adversarial strategy (i.e., one particular classification). 
Indeed, a maximum   $\CE(X,Y)$ can be achieved with a $P_{Z\mid W}$ that   simply 
causes a swapping of the associations between the labels $x$'s and the corresponding noisy locations $z$'s. 
This would change $P_{X\mid Z}$ and therefore 
fool the present classifier (because the prediction $P_{Y\mid Z}$ would not be equal anymore to $P_{X\mid Z}$), but at the next round, when $C$
will be trained on the new data, it will learn the new classification $P_{X\mid Z}$ and 
obtain, again, the maximum information about $x$ that can be inferred from $z$. 
The possibility of ending up in such cyclic behavior is experimentally proved in Section~\ref{sc:ce}.
Note that this problem does not happen with mutual information, because swapping the labels does not affect  $I(X;Y)$ at all.}

\changed{Since $G$ can only change the mechanism  $P_{Z\mid W}$, 
the only way for $G$ to reduce the mutual information $I(X;Y)$ is to reduce $I(X;Z)$
by reducing the correlation between $W$ and $Z$ ($X$ is correlated to $Z$ only via $W$) . This limits the 
information about $X$ that can be inferred from  $Z$, for any possible adversary, i.e., for any possible prediction 
$P_{Y\mid Z}$, hence also for the optimal one. Still, if $Z$ is very large  $I(X;Z)$ cannot be reduced directly in an efficient way, 
and this is the reason why $G$ needs the feedback of the optimal prediction $P_{Y\mid Z}$: in contrast to $I(X;Z)$, 
minimizing $I(X;Y)$ can be done effectively in neural networks via the gradient descent 
when $\calx$ (the domain of $X$ and $Y$) is ``reasonably  small''.} 

The above discussion about $I(X;Y)$ vs $\CE(X,Y)$ holds for the generator $G$, but what about the adversary  $C$? 
Namely, for a given $P_{Z\mid W}$, is it still necessary to train $C$ on $I(X;Y)$, or could we equivalently train it on $\CE(X,Y)$? 
The following result
% states that for $C$ it is indeed indeed possible to use $\CE(X,Y)$}.
answers this question positively.
\begin{restatable}{proposition}{propMICE}\label{prop:propMICE}
%	For a \changed{given} $G$ (i.e., a \changed{given} $P_{Z\mid W}$) 
%	we have: 
%	\[
%		\argmax_\mathit{C} \; I(X;Y)
%		=
%		\argmin_\mathit{C} \; \CE(X,Y)
%		~.
%	\]
\[
		\argmin_\mathit{G}  \max_\mathit{C}\; I(X;Y)
		=
		\argmin_\mathit{G}  I(X,Y')
		~,
\]
	\changed{with $Y^\prime$   defined by 
	$P_{Y^\prime\mid Z} = {\displaystyle \argmin_\mathit{C}} \; \CE(X,Y^\prime) = P_{X\mid Z}$}.
%	\changed{Furthermore, the optimal prediction  for $C$ is 
%	$P_{Y\mid Z} = \argmin_\mathit{C} \; \CE(X,Y) = P_{X\mid Z}$}.
\end{restatable}

\changed{Given the above result, and since minimizing $\CE(X,Y)$ is more efficient than maximizing  $I(X;Y)$, in our implementation we have used 
$\CE(X,Y)$ for the training of $C$. Of course, we cannot do the same  for $G$:  as discussed above, the generator needs to be trained  by using $I(X;Y)$.}

\changed{A consequence of \autoref{prop:propMICE} is that the adversary represented by $C$ at the  point of equilibrium 
is at least as strong as the Bayesian adversary, namely the adversary that minimizes the expected probability of error in the $1$-try attack 
(which consists in guessing a single secret $x$ given a single observable $z$~\cite{Smith:09:FOSSACS}.)
Indeed,  from $P_{Y\mid Z}$ one can derive the following decision function (deterministic classifier) $f^*:\calz\rightarrow \calx$, which assigns to any $z$ the class $y$ with highest predicted probability:
\begin{equation}\label{eqn:bestprediction}
f^*(z) \;=\; \argmax_y P_{Y\mid Z}(y|z)
\end{equation}
To state formally the property of the optimality of $f^*$ w.r.t. $1$-try attacks, let us recall the definition of the expected error $R(f)$ for a generic decision function $f:\calz\rightarrow \calx$:
\begin{equation}\label{eqn:error}
R(f) \;=\; \sum_{xz} P_{X,Z}(x,z) \overline{\mathbb{1}}_f(x,z)
\end{equation}
where
\begin{equation}\label{eqn:one}
\overline{\mathbb{1}}_f(x,z) \;=\; \begin{cases}
  1 \quad \text{if } f(z) \neq x\\    
  0 \quad \text{otherwise}  
\end{cases}
\end{equation}
We can now state the following result, that relates the error of the attacker $f^*$ (induced by the $C$ at the equilibrium point) and the minimum Bayes error of any adversary for the $G$ at the equilibrium point (cfr. \autoref{def:BayesError} and \autoref{prop:LowerBound}):
\begin{restatable}{proposition}{propCEB}\label{prop:propCEB}
If 
$P_{Y\mid Z} = \argmin_\mathit{C} \; \CE(X,Y)$, and $f^*$ is defined as in \eqref{eqn:bestprediction}, then:
\[
R(f^*) = B(X,Z)
\]
\end{restatable}
}

\subsection{Implementing Mutual Information}
% In order to make the function work with the back propagation algorithm, we implemented all the steps using TensorFlow and Keras native functions: summation, multiplication, logarithm, condition, while loop.

In order to describe the implementation of the mutual information loss
function, we will consider the training on a specific batch of data. 
\changed{This technique}  is based on the idea that the whole training
set of cardinality $N$ can be split into subsets of cardinality $N^\prime$
with $N^\prime \leq N$. This is useful to to fit data in the memory and,
since during each epoch the network is trained on all the batches, this
corresponds to using all the training data (provided that the data
distribution in each batch is a high fidelity representation of the training
set distribution, otherwise the learning could be unstable).

\changed{To obtain the mutual information between $X$ and $Y$ we estimate the distributions 
$P_X$, $P_Y$ and $P_{X,Y}$. Then we can compute $I(X;Y)$ using \eqref{eq:MI}, or equivalently as the formula:
\begin{equation}
\small
%I(X;Y) =
\sum_{x}P_X(x)\log  P_X(x) -\sum_{x,y}P_{X,Y}(x,y)\log \frac{P_{X,Y}(x,y)}{P_Y(y)} .
\end{equation}
}
%, 
%where $P_X(x)$ represents the probability of having $x$ as
%target class and $P_Y(y)$ represents the probability of having $y$ as
%predicted class.

\changed{Let us consider a batch consisting of $N^\prime$ samples of type $(z,x)$ in the context of  the
classification problem, and let   $| \calx |$ represents the cardinality of $\calx$, 
i.e., the total number of classes.
In the following we denote by  $\mathit{T}$ and $\mathit{Q}$, respectively, the target and the prediction matrices  for the batch. 
Namely, $\mathit{T}$ and $\mathit{Q}$ are $N^\prime\times \changed{ | \calx |}$ matrices, whose rows correspond to  samples and whose columns to classes, 
defined as follows. 
$T$  represents the class one-hot encoding:  the element in row $i$ and column $x$, $T(i, x)$, is  $1$ if $x$  is the target class for the sample $i$, and  $0$ otherwise. 
$\mathit{Q}$, on the other hand, reports the probability distribution over the
classes computed by the classifier: $Q(i, x)$ is the predicted probability   that   sample $i$ be in class $x$.}

\changed{The estimation of $P_X(x)$ for the given batch can be obtained by computing the frequency of $x$ among the samples, namely:
\begin{equation}
P_X(x) = \frac{1}{ N^\prime}\sum_{i=1}^{N^\prime}T(i, x).
\end{equation} 
Similarly,  $P_Y(y)$ is estimated as the expected prediction of $y$:  
\begin{equation}
P_Y(y) = \frac{1}{ N^\prime}\sum_{i=1}^{N^\prime}Q(i, y).
\end{equation}
}

\changed{The joint distribution $P_{X,Y}$  can be estimated by considering the correlation of $X$ and $Y$ through the samples. Indeed, 
the probability that sample $i$ has target  class $x$ and  predicted class $y$ can be computed as the product   $T(i,x) \, Q(i,y)$, and by summing up the 
contributions of all samples (where each sample contributes for 
$\nicefrac{1}{N^\prime}$) we obtain $P_{X, Y}(x,y)$. 
%Since each sample in the batch corresponds to a row in $\mathit{T}$ and in $\mathit{Q}$, these probabilities are represented by the tensor product of the two rows in   $\mathit{T}$ and   $\mathit{Q}$ corresponding to that sample.
}

\changed{
More precisely, 
for a  sample $i \in \{1, ..., N^\prime\}$ let us define the 
$ | \calx | \times  | \calx |$ matrix $\mathit{J_i}$ as 
$J_i (x,y) =T(i,x) \, Q(i,y)$.
Then we can estimate $P_{X,Y}(x,y)$ as:
%Namely,  $J_i$ represents the matrix of the joint distribution of the target and predicted classes for the sample $i$. 
%%(Each row corresponds to a target class and each column corresponds to a predicted class.)
%Note that, since $T$ represents the one-hot encoding, only the row of $J_i$ 
%corresponding to the target class  for $i$ in $T$ will have elements different from zero.
%$P_{X,Y}$ for the batch can then be estimated 
%as:
%by averaging the values on all the $N^\prime$ samples:
\begin{equation}
P_{X,Y}(x,y) =
\frac{1}{ N^\prime} \sum_{i=1}^{N^\prime}  J_i (x,y).
\end{equation}
}

\changed{
The estimation of the mutual information   relies on the estimation of the probabilities, which is based on the computation of the frequencies. 
Hence, in order to obtain a good estimation, the batches should be large enough to  represent well  the true distributions. 
Furthermore, if the batch size is too small, the gradient descent is   unstable since the representation of the   distribution changes from one batch to the other. 
In the ML literature there are standard validation techniques (such as the \emph{cross validation}) that provide guidelines  
%on the size of the 
%batches necessary 
to achieve a  ``good enough'' estimation of the probabilities. }

\subsection{Base models}

The base model $C_0$ is simply the ``blank'' classifier that has not learnt anything yet \mods{(i.e. the weights are initialized according to the Glorot initialization, \changed{which is a  standard initialization technique~\cite{Glorot:10:PMLR}})}. As for $G_0$, we have found out experimentally that it is convenient to start with a noise function pretty much spread out. 
This is because in this way the generator has more data points with non-null probability to consider, and can figure out faster which way to go to minimize the mutual information. 

\subsection{Utility}
\label{utility_implementation}
The utility constraint is incorporated in the loss function of $G$ in the following way: 
	\begin{equation}
	\mathit{Loss}_{\mathit{G}} = \alpha \times \mathit{Loss}_\mathit{utility} + \beta \times I(X;Y),
	\end{equation}
where $\alpha$ and $\beta$ are parameters that allow us to tune the trade-off between utility and privacy. 
The purpose of $\mathit{Loss}_\mathit{utility}$  is to ensure that the constraint on utility is respected, i.e., that  the obfuscation mechanism that $G$ is trying to produce 
stays within the domain $M_L$. 
We recall that  $M_L$ represents the constraint  $\loss[Z\mid W]  \leq L$ (cfr.~\eqref{eq:UtilityConstraint}).    
Since we need to compute the gradient on the loss,  we need a derivable function for $\mathit{Loss}_\mathit{utility}$. 
We propose to implement it using $\mathrm{softplus}$, which is a function of two arguments in $\mathbb{R}$ defined as:
%	\begin{equation}
$
	\mathrm{softplus}(a,b) = \ln(1+e^{(a-b)}).
$
%	\end{equation}
	This function is non negative, monotonically increasing, and its value is close to $0$ for $a<b$, while it  grows very quickly for $a>b$. 
Hence, we define
	\begin{equation}
	\mathit{Loss}_\mathit{utility} (P_{Z|W})= \mathrm{softplus}(\loss[Z\mid W],L).
	\end{equation}
With this definition, $\mathit{Loss}_\mathit{utility}$ does not interfere with $I(X;Y)$ when the constraint $\loss[Z\mid W]  \leq L$ is respected, and 
it forces $G$ to stay within the constraint because its growth when the constraints is not respected is very steep.

\subsection{On the convergence of our method}
\Kt{Clean this up, the same things are said elsewhere}
In principle, at a each iteration $i$, our method relies on the ability of the network $\mathit{G}_i$ to improve the obfuscation mechanism starting from the one produced by $\mathit{G}_{i-1}$, and given only the original locations and the model $\mathit{C}_{i}$, which are used  to determine the direction of the gradient for $\mathit{Loss}_{\mathit{G}}$. The classifier 
$\mathit{C}_{i}$ is a particular adversary modeled by its weights and its biases. However, thanks to the fact that  the main component of 
$\mathit{Loss}_{\mathit{G}}$ is $I(X;Y)$ and not the the cross entropy, $\mathit{G}_i$ takes into account all the attacks that would be possible from $\mathit{C}_{i}$'s information. 
We have experimentally verified that indeed, using the mutual information rather than the cross entropy, determines a substantial improvement on the convergence process, and the resulting mechanisms provide a better privacy (for the same utility level). Again, the reason is that the the cross entropy would be subject to the ``swapping effect'' illustrated by Example~\ref{exa:AliceandBob} in the introduction. 

Another improvement on the convergence is due the fact that, as explained before,  we reset the classifier to the initial weight setting ($\mathit{C}_{0}$) at each iteration, instead than letting 
$\mathit{C}_{i}$ evolve from $\mathit{C}_{i-1}$. 
%We have experimentally verified that the precision of $\mathit{C}_{i}$ improves more rapidly if we start from an ``agnostic'' situation rather than from the knowledge accumulated in previous steps. This is, intuitively, due to the fact that the noise produced by  $\mathit{G}$ may change substantially from an iteration to the next, hence $\mathit{C}_{i}$  would have to ``unlearn'' the now obsolete information of $\mathit{C}_{i-1}$, if the latter were its starting point. 

The  function that  $G$ has to minimize, $\mathit{Loss}_{\mathit{G}}$,
%is a combination of a softplus for the utility loss and the mutual information, and it 
is convex wrt   $P_{Z\mid W}$. This means that there are only global minima, although there can be many of them, all equivalent. Hence for sufficiently small updates the noise distribution modeled by $P_{Z\mid W}$ converges to one of these optima, provided  that the involved network has enough capacity to compute the gradient descent involved in the training algorithm. In practice, however, the network $\mathit{G}$ represents a limited family of noise distributions, and instead of optimizing the noise distribution itself we optimize the weights of this network, which introduces multiple critical points in the parameter space.

\subsubsection*{Number of epochs and  batch size}

The convergence of the game can be quite sensitive to the number of epochs and  batch size. 
We just give two hints here, 
referring to literature~\cite{Keskar:18:CoRR} for a general discussion about the impact they have on learning. 
% \begin{itemize}

First, choosing a  batch   too small for training $\mathit{G}$ might result in \mods{ too strict a constraint on the utility}. In fact, since the utility loss   is  an expectation,
 a  larger number of samples makes it more likely that some points are pushed further than the threshold, taking advantage of the fact that their loss may be compensated by other data points for which the loss is small.  

Second, training $\mathit{C}$ for too few epochs might result into a too weak adversary. On the other hand if it is trained for a long time we should make sure that the classification performances do not  drop over the validation and test set because that might indicate an overfitting problem. 
% \end{itemize}
% !TEX root = main.tex
\section{Cross Entropy vs Mutual Information: demonstration on synthetic data}
\label{synthetic_experiments}

In this section we perform experiments on a   synthetic dataset to
obtain an intuition about the behaviour of our method. The dataset is
constructed with the explicit purpose of being simple, to facilitate the
interpretation of the results. The main outcome of these experiments is
confirming the fact that, as discussed in Sec~\ref{sec:mi-vs-ce},
training the generator $G$ wrt \emph{cross entropy is not sound}. Even in our simple
synthetic case, training $G$ with $\CE(X,Y)$ as the loss function fails to
converge: $G$ is just ``moving points around'', temporarily
fooling the \emph{current} classifier, but failing to really hide the
correlation between the secrets and the reported locations.

On the other hand, training $G$ with mutual information behaves as expected: the resulting network generates
noise that mixes all classes together, making the classification problem hard
for \emph{any} adversary, not only for the \emph{current} one.
Note that cross entropy is still used, but only for $C$ (cfr. Sec~\ref{sec:mi-vs-ce}).

\subsubsection*{The dataset}
\label{synthetic_dataset}
We consider a simple location privacy problem; $4$ users $\calx = \caly =
\{blue, red, green, yellow\}$ want to disclose their location while
protecting their identities.
Both the real locations $\calw$ as well as the reported locations $\calz$
are taken to be
all locations in a squared region of  $6.5\times 6.5$ sq km centered in 5, Boulevard de S\'ebastopol, Paris. Each location entry is defined by a pair of coordinates normalized in $[-1, 1]$.

The synthetic dataset consists of $600$ real locations for each of the $4$ users (classes),
for a total of $2400$ entries. 
The locations of each user are  placed around one of the vertices
of a square of $300\times 300$ sq meters centered in 
$5$, Boulevard de S\'ebastopol, Paris. (Each user corresponds to a different vertex.)
They are randomly generated  so to form a  cloud of $600$ entries around each vertex  % ($480 + 120$) 
and in such a way that no locations falls further than about $45$m from the corresponding vertex.  
These sets are represented in
Fig.~\ref{fig:synth_ce} ((a) and (b), left): it is evident from the figure that the four
classes are easily distinguishable; without noise a linear classifier could predict the
class of each location with no error at all.

% These distributions determine the random variables $X$ and $W$, and their correlation.

%For each user, $480$ locations will be used for training and validation
%purposes, and $120$ will be used for testing.

Of the total $2400$ entries of the dataset we use
$1920$ for training and validation ($480$ for each user) and 
$480$ for testing ($120$ for each user).

% The goal is to produce an obfuscation mechanism that achieves a good  trade-off between utility and privacy. 
% Though the various experiments, we   compare our method with the planar Laplace mechanism 
% used in \emph{geo-indistinguishability}~\cite{Andres:13:CCS}, which is    popular in the domain of location privacy. 

\subsubsection*{Network architecture}
A relatively simple architecture is used for both $\mathit{G}$ and $\mathit{C}$ networks.
They consist of three fully connected hidden layers of
neuron with ReLU function. In particular   $\mathit{C}$   has 60, 100
and 51 hidden neurons respectively in the first, second and third hidden
layers.
% The architecture parameters are the result of a tuning involving 10
% fold cross validation with the objective of reducing the classification loss
% over the original locations dataset in \ref{Govalla_experiments}.
The
$\mathit{G}$ network has 100 neurons in each hidden layer; such an
architecture has proved to be enough to learn how to reproduce the Laplace
noise distribution ($\epsilon=\ln(2)/100$) with a negligible loss.

\subsubsection*{Bayes error estimation}%\label{bayes_error}
As explained in Section~\ref{sec:oursetting}, we use the
Bayes error $B(X\mid Z)$ to evaluate the level of protection offered by a mechanism. 
To this purpose, we discretize $\calz$ into a grid over the 
$6.5\times 6.5$ sq km region, thus determining a partition of the region into a number of disjoint \emph{cells}. 
We will create different grid settings to see how the partition affects  the Bayes error.
In particular, we will consider the cases where the side of a cell is 
$25$m, $50$m, $100$m and $500$m  long, which corresponds to  
$260\times 260 = 67600$, $130\times 130=16900$, $65\times 65= 4225$ and $13\times 13 = 169$ cells, respectively.

We   run   experiments with different numbers of obfuscated locations (hits).
% created for each original one. 
Specifically, for each grid  we  consider  
 $10$, $100$, $200$ and $500$ obfuscated hits for each original one.

Each  hit falls in exactly one  cell. Hence, we  can  estimate the probability that a hit is in cell $i$ as:
\begin{equation}
P(cell_{i}) \; = \; \frac{\text{number of hits in }cell_{i}}{\text{total number of hits}},
\end{equation}
and the probability that a hit  in cell  ${i}$  belong to class $j$:
\begin{equation}
P(Class_{j}|cell_i) \; = \; \frac{\text{number of hits of }class_ j \text{ in }cell_{i}}{\text{number of hits in } cell_{i}},
\end{equation}
We can now estimate of the Bayes error as follows:
\begin{equation}
B( X\mid Z) = 1 - \sum_{i=0}^{k-1} \max_{j}\, P(Class_{j}|cell_i)P(cell_{i})
\end{equation}
where $k$ is the total number of cells. 

Note that   these computations are influenced by the chosen grid. In particular we have two extreme cases:
\begin{itemize}
\item when the grid consists of only one cell the Bayes error is $1-\nicefrac{1}{k} = \nicefrac{k-1}{k}$ for any obfuscation mechanism $P_{Z\mid W}$. 
\item when the number of cells is large enough  so that each cell  contains at most one hit, then the Bayes error is $0$ for any obfuscation mechanism. 
\end{itemize}
In general, we expects   
%a coarser granularity 
%increases the Bayes error, 
%for any   mechanisms. 
a finer granularity to give higher discrimination power and to decrease  the Bayes error,
especially with methods that scatter the obfuscated locations far away. 
%in different regions of the domain. 
%Besides, it is important to consider that a particular choice for the cell size might introduce Bayes error even before the obfuscation mechanisms are applied to the original data.
 
%We show results which are relative to the Bayes error measurement for all the obfuscation methods. Moreover we compare the obfuscation effects on the Bayes error for both training and testing data, so to asses the methods' performance quality on new data never seen at training time. 
We estimate the Bayes error on the testing data in order to evaluate how well the obfuscation mechanisms protect new data samples never seen during the training phase. Moreover we evaluate the Bayes error on the same data we used for training and we compare the results with those obtained for the testing data. We notice that, in general, the difference between the two results is not large, meaning that the deployed mechanisms efficiently protect the new samples as well.
%Indeed the method is almost as effective on the new data as in the one used in training. If this is not the case and the obfuscation looks much more effective on the training samples than on the testing ones, it means that the deployed mechanism has adapted too much to the empirical distribution of the training samples and it does not generalize to new data.

\subsubsection*{The planar Laplace mechanism}%~\label{sec:PL}
We compare our method against the  
planar Laplace  mechanism\cite{Andres:13:CCS}, whose
 probability density to report $z$, when the true
location is $w$, is:
\begin{equation}
   \label{eqn:Laplace}\call^\epsilon_{w}(z)=\frac{\epsilon^2}{2\pi}\, e^{-\epsilon\,d(w,z)}\, ,
\end{equation}
where $d(w,z)$ is the Euclidean distance between  $w$ and $z$.

In order to compare the the Laplace mechanism with ours, we  need to tune the privacy parameter $\epsilon$ so that 
the expected distortion of $\call^\epsilon$ is  the same as the upper bound on the utility loss applied in our method, i.e. $L$. 
To this purpose, we recall that the expected distortion $\loss[Z\mid W]$ of  the  planar Laplace depends only on $\epsilon$ (not on the prior $P_W$), and it is given by:  
\begin{equation}\label{eqn:distortionLaplace}
   \loss[Z\mid W]
   % \;{=} \sum_{wz} P_W(w) \call_{w}(z) d(w,z)
   \;=\; \frac{2}{\epsilon}.
\end{equation}

\subsection{Experiment 1: relaxed  utility constraint}\label{sec:sr}

As a first experiment, we choose for the upper bound $L$ on the expected
distortion a value high enough so that in principle we can achieve the
highest possible privacy, which is obtained when the observed obfuscated location
gives no information about the true location, which
means that $I(X;Y) = 0$. In this case,  the attacker can only do
random guessing. Since we have $4$ users, the Bayes error is $B(X\mid Y) =
1- \nicefrac{1}{4} = 0.75$.

For the distortion, we take $\ell(w,z)$ to be the geographical distance between $w$ and $z$.
One way to achieve the maximum privacy is to map all locations into the middle point. To compute a sufficient $L$, note that the vertices of the original locations   form a square of  side $300$m, hence each vertex is at a distance $300\times \nicefrac{\sqrt{2}}{2} \approx 212$m from the center. Taking into account that the locations can be as much as $45$m away from the corresponding vertex, we conclude that any value of $L$ larger than $ 212 + 45 = 247$m should  be enough to allow us to obtain the maximum privacy. We set the upper bound on the distortion a little higher: 
\begin{equation}\label{eqn:upperbound}
L\;= \;270\mbox{m},
\end{equation}
but we will see from the experiments that a much smaller value of $L$ would have been sufficient. 

We now need to tune the planar Laplace so that the expected distortion is at least $L$. We decide to set:
\begin{equation}
\epsilon= \frac{\ln 2}{100}
\end{equation}
which, using Equation~\eqref{eqn:distortionLaplace}, gives us a value 
\begin{equation}\label{eqn:distortionLsr}
\loss[Z\mid W] \approx 288\mbox{m} > L.
\end{equation}

We have   used this instance of the planar Laplace also as a starting point of our method:  we have defined $G_0$ as $\call^\epsilon$ with $\epsilon= \nicefrac{\ln 2}{100}$.
For the next steps,  $\mathit{G}_i$  and  $\mathit{C}_i$ are constructed as explained in 
 Algorithm~\ref{alg:game}.
In particular, we train the generator with a batch size of $128$ samples for $100$ epochs during each iteration. The learning rate is set to $ 0.0001$. For this particular experiment we set the weight for the utility loss to $1 $ and the weight for the mutual information to $2$.
The classifier   is trained with a batch size of $512$ samples and $3000$ epochs for each iteration. The learning rate for the classifier is set to $0.001$. 

\begin{figure*}[t]
   \begin{subfigure}[b]{0.5\textwidth}
      \includegraphics[trim={0 120pt 0pt 130pt},clip,width=\textwidth]{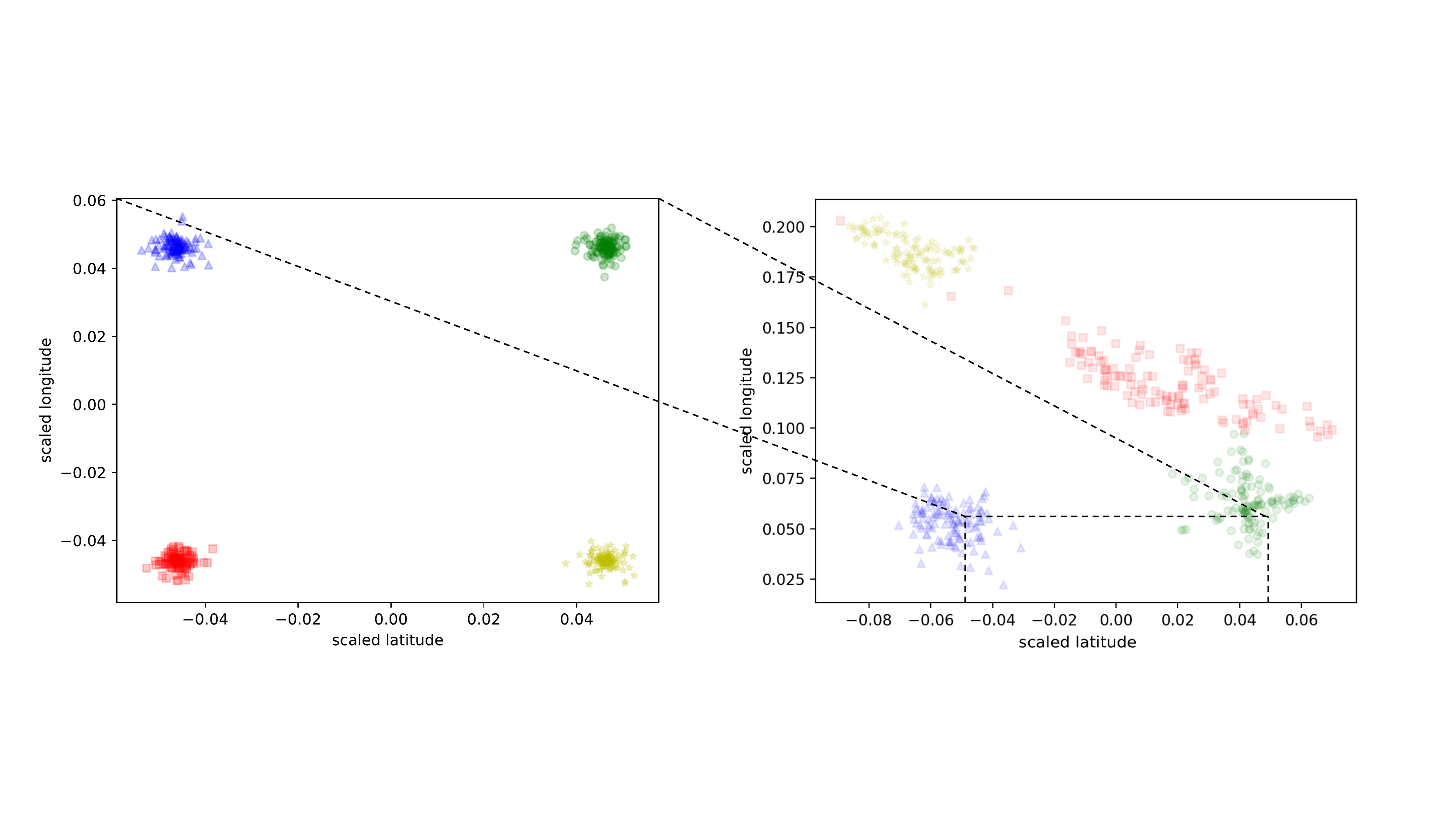}
      \caption{\footnotesize Iteration 30}
   \end{subfigure}
   \begin{subfigure}[b]{0.5\textwidth}
      \includegraphics[trim={0 120pt 0 130pt},clip,width=\textwidth]{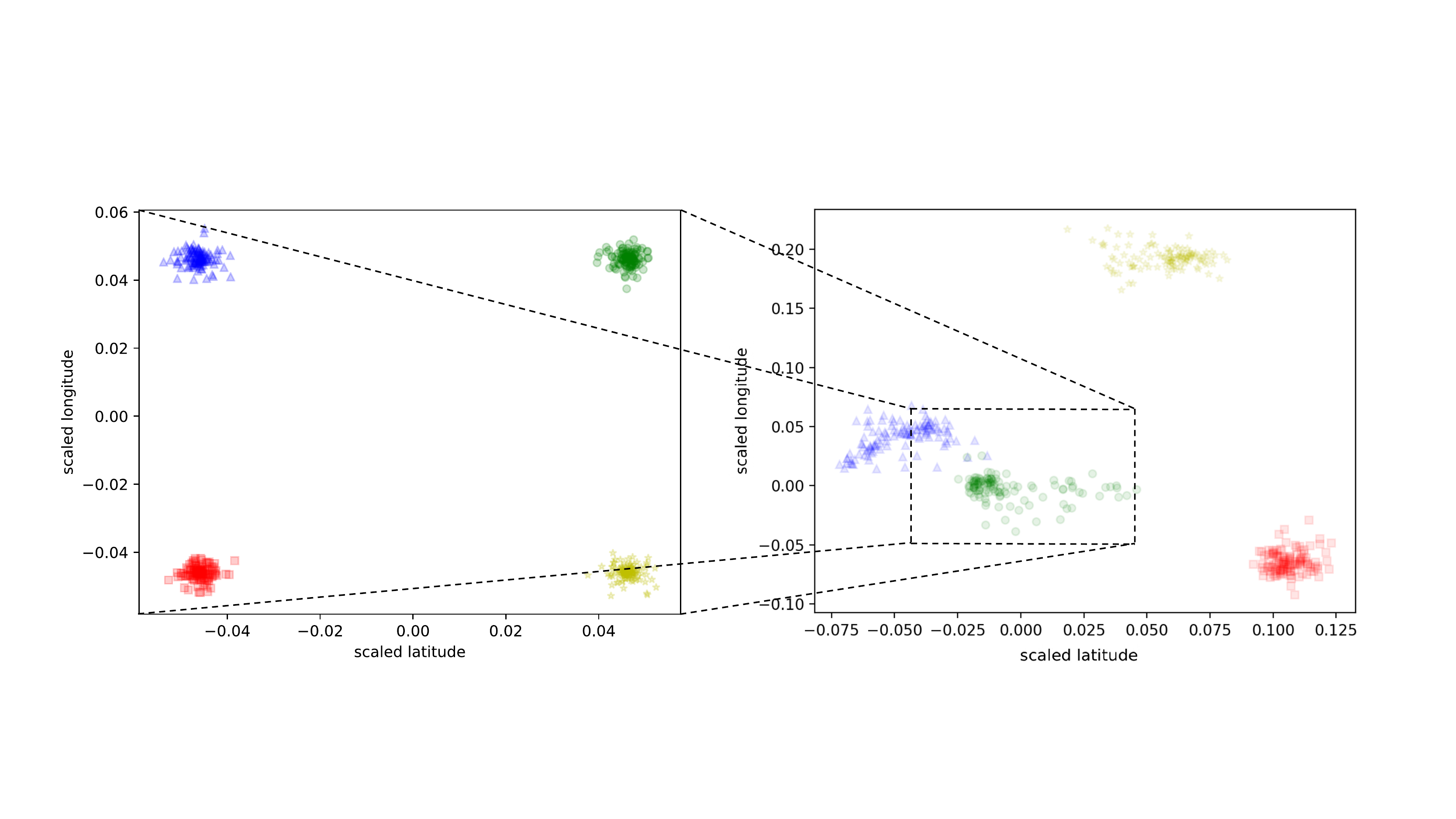}
      \caption{\footnotesize Iteration 40}
   \end{subfigure}
   \caption{\footnotesize Using cross entropy for producing the noise does not make the system converge. The left sides of  Figures (a) and (b)  show the original synthetic  data without noise. The right sides show  the noisy data at different iterations. $L= 270$m.}
   \label{fig:synth_ce}
\end{figure*}

\begin{figure*}[t]
   \centering
   \begin{tikzpicture}[every node/.style={node distance=6cm, inner sep=0pt, outer sep=0pt}]
   \node (img1)[label={[font=\small, black, rotate=0]below:(a)}] {\includegraphics[width=0.33\textwidth]{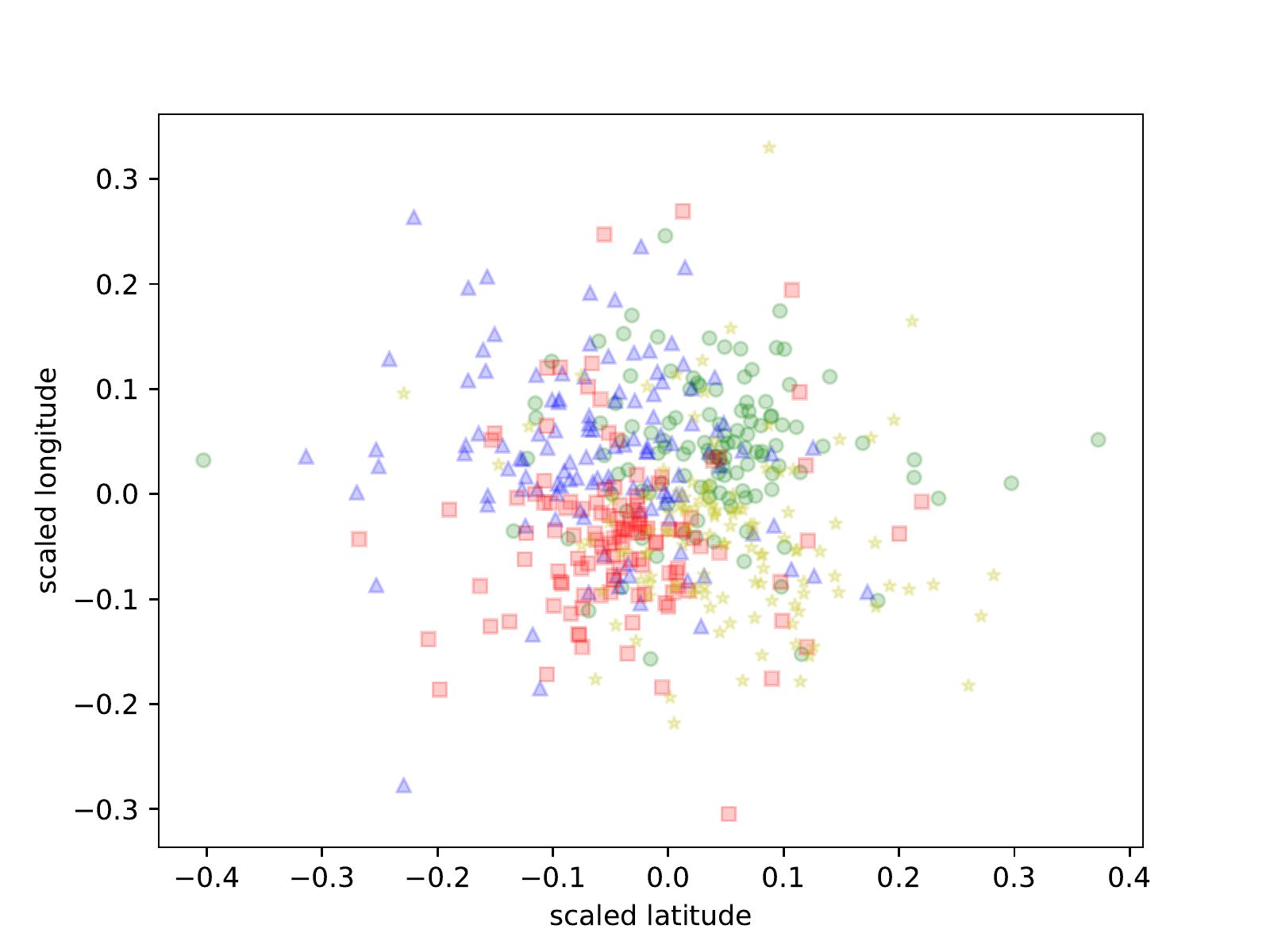}};
   %\node (img1)[label={[font=\small, black, rotate=0]below:(a)}] {\includegraphics[width=0.33\textwidth]{img/for_catuscia/plot_synthetic_data/synthetic_data_laplacian_noise_test.pdf}};
   \node[right of=img1,label={[font=\small, black, rotate=0]below:(b)}] (img2) {\includegraphics[width=0.33\textwidth]{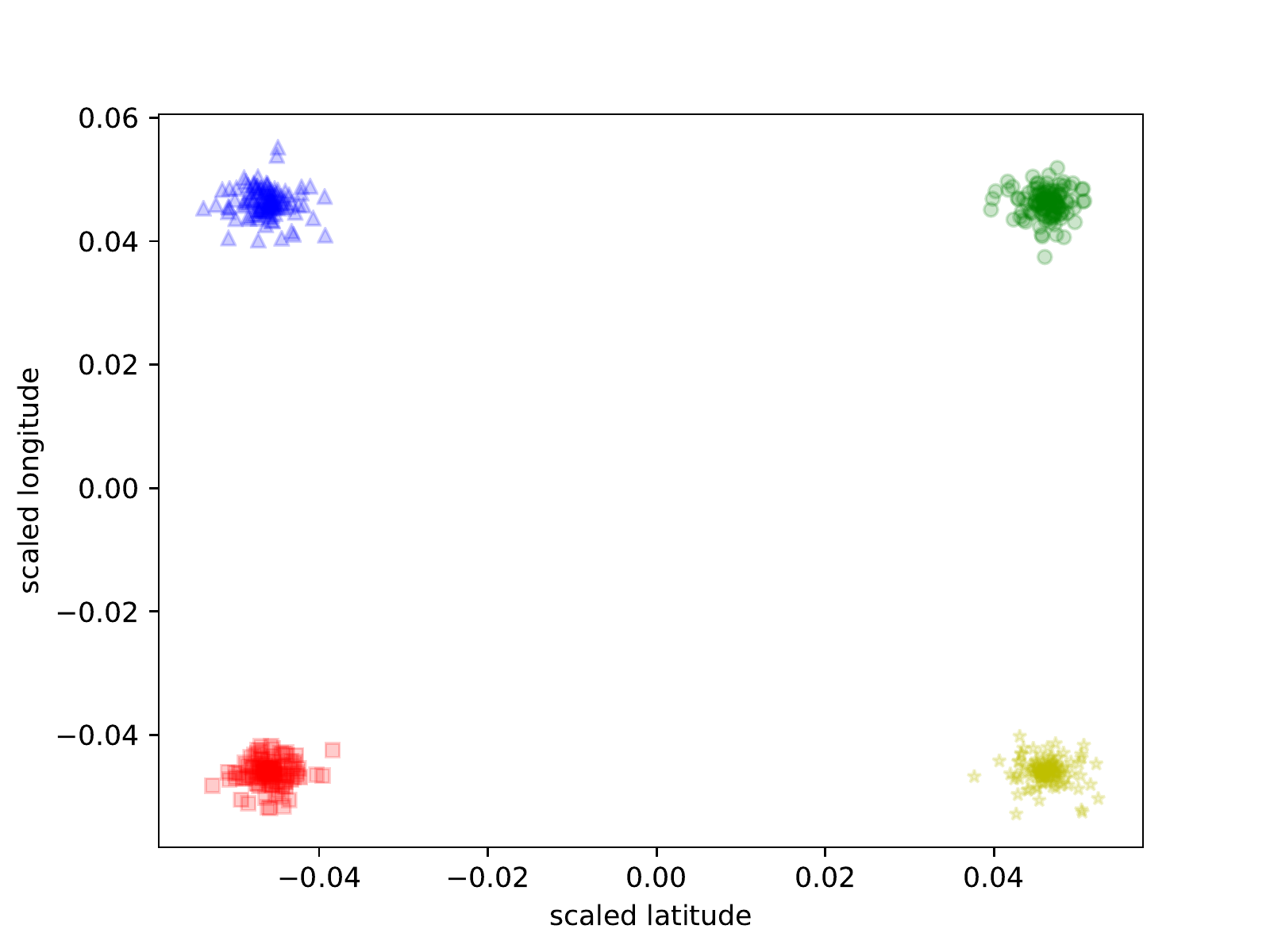}};
   %\node[right of=img1,label={[font=\small, black, rotate=0]below:(b)}] (img2) {\includegraphics[width=0.33\textwidth]{img/for_catuscia/plot_synthetic_data/synthetic_data_no_noise_test.pdf}};
   \node[right of=img2,label={[font=\small, black, rotate=0]below:(c)}] (img3) {\includegraphics[width=0.33\textwidth]{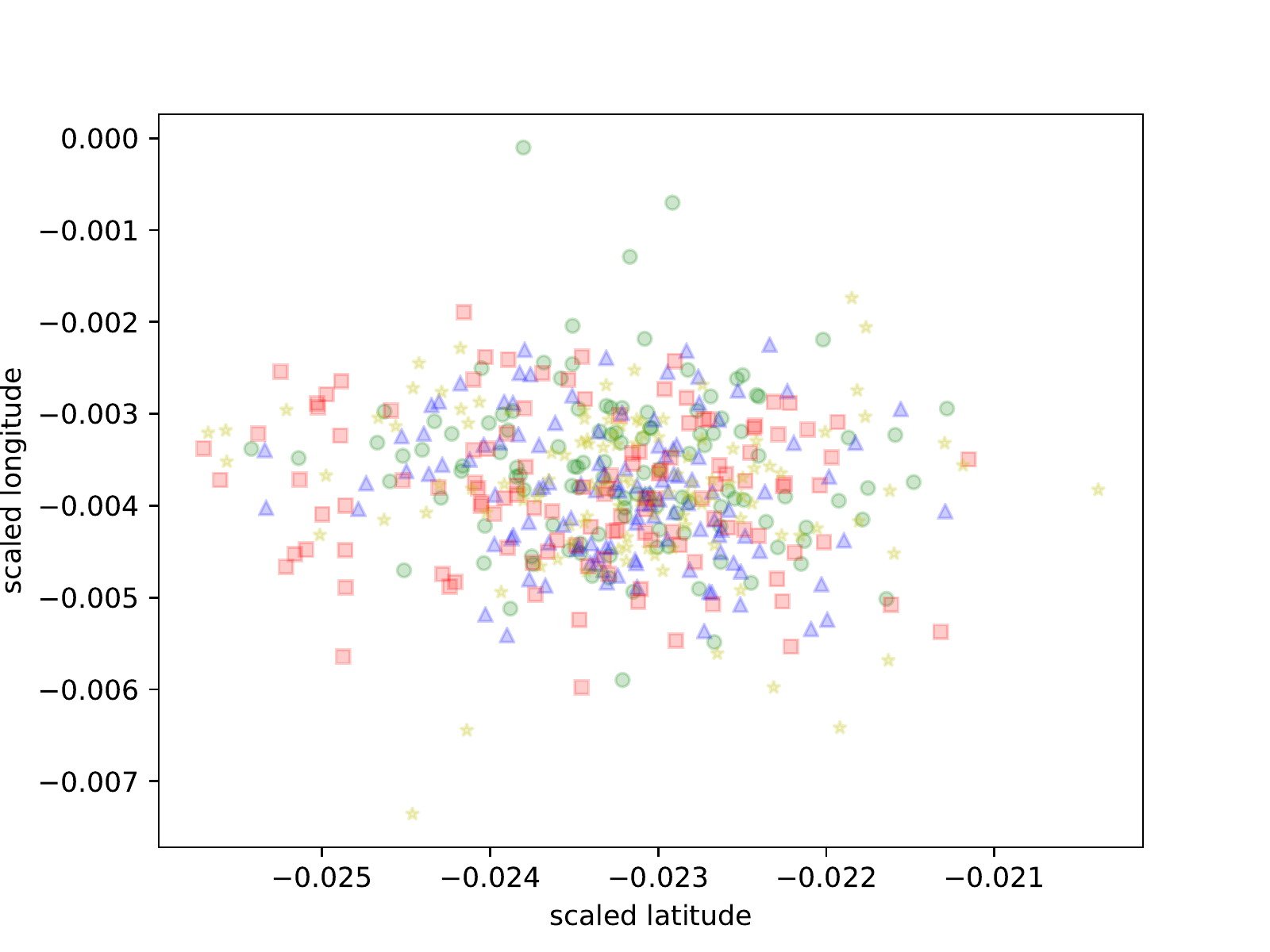}};
   \draw[dashed] (-0.15,0.19) -- (3.7,1.7);
   \draw[dashed] (-0.15,-0.35) -- (3.7,-1.74);
   \draw [dashed] (0.44,0.19) -- (8.4,1.7);
   \draw [dashed] (0.44,-0.35) -- (8.4,-1.75);

   \draw[dashed] (-0.15,0.19) -- (-0.15,-0.35);
   \draw[dashed] (-0.15,-0.35) -- (0.44,-0.35);
   \draw [dashed]  (0.44,0.19) -- (0.44,-0.35);
   \draw [dashed]  (0.44,0.19) -- (-0.15,0.19);

   \draw [dashed](5.07,-0.12) -- (9.8,1.7);
   \draw [dashed](5.07,-0.2756) -- (9.8,-1.75);
   \draw [dashed](5.3,-0.12) -- (14.4,1.68);
   \draw [dashed](5.3,-0.275) -- (14.41,-1.74);

   \draw [dashed](5.07,-0.12) -- (5.07,-0.275);
   \draw [dashed](5.07,-0.275) -- (5.3,-0.275);
   \draw [dashed](5.3,-0.12) -- (5.3,-0.275);
   \draw [dashed](5.3,-0.12) -- (5.07,-0.12);
   %\node[right=0pt of img3] (img4) {\includegraphics[height=1cm]{example-image}};
   %\node[right=0pt of img4] (img5) {\includegraphics[height=2cm]{example-image}};
   \end{tikzpicture}
   \caption{\footnotesize  Synthetic testing  data. From left to right: Laplace noise, no noise, our noise produced using mutual information. $L= 270$m.}
   \label{fig:synth_ts_set}
\end{figure*}

\begin{figure*}[t]%
{\footnotesize
%\centering
\begin{subfigure}{1\columnwidth}\centering
%\begin{table}[h]
%\begin{center}
\begin{tabular}{|c|c|c|c|}
\hline
\multicolumn{4}{|c|}{\fix{Number of cells}} \\ 
\hline
$13\times13$ & $65\times65$ & $130\times130$ & $260\times260$\\%\cline{3-10}
\hline
$0.75$ & $0.00$ & $0.00$ & $0.00$\\%\cline{3-10}
\hline
\end{tabular}
%\end{center}
\caption{Training data.}%Experiments on the original version of the data in SynthTV.}
\label{tab:bayes_tr_real_a}
\end{subfigure}\hfill%
%\end{table}
%\begin{table}[h]
\begin{subfigure}{1\columnwidth}\centering
%\begin{table}[h]
%\begin{center}
\begin{tabular}{|c|c|c|c|}
\hline
\multicolumn{4}{|c|}{\fix{Number of cells}} \\ 
\hline
$13\times13$ & $65\times65$ & $130\times130$ & $260\times260$\\%\cline{3-10}
\hline
$0.75$ & $0.00$ & $0.00$ & $0.00$\\%\cline{3-10}
\hline
\end{tabular}
%\end{center}
\caption{Testing data.}
\label{tab:bayes_tr_real_b}
\end{subfigure}}		\caption{\footnotesize  Estimation of $B(X\mid Z)$ on the original version of the synthetic data.}
\label{real_bayes_synt}
\end{figure*}

 \begin{figure*}[t]%
{\footnotesize
%\centering
\begin{subfigure}{1\columnwidth}\centering
%\begin{table}[h]
%\begin{center}
\begin{tabular}{|c|c|c|c|c|c|c|c|c|}
\cline{2-9}
\multicolumn{1}{c|}{} & \multicolumn{8}{c|}{\fix{Number of cells}} \\ 
\cline{2-9}
\multicolumn{1}{c|}{} & \multicolumn{2}{c|}{$13\times13$} & \multicolumn{2}{c|}{$65\times65$} & \multicolumn{2}{c|}{$130\times130$}& \multicolumn{2}{c|}{$260\times260$}\\
%\cline{3-10}
%\hhline{|~|-|-|-|-|-|-|-|-} 
\hline
\fix{Obf} &  \cellcolor{lightgray}Lap & Our &  \cellcolor{lightgray}Lap & Our&  \cellcolor{lightgray}Lap & Our&  \cellcolor{lightgray}Lap & Our\\
\hline
%\multirow{4}{*}{\begin{sideways} \fix{Obf.}~ \end{sideways}} 
$10$ & $\cellcolor{lightgray}0.60$ & $0.75$ &  $\cellcolor{lightgray}0.40$ & $0.75 $ & $\cellcolor{lightgray}0.38$ & $0.75$ &$\cellcolor{lightgray}0.35$ &$ 0.73$  \\  %[5pt]
%\cline{2-10}
\hhline{|-|-|-|-|-|-|-|-|-} 
$100$ &  $\cellcolor{lightgray}0.60$ & $0.75$ &$ \cellcolor{lightgray}0.41$ & $0.75 $& $\cellcolor{lightgray}0.40$ & $0.75$ & $ \cellcolor{lightgray}0.39 $& $0.74$ \\  %[5pt]
\hhline{|-|-|-|-|-|-|-|-|-} 
$200$ & $\cellcolor{lightgray}0.60$& $0.75$ & $\cellcolor{lightgray}0.41 $& $0.75 $&  $\cellcolor{lightgray}0.40$ & $0.75$ &  $\cellcolor{lightgray}0.39 $&$ 0.74$ \\  %[5pt]
\hhline{|-|-|-|-|-|-|-|-|-} 
$500$ & $ \cellcolor{lightgray}0.60$ &$ 0.75 $& $\cellcolor{lightgray}0.41 $& $0.75 $& $ \cellcolor{lightgray}0.40$ &$ 0.75 $&$\cellcolor{lightgray}0.40 $&$ 0.74$ \\  %[5pt]
\hline
\end{tabular}
%\end{center}
\caption{Training data.}
\label{tab:bayes_tr_syn}
\end{subfigure}\hfill%
%\end{table}
%\begin{table}[h]
\begin{subfigure}{1\columnwidth}\centering
%\begin{center}
\begin{tabular}{|c|c|c|c|c|c|c|c|c|}
\cline{2-9}
\multicolumn{1}{c|}{} & \multicolumn{8}{c|}{\fix{Number of cells}} \\ 
\cline{2-9}
\multicolumn{1}{c|}{} & \multicolumn{2}{c|}{$13\times13$} & \multicolumn{2}{c|}{$65\times65$} & \multicolumn{2}{c|}{$130\times130$}& \multicolumn{2}{c|}{$260\times260$}\\
%\cline{3-10}
%\hhline{~|-|-|-|-|-|-|-|-} 
\hline
\fix{Obf} &  \cellcolor{lightgray}Lap & Our &  \cellcolor{lightgray}Lap & Our&  \cellcolor{lightgray}Lap & Our&  \cellcolor{lightgray}Lap & Our\\
\hline
%\multirow{4}{*}{\begin{sideways} \fix{Obf.}~ \end{sideways}} 
$10$ &  $\cellcolor{lightgray}0.59$ & $0.75$ &  $\cellcolor{lightgray}0.38$ & $0.75$ & $\cellcolor{lightgray}0.36$ & $0.75$ & $\cellcolor{lightgray}0.26$ & $0.73$  \\  %[5pt]
\hhline{|-|-|-|-|-|-|-|-|-} 
$100$ &  $\cellcolor{lightgray}0.60$ & $0.75$ &  $\cellcolor{lightgray}0.40$ & $0.75$& $\cellcolor{lightgray}0.39$ & $0.75$ &  $\cellcolor{lightgray}0.37$ & $0.74$   \\  %[5pt]
\hhline{|-|-|-|-|-|-|-|-|-} 
$200$ & $ \cellcolor{lightgray}0.60$ &$0.75$ & $\cellcolor{lightgray}0.41$ & $0.75$ & $\cellcolor{lightgray}0.39$ & $0.75$ & $\cellcolor{lightgray}0.38$ & $0.74$    \\  %[5pt]
\hhline{|-|-|-|-|-|-|-|-|-} 
$500$  & $\cellcolor{lightgray}0.60$ & $0.75$ & $\cellcolor{lightgray}0.41$ & $0.75$& $\cellcolor{lightgray}0.40$ & $0.75$ & $\cellcolor{lightgray}0.39$ & $0.74$  \\  %[5pt]
\hline
\end{tabular}\label{tab:bayes_ts_syn}
%\end{center}
\caption{Testing data.}
%\end{table}
\end{subfigure}%
}		
\caption{\footnotesize Estimation of $B(X\mid Z)$ on   synthetic data for the Laplace and   our mechanisms, with $L=270$m.
The empirical utility loss for training and testing data is $\approx 282.07\mbox{m}-298.40$m respectively for the Laplace and $\approx 219.70\mbox{m}-219.26$m for ours.
The optimal mechanism gives $ {B(X\mid Z)}= 1-\nicefrac{1}{4}=0.75$.} 
\label{tab:bayes_synt}
\end{figure*}

%%%%%%%%%%%%%%%%%%%%%%%%%%%%%%%%%%%%%%%%%%%%%%%%%%%%%%%

\subsubsection{Training $G$ wrt cross entropy}\label{sc:ce}

As discussed in Sec~\ref{sec:mi-vs-ce}, training $G$ wrt $\CE(X,Y)$ is not sound.
This is confirmed in the experiments by the fact that $G$ is failing to
converge. Fig.~\ref{fig:synth_ce} shows the distribution generated by $G$ in
two different iterations of the game. We observe that, trying to fool the
classifier $C$, the generator on the right-hand side has simply moved
locations around, so that each class has been placed in a different
area. This clearly confuses a classifier trained on the distribution
of the left-hand side, however the correlation between labels and location
is still evident. A classifier trained on the new $G$ can infer the labels
as accurately as  before.

As a consequence,  after each iteration, the accuracy of the newly trained
$C_i$ is always $1$, while the Bayes error $B(X|Z)$ is $0$. The generator fails to converge
to a distribution that effectively protects the users' privacy.
We can hence conclude that the use of cross entropy is unsound for training $G$.

\subsubsection{Training $G$ wrt mutual information}

Using now $I(X;Y)$ for training $G$ (while still using the more efficient cross entropy
for $C$, as explained in Sec~\ref{sec:mi-vs-ce}), we observe a totally different behaviour.
After each iteration the accuracy of the classifier drops, showing that the generator
produces meaningful noise.
Around iteration $i=149$ the accuracy of  $\mathit{C}_i$ becomes $\approx 0.25$ both over the training and the validation set. This means that $\mathit{C}_i$ just randomly predicts one of the four classes. We conclude that the noise injection is maximally effective, since $0.75$ is the maximum possible Bayes error.
Hence we know that we can stop. 

The result of our method, i.e., the  final generator $G_i$, to the  testing set is reported in Fig.~\ref{fig:synth_ts_set}(c).
%and training sets, are reported in Fig.~\ref{fig:synth_ts_set}(c) and~\ref{fig:synth_tr_set}(c), respectively. 
The empirical   distortion is $\approx 219.26$m.
%for the training and testing sets is $\approx 219.70$m and $\approx 219.26$m, respectively. 
This is way below the limit of $270$m set in~\eqref{eqn:upperbound}, and it is due to the fact that to achieve the 
optimum privacy we probably do not need more than $\approx 220$m. In fact, the distance of the vertices from the 
center is $\approx 212$m, and even though some locations are further away (up to $ 45$m more), there are also locations that are closer, and that compensate the  utility loss (which is a linear average measure).

For comparison, the result of the application of the planar Laplace to the
testing set is illustrated in Fig.~\ref{fig:synth_ts_set}(a). 
%In the appendix
%we report also the result on the training set
%(Fig.~\ref{fig:synth_tr_set}(a)). 
The empirical distortion (i.e., the
distortion computed on the sampled obfuscated locations) is 
%$\approx 282.07$m for the training data and 
$\approx 298.40$m,
% for the testing data, respectively, 
which is in line with the theoretical distortion formulated
in~\eqref{eqn:distortionLsr}.

From Fig.~\ref{fig:synth_ts_set} 
%and ~\ref{fig:synth_tr_set} 
we can see that, while the Laplace tends to ``spread out'' the obfuscated locations, our method tends to concentrate them into a single point \changed{(mode collapse), i.e., the mechanism is almost deterministic. This  is due to the fact that the utility constraint is sufficiently loose   to allow the noisy locations to be displaced enough so to overlap all in  the same point. When the utility constraint is  stricter, the mechanism is forced to be probabilistic (and the mode collapse does not happen anymore). For example, consider two individuals, $A$ and $B$, in locations $a$ and $b$ respectively, at distance $100$m, assume that  $L=40$m. Assume also, for simplicity, that there are no other locations available. Then the optimal solution maps $a$  into $b$ with probability $\nicefrac{2}{5}$,  and into itself with probability $\nicefrac{3}{5}$ and vice versa for $b$)}.  Nevertheless, we can expect that our mechanism will tend to overlap   the 
obfuscated locations of different classes, as much as allowed by the utility constraint. With the Laplace, on the contrary, the areas of the various classes remain pretty separated. 
This is reflected by the Bayes error estimation  reported in Fig.~\ref{tab:bayes_synt}. 

We note that the Bayes error of the planar Laplace tend to decrease as the grid becomes finer. We believe that this is due to the fact that, with a coarse grid, there is  an effect of confusion simply due to the large size of each cell.  
We remark that the behavior of our noise, on the contrary, is quite stable. 
Note that, when the grid is very coarse ($13\times 13$ cells)
the Bayes error is $0.75$ already on the original data (cfr. Fig.~\ref{real_bayes_synt}), which must be due to the fact that all the vertices are in the same cell. 
While the Bayes error remains $0.75$ also with our obfuscation mechanism, with  Laplace it decreases to $0.60$. 
The reason is that the noise scatters the locations in different cells, and they become, therefore distinguishable.   

\begin{figure*}[t]
   \centering
   \begin{tikzpicture}[every node/.style={node distance=6cm, inner sep=0pt, outer sep=0pt}]
   \node (img1)[label={[font=\small, black, rotate=0]below:(a)}] {\includegraphics[width=0.33\textwidth]{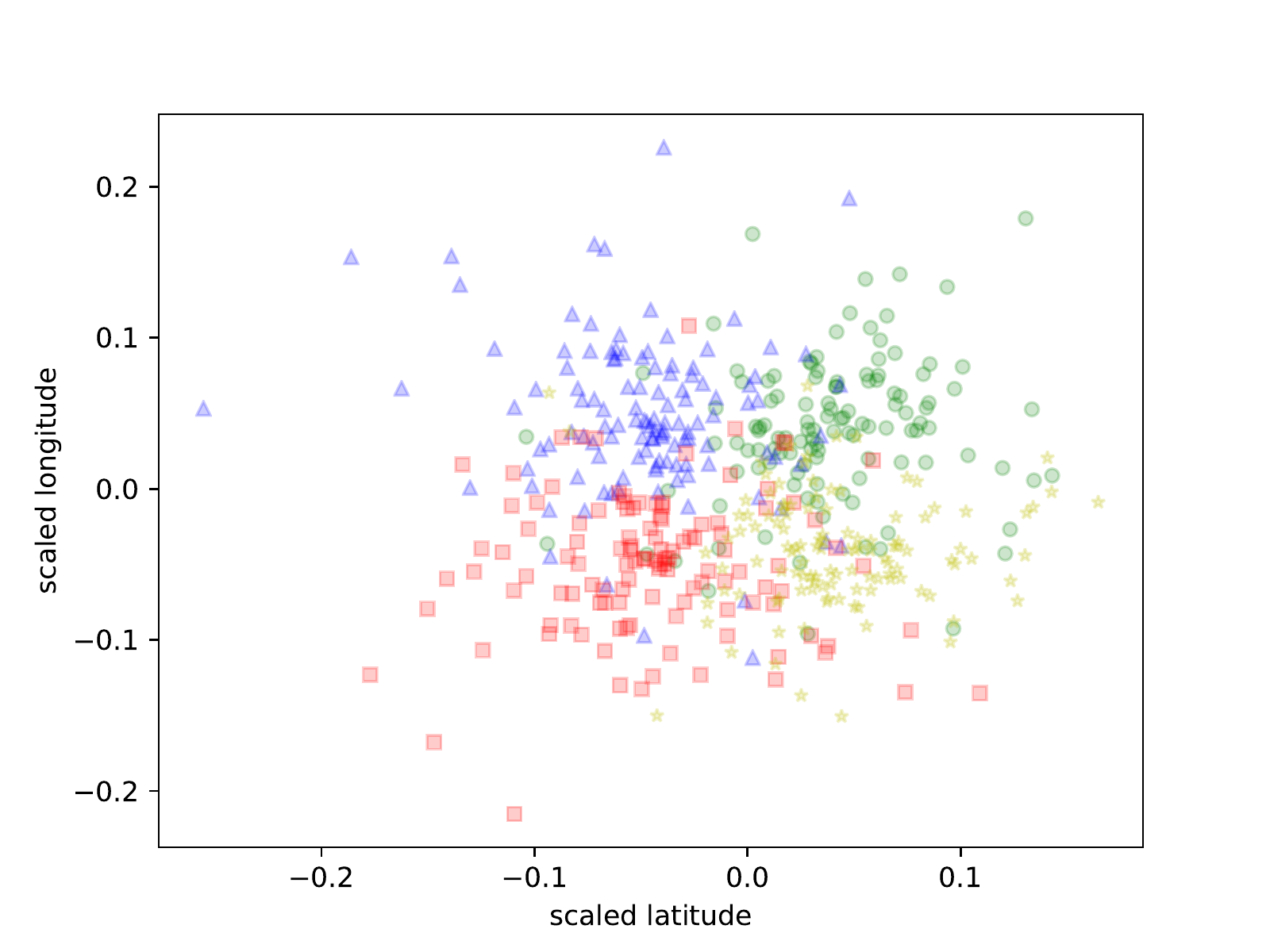}};
   \node[right of=img1,label={[font=\small, black, rotate=0]below:(b)}] (img2) {\includegraphics[width=0.33\textwidth]{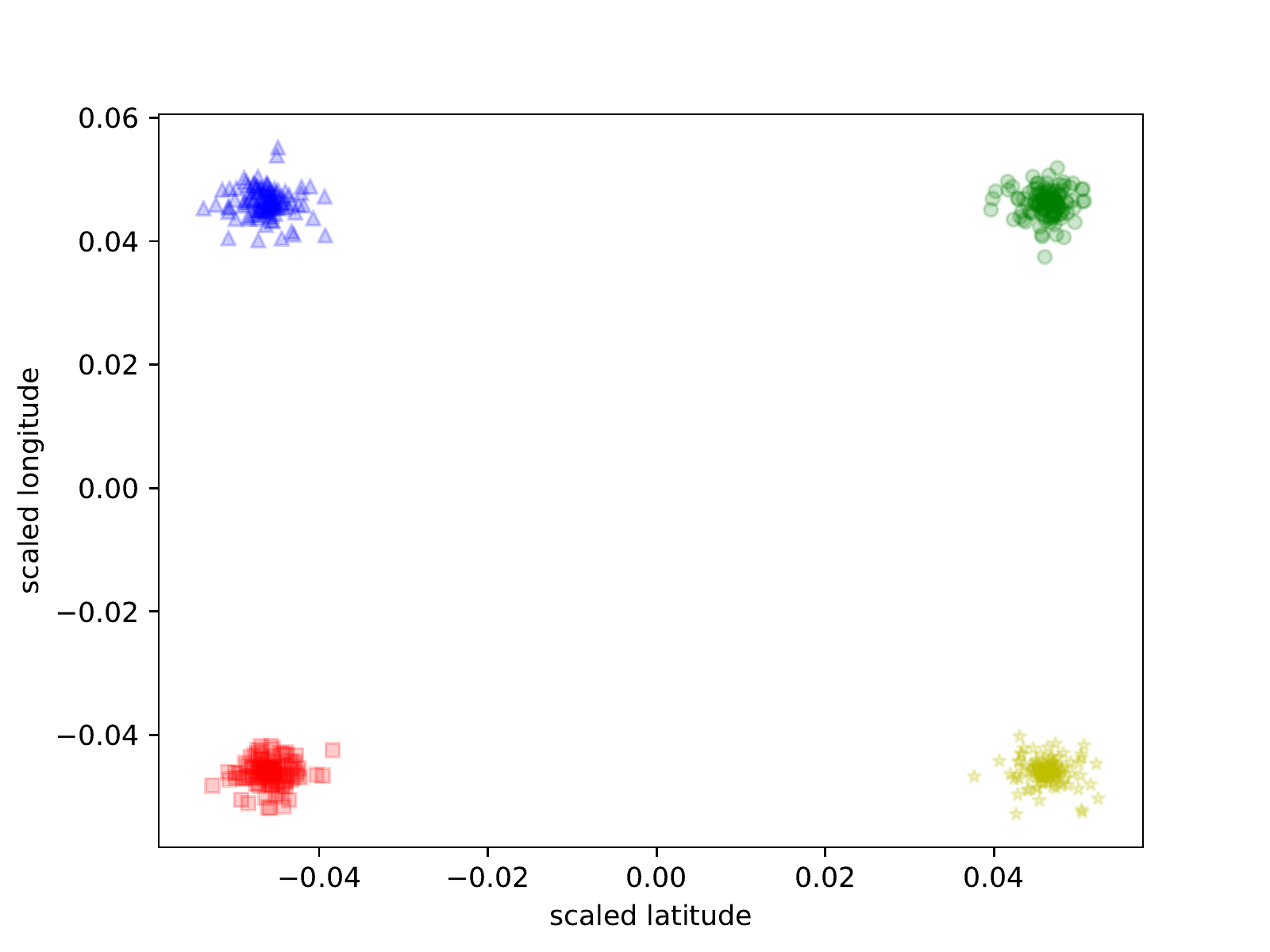}};
   \node[right of=img2,label={[font=\small, black, rotate=0]below:(c)}] (img3) {\includegraphics[width=0.33\textwidth]{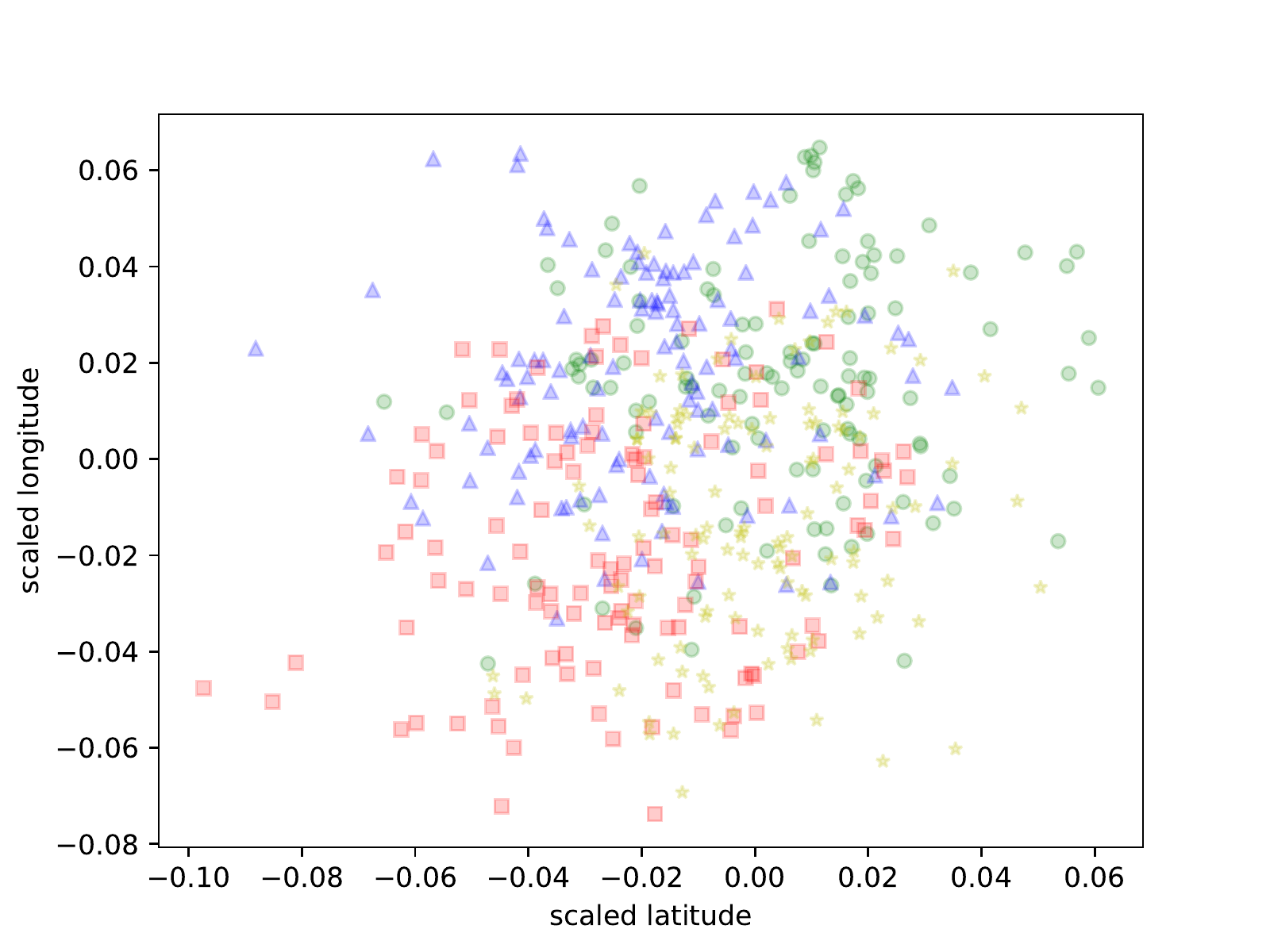}};
   %\node (img1)[label={[font=\small, black, rotate=0]below:(a)}] {\includegraphics[width=0.33\textwidth]{img/new_exp/synthetic_data/Laplace_noise/plotvertices_synthetic_test_set_standalone_classifier_NOISY_input_1_repetition.pdf}};
   %\node[right of=img1,label={[font=\small, black, rotate=0]below:(b)}] (img2) {\includegraphics[width=0.33\textwidth]{img/new_exp/synthetic_data/plotvertices_synthetic_test_set_standalone_classifier_input.pdf}};
   %\node[right of=img2,label={[font=\small, black, rotate=0]below:(c)}] (img3) {\includegraphics[width=0.33\textwidth]{img/new_exp/synthetic_data/G_noise/plot_classifier_no_noise_prediction_G_noise_prediction_test_G_noise.pdf}};
   %\draw [dashed] (-0.02, 0.33) -- (5.12,1.35);
   %\draw [dashed] (-0.02,-0.42) -- (5.12,-1.15);
   %\draw [dashed] (1.0, 0.33) -- (8,1.35);
   %\draw [dashed] (1.0,-0.42) -- (8,-1.15);

   \draw [dashed] (-0.02, 0.33) -- (3.8,1.72);
      \draw [dashed] (-0.02,-0.42) -- (3.8,-1.75);  
      \draw [dashed] (1.0, 0.33) -- (8.4,1.68);
      \draw [dashed] (1.0,-0.42) -- (8.4,-1.75);

   \draw [dashed] (-0.02, 0.33) -- (-0.02,-0.42);
   \draw [dashed] (-0.02,-0.42) -- (1.0,-0.42);
   \draw [dashed] (1.0,-0.42) -- (1.0, 0.33);
   \draw [dashed] (1.0, 0.33) -- (-0.02, 0.33);

   %\draw [dashed](3.99,1.55) -- (9.95,1.55);
   %\draw [dashed](3.99,-1.6) -- (9.95,-1.65);
   %\draw [dashed](7.87,1.55) -- (14.23,1.55);
   %\draw[dashed] (7.89,-1.6) -- (14.23,-1.65);

      \draw [dashed] (3.8,1.72)  -- (11,1.4);
      \draw [dashed] (3.8,-1.75) -- (11,-1.35);
      \draw [dashed] (8.4,1.68)  -- (14.05,1.4);
      \draw [dashed] (8.4,-1.75) -- (14.05,-1.35);

      \draw [dashed] (11,1.4)  -- (11,-1.35);
      \draw [dashed] (11,-1.35) -- (14.05,-1.35);
      \draw [dashed] (14.05,-1.35)  -- (14.05,1.4);
      \draw [dashed] (14.05,1.4) -- (11,1.4);

   %   \draw [dashed] (9.95, 1.7)  -- (8.405,1.7);
   %   \draw [dashed] (9.95, -1.75)  -- (8.405,-1.75);
   %
   %   \draw [dashed] (9.95, 1.7)  -- (8.405,1.7); %double
   %   \draw [dashed] (9.95, -1.75)  -- (8.405,-1.75); %double

   %\node[right=0pt of img3] (img4) {\includegraphics[height=1cm]{example-image}};
   %\node[right=0pt of img4] (img5) {\includegraphics[height=2cm]{example-image}};
   \end{tikzpicture}
   \caption{\footnotesize  Synthetic testing  data. From left to right: Laplace noise, no noise, our noise produced using mutual information. $L=173$m.}
   \label{fig:ts_syn_new_exp}
\end{figure*}

%%%%%%%%%%%%%%%%%%%%%%%%%%%%%%%%%%%%%%
\begin{figure*}[t]%
   {\footnotesize
   %\centering
   \begin{subfigure}{1\columnwidth}\centering
      %\begin{table}[h]
      %\begin{center}
      \begin{tabular}{|c|c|c|c|c|c|c|c|c|}
      \cline{2-9}
      \multicolumn{1}{c|}{} & \multicolumn{8}{c|}{\fix{Number of cells}} \\ 
      \cline{2-9}
      \multicolumn{1}{c|}{} & \multicolumn{2}{c|}{$13\times13$} & \multicolumn{2}{c|}{$65\times65$} & \multicolumn{2}{c|}{$130\times130$}& \multicolumn{2}{c|}{$260\times260$}\\
      %\cline{3-10}
      %\hhline{|~|-|-|-|-|-|-|-|-} 
      \hline
      \fix{Obf} &  \cellcolor{lightgray}Lap & Our &  \cellcolor{lightgray}Lap & Our&  \cellcolor{lightgray}Lap & Our&  \cellcolor{lightgray}Lap & Our\\
      \hline
      %\multirow{4}{*}{\begin{sideways} \fix{Obf.}~ \end{sideways}} 
      $10$ & $\cellcolor{lightgray}0.64$ & $\cellcolor{white}0.74$ &  $\cellcolor{lightgray}0.26$ & $0.45 $ & $\cellcolor{lightgray}0.23$ & $0.43$ &$\cellcolor{lightgray}0.22$ &$ 0.41$  \\  %[5pt]
      %\cline{2-10}
      \hhline{|-|-|-|-|-|-|-|-|-} 
      $100$ &  $\cellcolor{lightgray}0.64$ & $\cellcolor{white}0.74$ &$ \cellcolor{lightgray}0.26$ & $0.45 $& $\cellcolor{lightgray}0.24$ & $0.43$ & $ \cellcolor{lightgray}0.23 $& $0.42$ \\  %[5pt]
      \hhline{|-|-|-|-|-|-|-|-|-} 
      $200$ & $\cellcolor{lightgray}0.64$& $\cellcolor{white}0.74$ & $\cellcolor{lightgray}0.26 $& $0.45 $&  $\cellcolor{lightgray}0.24$ & $0.43$ &  $\cellcolor{lightgray}0.24 $&$ 0.42$ \\  %[5pt]
      \hhline{|-|-|-|-|-|-|-|-|-} 
      $500$ & $ \cellcolor{lightgray}0.64$ &$\cellcolor{white} 0.74 $& $\cellcolor{lightgray}0.26 $& $0.45 $& $ \cellcolor{lightgray}0.24$ &$ 0.43 $&$\cellcolor{lightgray}0.24$&$ 0.42$ \\  %[5pt]
      \hline
      \end{tabular}
      %\end{center}
      \caption{Training data.}
      \label{tab:bayes_tr_syn_new}
   \end{subfigure}\hfill%
   %\end{table}
   %\begin{table}[h]
   \begin{subfigure}{1\columnwidth}\centering
      %\begin{center}
      \begin{tabular}{|c|c|c|c|c|c|c|c|c|}
      \cline{2-9}
      \multicolumn{1}{c|}{} & \multicolumn{8}{c|}{\fix{Number of cells}} \\ 
      \cline{2-9}
      \multicolumn{1}{c|}{} & \multicolumn{2}{c|}{$13\times13$} & \multicolumn{2}{c|}{$65\times65$} & \multicolumn{2}{c|}{$130\times130$}& \multicolumn{2}{c|}{$260\times260$}\\
      %\cline{3-10}
      %\hhline{~|-|-|-|-|-|-|-|-} 
      \hline
      \fix{Obf} &  \cellcolor{lightgray}Lap & Our &  \cellcolor{lightgray}Lap & Our&  \cellcolor{lightgray}Lap & Our&  \cellcolor{lightgray}Lap & Our\\
      \hline
      %\multirow{4}{*}{\begin{sideways} \fix{Obf.}~ \end{sideways}} 
      $10$ &  $\cellcolor{lightgray}0.63$ & $\cellcolor{white}0.74$ &  $\cellcolor{lightgray}0.25$ & $0.44$ & $\cellcolor{lightgray}0.23$ & $0.42$ & $\cellcolor{lightgray}0.19$ & $0.39$  \\  %[5pt]
      \hhline{|-|-|-|-|-|-|-|-|-} 
      $100$ &  $\cellcolor{lightgray}0.64$ & $\cellcolor{white}0.74$ &  $\cellcolor{lightgray}0.26$ & $0.45$& $\cellcolor{lightgray}0.24$ & $0.43$ &  $\cellcolor{lightgray}0.23$ & $0.42$   \\  %[5pt]
      \hhline{|-|-|-|-|-|-|-|-|-} 
      $200$ & $ \cellcolor{lightgray}0.64$ &$\cellcolor{white}0.74$ & $\cellcolor{lightgray}0.26$ & $0.45$ & $\cellcolor{lightgray}0.23$ & $0.43$ & $\cellcolor{lightgray}0.23$ & $0.42$    \\  %[5pt]
      \hhline{|-|-|-|-|-|-|-|-|-} 
      $500$  & $\cellcolor{lightgray}0.64$ & $\cellcolor{white}0.74$ & $\cellcolor{lightgray}0.26$ & $0.45$& $\cellcolor{lightgray}0.23$ & $0.43$ & $\cellcolor{lightgray}0.23$ & $0.42$  \\  %[5pt]
      \hline
      \end{tabular}		
      %\end{center}
      \caption{Testing data}\label{Subfig:Synt_testingdata_b}
      %\end{table}
   \end{subfigure}%
   }
   \caption{\footnotesize Estimation of $B(X\mid Z)$ on the synthetic data for the Laplace and for our mechanisms, with $L=173$m.
   The empirical utility loss for training and testing data is $\approx 170.53$m  -- $172.35$m respectively for the Laplace and $\approx 166.78$m -- $171.50$m for ours.  The optimal mechanism gives $B(X\mid Z)=0.50$, 
   since the utility bound is large enough to let mixing the red and blue points, as well as the green and the yellow, but does not allow more confusion than that. }
   \label{tab:bayes_synt_new}
\end{figure*}

\subsection{Experiment 2: stricter  utility constraint}
\label{synthetic_laplace_injection_loss_constr}

We are now interested in investigating how our method  behaves when a stricter constraint on the utility loss is imposed. 
In order to do so, we run an experiment similar to the one 
in~Section~\ref{sec:sr}. We repeat the same steps but now we set $L$ and the privacy parameter  (and consequently the distortion rate) of the planar Laplace as follows:
\begin{equation}
L = 173\mbox{m} \qquad
\epsilon = \frac{\ln 2}{60}\qquad
\loss[Z\mid W] \approx 173.12\mbox{m} 
\end{equation}

Similarly to the previous section, training $G$ wrt cross entropy fails to converge, producing
generators that achieve no privacy protection.
As a consequence, we only show the results of training $G$ wrt mutual information.

The result of the application of the Laplace mechanism is illustrated in Fig.~\ref{fig:ts_syn_new_exp}(a).
%and~\ref{fig:tr_syn_new_exp}(a)  (training data), respectively.
The empirical distortion is 
%$\approx 170,53$m for the training data, and 
$\approx 172.35$m.
% for the testing data. 

Following the same pattern as in Section~\ref{sec:sr}, we train $\mathit{G}$ and $\mathit{C}$.
The training of $\mathit{G}$ is performed for 30 epochs during each iteration with a batch size of 512 samples and a learning rate of 0.0001. The classifier
 $\mathit{C}$ is trained for 3000 epochs with a batch size of 512 samples and 0.001 as the value for the learning rate during each iteration. 
We are particularly interested in the 24th iteration where  $\mathit{C}$'s performance is degraded by the obfuscation performed by $\mathit{G}$ trained during the previous iteration.
Training $\mathit{C}$ with 32 samples batch size and learning rate set to 0.001 for 100 epochs with the obfuscated data gives the results reported in Table~\ref{tab_synthetic_laplace_results_new_exp_our_noise}. 
\begin{table}[h]
   \begin{center}
   \begin{tabular}{|c|c|c|}
   \hline
   %\textbf{}&\multicolumn{2}{|c|}{\textbf{Table Column Head}} \\
   \textbf{Data}&\textbf{Accuracy}&\textbf{F1\_score} \\
   \hline
   %\cline{1-3} 
   Training data & $\approx 0.55$ & $\approx 0.54$ \\%& \textbf{\textit{Table column subhead}}& \textbf{\textit{Subhead}}& \textbf{\textit{Subhead}} \\
   \hline
   Validation data & $\approx 0.53$ & $\approx 0.53$ \\
   \hline
   Test data & $\approx 0.52$ & $\approx 0.51$ \\
   \hline
   %copy& More table copy$^{\mathrm{a}}$& &  
   %\hline
   %\multicolumn{4}{l}{$^{\mathrm{a}}$Sample of a Table footnote.}
   \end{tabular}
   \caption{\small  Summary of the experiment with our noise.}
   \label{tab_synthetic_laplace_results_new_exp_our_noise}
   \end{center}
\end{table}
In this case, increasing the number of epochs does not 
improve the classification precision and makes $\mathit{C}$ more prone to overfitting.

The obfuscation provided by  $\mathit{G}$ at the 24th iteration produces the distributions on the testing 
%and  training  data  
illustrated in Fig.~\ref{fig:ts_syn_new_exp}(c).
%and \ref{fig:tr_syn_new_exp}(c)  respectively. 
The empirical distortion is 
%$\approx 166.78$m \mbox{(training data)} and  
$\approx 171.50$m.
%~(testing data).
The  estimated   Bayes error for the two mechanisms is reported  in Fig.~\ref{tab:bayes_synt_new}.
%\vspace{-2cm}

% !TEX root = main.tex

\section{Experiments on the Gowalla dataset}	
\label{Govalla_experiments}

In the previous section we saw that our method
behaves as expected in a simple synthetic dataset, producing an obfuscation mechanism
that is close to the optimal one (when $G$ is trained wrt mutual information). We now study the behaviour of our method to
real location data from the Gowalla dataset. Since cross entropy was shown to be unsound,
we only present results using mutual information for training $G$.

\subsubsection*{The dataset}
The dataset consists of data extracted from the \emph{Gowalla} dataset~\cite{Gowalla}, a collection of check-ins made available by the Gowalla location-based social network. 
Among all the provided features, only the users' identifiers (classes), the latitude and longitude of the check-in locations are considered. 
The data are selected as follows:
\begin{enumerate}
\item we consider a squared region centered in 5, Boulevard de S\'ebastopol, Paris, France with 4500m long side;
\item we select the 6 users who checked in the region most frequently, we retain their locations and  discard the rest;
\item we filter the obtained locations to reduce the overlapping of the data belonging to different classes by randomly selecting  
 for each class 82 location samples for training and validation purpose, and 20 samples for the test. 
\end{enumerate}
We obtain $492$ pairs $(\mathit{locations}, \mathit{id})$   to train and validate the model, and $120$ to test it. For each of these, the generator creates 10 pairs with noisy locations using different seeds. 
%Indeed, each $G_i$ is trained to produce a noisy location from a real location and a seed. 
As usual,  $G_0$ does it using the Laplace function, the other $G_i$'s use the mechanism learnt at the previous step $i-1$.
%The triples $(\mathit{locations}, \mathit{id})$ are called samples and they are the input of the generator. 
%Since there are 10 different seeds for each original pair $(\mathit{locations}, \mathit{id})$, 
Thus in total we obtain    4920 pairs for training and 1200 for testing. 
Fig.\ref{fig:ts_real_data}
%and ~\ref{fig:tr_real_data}(b)  
shows the result of the  mechanism applied to the testing data,
%and training data, respectively.
where each color corresponds to   a different user.

\begin{figure*}[t]
   \centering
   \begin{tikzpicture}[every node/.style={node distance=6cm, inner sep=0pt, outer sep=0pt}]
   %\node (img1)[label={[font=\small, black, rotate=0]below:(a)}] {\includegraphics[width=0.33\textwidth]{img/for_catuscia/plot_real_data/real_data_Laplace_noise_test.pdf}};
   %\node[right of=img1, label={[font=\small, black, rotate=0]below:(b)}] (img2) {\includegraphics[width=0.33\textwidth]{img/for_catuscia/plot_real_data/real_data_no_noise_test.pdf}};
   %\node[right of=img2,label={[font=\small, black, rotate=0]below:(c)}] (img3) {\includegraphics[width=0.33\textwidth]{img/for_catuscia/plot_real_data/real_data_our_noise_test.pdf}};
   \node (img1)[label={[font=\small, black, rotate=0]below:(a)}] {\includegraphics[width=0.33\textwidth]{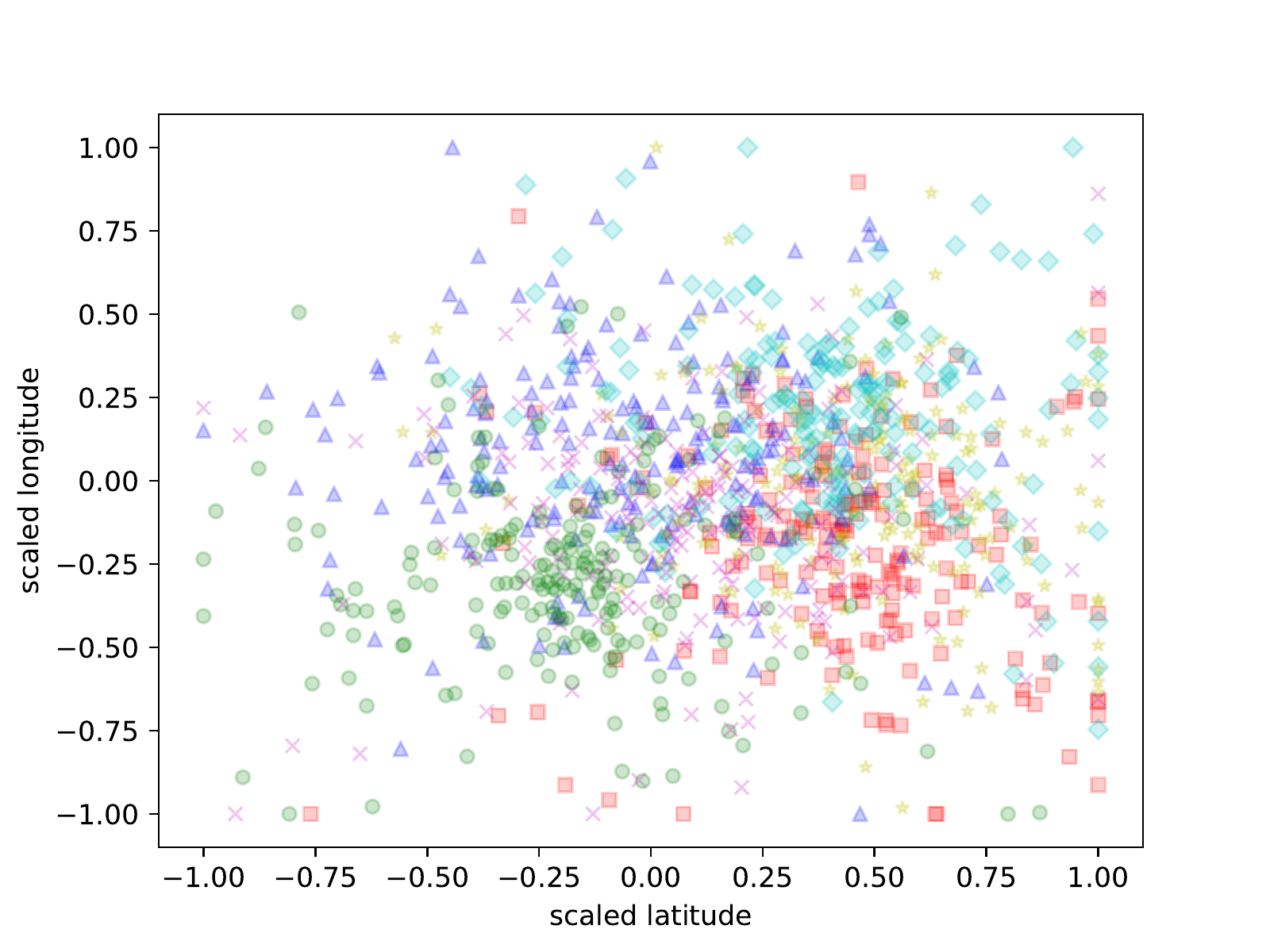}};
   \node[right of=img1, label={[font=\small, black, rotate=0]below:(b)}] (img2) {\includegraphics[width=0.33\textwidth]{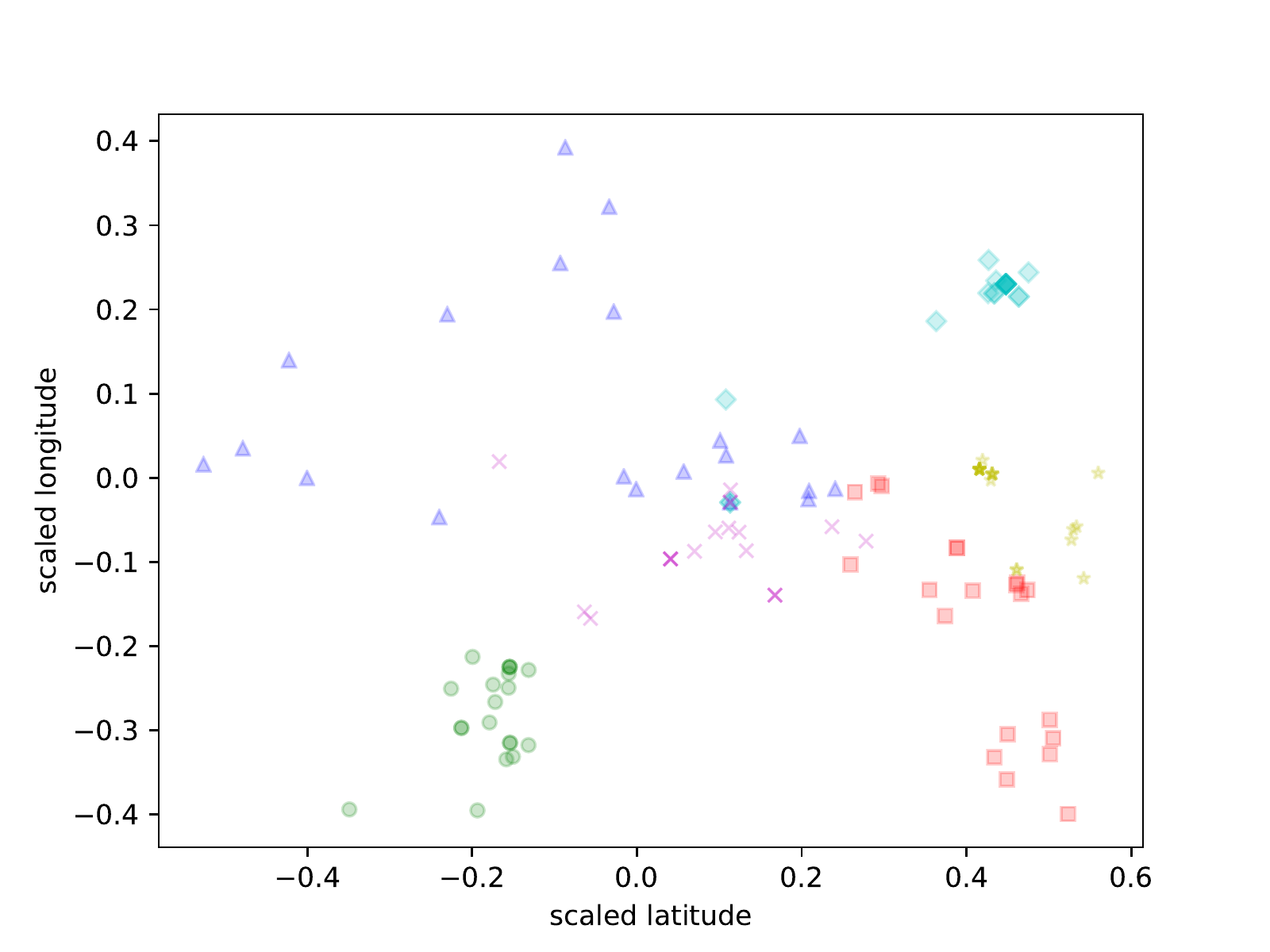}};
   \node[right of=img2,label={[font=\small, black, rotate=0]below:(c)}] (img3) {\includegraphics[width=0.33\textwidth]{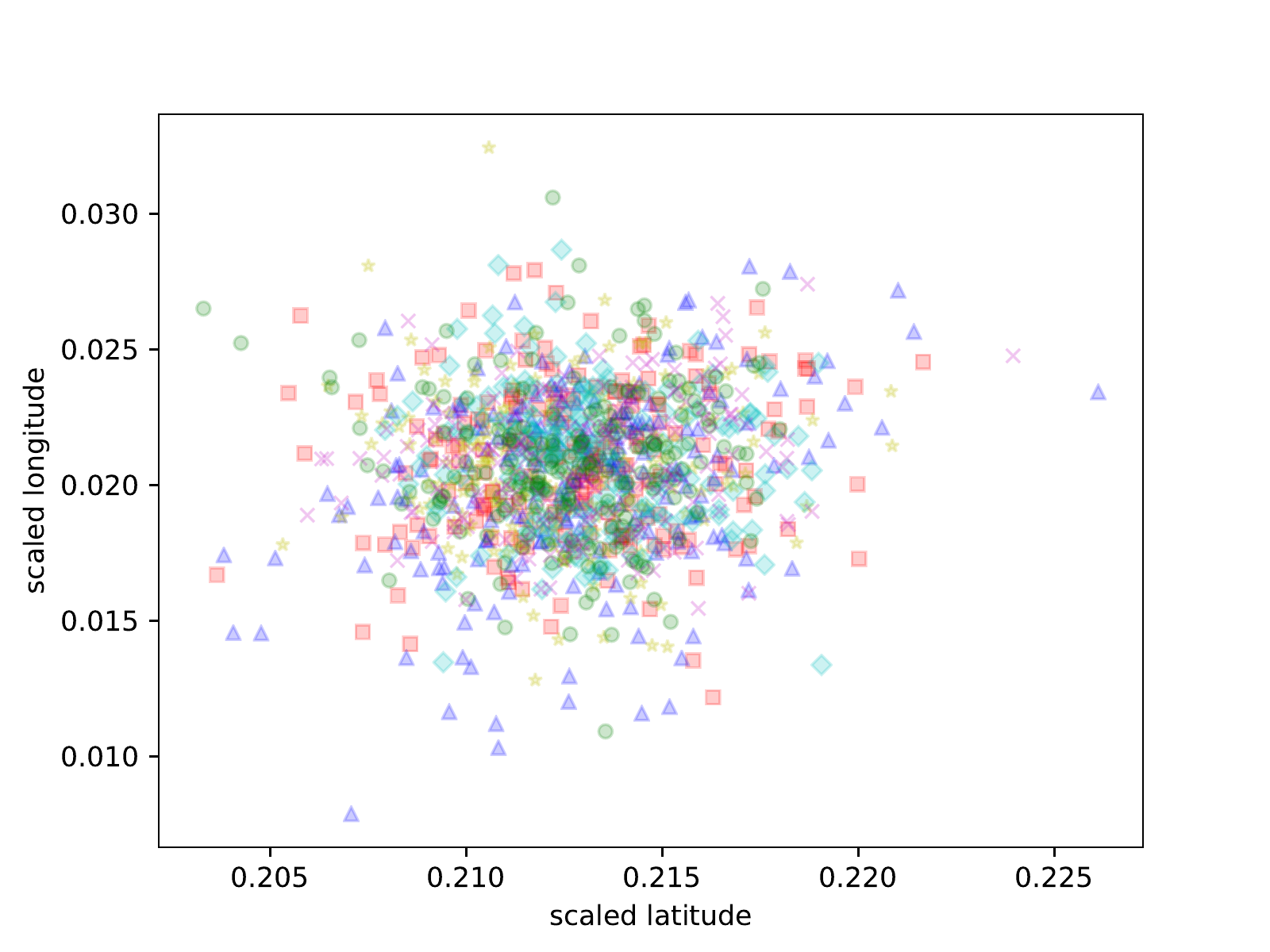}};
   \draw [dashed] (-1.05, 0.59) -- (3.8,1.7);
   \draw [dashed] (-1.05,-0.64) -- (3.85,-1.76);
   \draw [dashed] (1.22, 0.59) -- (8.4,1.7);
   \draw [dashed] (1.22,-0.64) -- (8.4,-1.74);

   \draw [dashed] (-1.05, 0.59) -- (-1.05,-0.64);
   \draw [dashed] (-1.05,-0.64) -- (1.22,-0.64);
   \draw [dashed] (1.22,-0.64) -- (1.22, 0.59);
   \draw [dashed] (1.22, 0.59) -- (-1.05, 0.59);

   \draw [dashed](6.79,0.12) -- (9.8,1.7);
   \draw [dashed](6.79,0.02) -- (9.8,-1.75);
   \draw [dashed](6.87,0.12) -- (14.4,1.7);
   \draw[dashed] (6.87,0.02) -- (14.4,-1.75);

   \draw [dashed](6.79,0.12) -- (6.79,0.02);
   \draw [dashed](6.79,0.02) -- (6.87,0.02);
   \draw [dashed](6.87,0.02) -- (6.87,0.12);
   \draw[dashed] (6.87,0.12) -- (6.79,0.12);
   %\node[right=0pt of img3] (img4) {\includegraphics[height=1cm]{example-image}};
   %\node[right=0pt of img4] (img5) {\includegraphics[height=2cm]{example-image}};
   \end{tikzpicture}
   \caption{\footnotesize Gowalla  testing  data.  From left to right: Laplace noise, no noise, our noise produced using mutual information. $L=1150$m.}
   \label{fig:ts_real_data}
\end{figure*}

\begin{figure*}[t]%
   {\footnotesize
   %\centering
   \begin{subfigure}{1\columnwidth}\centering
   %\begin{table}[h]
   %\begin{center}
   \begin{tabular}{|c|c|c|c|}
   \hline
   \multicolumn{4}{|c|}{\fix{Number of cells}} \\ 
   \hline
   $13\times13$ & $65\times65$ & $130\times130$ & $260\times260$\\%\cline{3-10}
   \hline
   $0.12$ & $0.06$ & $0.04$ & $0.03$\\%\cline{3-10}
   \hline
   \end{tabular}
   %\end{center}
   \caption{Training data.}%Experiments on the original version of the data in RealTR.}
   \label{tab:bayes_tr_real}
   \end{subfigure}\hfill%
   %\end{table}
   %\begin{table}[h]
   \begin{subfigure}{1\columnwidth}\centering
   %\begin{table}[h]
   %\begin{center}
   \begin{tabular}{|c|c|c|c|}
   \hline
   \multicolumn{4}{|c|}{\fix{Number of cells}} \\ 
   \hline
   $13\times13$ & $65\times65$ & $130\times130$ & $260\times260$\\%\cline{3-10}
   \hline
   $0.11$ & $0.04$ & $0.03$ & $0.03$\\%\cline{3-10}
   \hline
   \end{tabular}
   %\end{center}
   \caption{Testing data.}%Experiments on the original version of the data in RealTS.}
   \label{tab:bayes_tr_real}
   \end{subfigure}}		\caption{\footnotesize  Estimation of $B(X\mid Z)$  on the original version of the data from Gowalla.}
   \label{real_bayes_Gowalla}
\end{figure*}

\begin{figure*}[t]%
   {\footnotesize
   %\centering
   \begin{subfigure}{1\columnwidth}\centering
   %\begin{table}[h]
   %\begin{center}
   \begin{tabular}{|c|c|c|c|c|c|c|c|c|}
   \cline{2-9}
   \multicolumn{1}{c|}{} & \multicolumn{8}{c|}{\fix{Number of cells}} \\ 
   \cline{2-9}
   \multicolumn{1}{c|}{} & \multicolumn{2}{c|}{$13\times13$} & \multicolumn{2}{c|}{$65\times65$} & \multicolumn{2}{c|}{$130\times130$}& \multicolumn{2}{c|}{$260\times260$}\\%\cline{3-10}
   %\hhline{|~|-|-|-|-|-|-|-|-} 
   \hline
   \fix{Obf} &  \cellcolor{lightgray}Lap & Our &  \cellcolor{lightgray}Lap & Our&  \cellcolor{lightgray}Lap & Our&  \cellcolor{lightgray}Lap & Our\\
   \hline
   %\multirow{4}{*}{\begin{sideways} \fix{Obf.}~ \end{sideways}} 
   $10$ &  $\cellcolor{lightgray}0.56$ & $0.83$ &  $\cellcolor{lightgray}0.37$ & $0.83$ & $\cellcolor{lightgray}0.19$ & $0.82$ & $\cellcolor{lightgray}0.06$ & $0.80$  \\  %[5pt]
   \hhline{|-|-|-|-|-|-|-|-|-} 
   $100$ & $\cellcolor{lightgray}0.57$ & $0.83$ & $\cellcolor{lightgray}0.53$ & $0.83$& $\cellcolor{lightgray}0.46$ & $0.82$& $\cellcolor{lightgray}0.31$ & $0.81$   \\  %[5pt]
   \hhline{|-|-|-|-|-|-|-|-|-} 
   $200$ & $\cellcolor{lightgray}0.57$ & $0.83$  & $\cellcolor{lightgray}0.55$& $0.83$& $\cellcolor{lightgray}0.50$& $0.82$&$\cellcolor{lightgray}0.40$ & $0.81$   \\  %[5pt]
   \hhline{|-|-|-|-|-|-|-|-|-} 
   $500$ &  $\cellcolor{lightgray}0.57$ & $0.83$  &  $\cellcolor{lightgray}0.56$ & $0.83$ & $\cellcolor{lightgray}0.54$ & $0.82$ & $\cellcolor{lightgray}0.48$ & $0.81$   \\  %[5pt]
   \hline
   \end{tabular}
   %\end{center}
   \caption{Training data.}
   \label{tab:bayes_tr_real}
   \end{subfigure}\hfill%
   %\end{table}
   %\begin{table}[h]
   \begin{subfigure}{1\columnwidth}\centering
   %\begin{center}
   \begin{tabular}{|c|c|c|c|c|c|c|c|c|}
   \cline{2-9}
   \multicolumn{1}{c|}{} & \multicolumn{8}{c|}{\fix{Number of cells}} \\ 
   \cline{2-9}
   \multicolumn{1}{c|}{} & \multicolumn{2}{c|}{$13\times13$} & \multicolumn{2}{c|}{$65\times65$} & \multicolumn{2}{c|}{$130\times130$}& \multicolumn{2}{c|}{$260\times260$}\\%\cline{3-10}
   %\hhline{~|-|-|-|-|-|-|-|-} 
   \hline
   \fix{Obf} &  \cellcolor{lightgray}Lap & Our &  \cellcolor{lightgray}Lap & Our&  \cellcolor{lightgray}Lap & Our&  \cellcolor{lightgray}Lap & Our\\
   \hline
   %\multirow{4}{*}{\begin{sideways} \fix{Obf.}~ \end{sideways}} 
   $10$ &  $\cellcolor{lightgray}0.51$ & $0.83$  &  $\cellcolor{lightgray}0.18$ & $0.82$ & $\cellcolor{lightgray}0.07$ & $0.80$ &$\cellcolor{lightgray}0.01$ &$0.79$   \\  %[5pt]
   \hhline{|-|-|-|-|-|-|-|-|-} 
   $100$ & $\cellcolor{lightgray}0.56$ & $0.83$ & $\cellcolor{lightgray}0.45$ & $0.82$& $\cellcolor{lightgray}0.30$ & $0.81$ &$\cellcolor{lightgray}0.13$& $0.80$   \\  %[5pt]
   \hhline{|-|-|-|-|-|-|-|-|-}  
   $200$ &  $\cellcolor{lightgray}0.56$& $0.83$ &  $\cellcolor{lightgray}0.49$ & $0.82$  & $\cellcolor{lightgray}0.38$ & $0.81$&$\cellcolor{lightgray}0.21$ & $0.80$  \\  %[5pt]
   \hhline{|-|-|-|-|-|-|-|-|-}  
   $500$ &  $\cellcolor{lightgray}0.56$ & $0.83$& $\cellcolor{lightgray}0.53$ & $0.82$ & $\cellcolor{lightgray}0.46$ & $0.81$& $\cellcolor{lightgray}0.33$ & $0.80$   \\  %[5pt]
   \hline
   \end{tabular}		\label{tab:bayes_ts_real}
   %\end{center}
   \caption{Testing data.}
   %\end{table}
   \end{subfigure}%
   }		\caption{\footnotesize Estimation of $B(X\mid Z)$ on the Gowalla data for the Laplace and for our mechanisms, with $L=1150$m.
   The utility loss  for training and testing data  is 
 $\approx 1127.83\mbox{m}-1132.63$m respectively for the Laplace and 
$\approx  961.38\mbox{m}-979.40$m for ours.  The optimal mechanism gives $B(X\mid Z)=1-\nicefrac{1}{6}=0.83$.}
   \label{real_bayes}
\end{figure*}

%%%%%%%%%%%%%%%%%%%%%%%%%%%%%%%%%%%%%%%

%%%%%%%%%%%%%%%%%%%%%%%%%%%%%%%%%%%%%%%%%%%%%%%%%%%%%%%%%%%%%%%

\begin{figure*}[t]
   \centering
   \begin{tikzpicture}[every node/.style={node distance=6cm, inner sep=0pt, outer sep=0pt}]
   %\node (img1)[label={[font=\small, black, rotate=0]below:(a)}] {\includegraphics[width=0.33\textwidth]{img/new_exp/real_data/Laplace_noise/plot_adv_selected_test_df.pdf}};
   %\node[right of=img1,label={[font=\small, black, rotate=0]below:(b)}] (img2) {\includegraphics[width=0.33\textwidth]{img/new_exp/real_data/plot_test_no_noise_df.pdf}};
   %\node[right of=img2,label={[font=\small, black, rotate=0]below:(c)}] (img3) {\includegraphics[width=0.33\textwidth]{img/new_exp/real_data/G_noise/plot_classifier_no_noise_prediction_G_noise_prediction_test_G_noise.pdf}};
   \node (img1)[label={[font=\small, black, rotate=0]below:(a)}] {\includegraphics[width=0.33\textwidth]{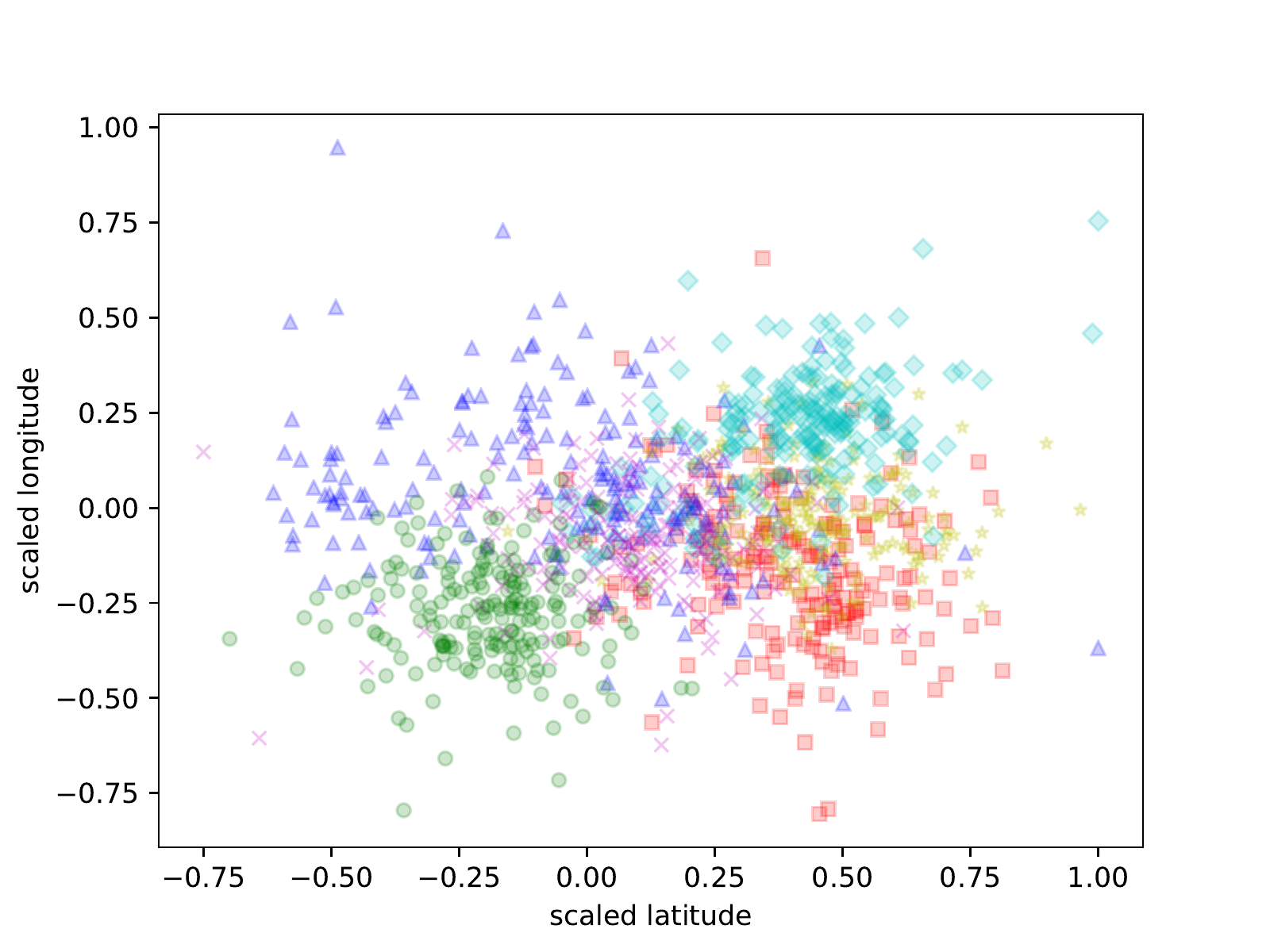}};
   \node[right of=img1,label={[font=\small, black, rotate=0]below:(b)}] (img2) {\includegraphics[width=0.33\textwidth]{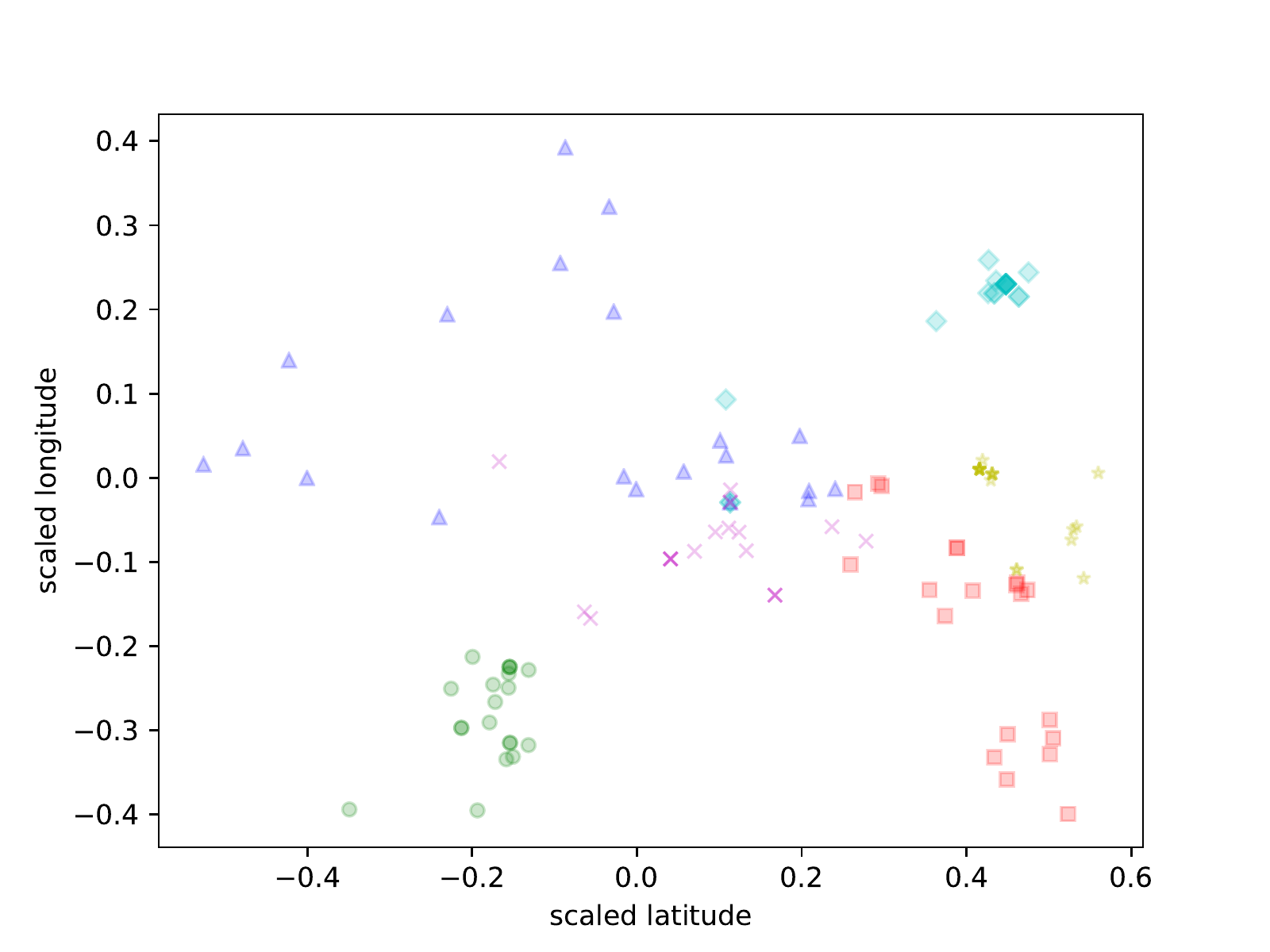}};
   \node[right of=img2,label={[font=\small, black, rotate=0]below:(c)}] (img3) {\includegraphics[width=0.33\textwidth]{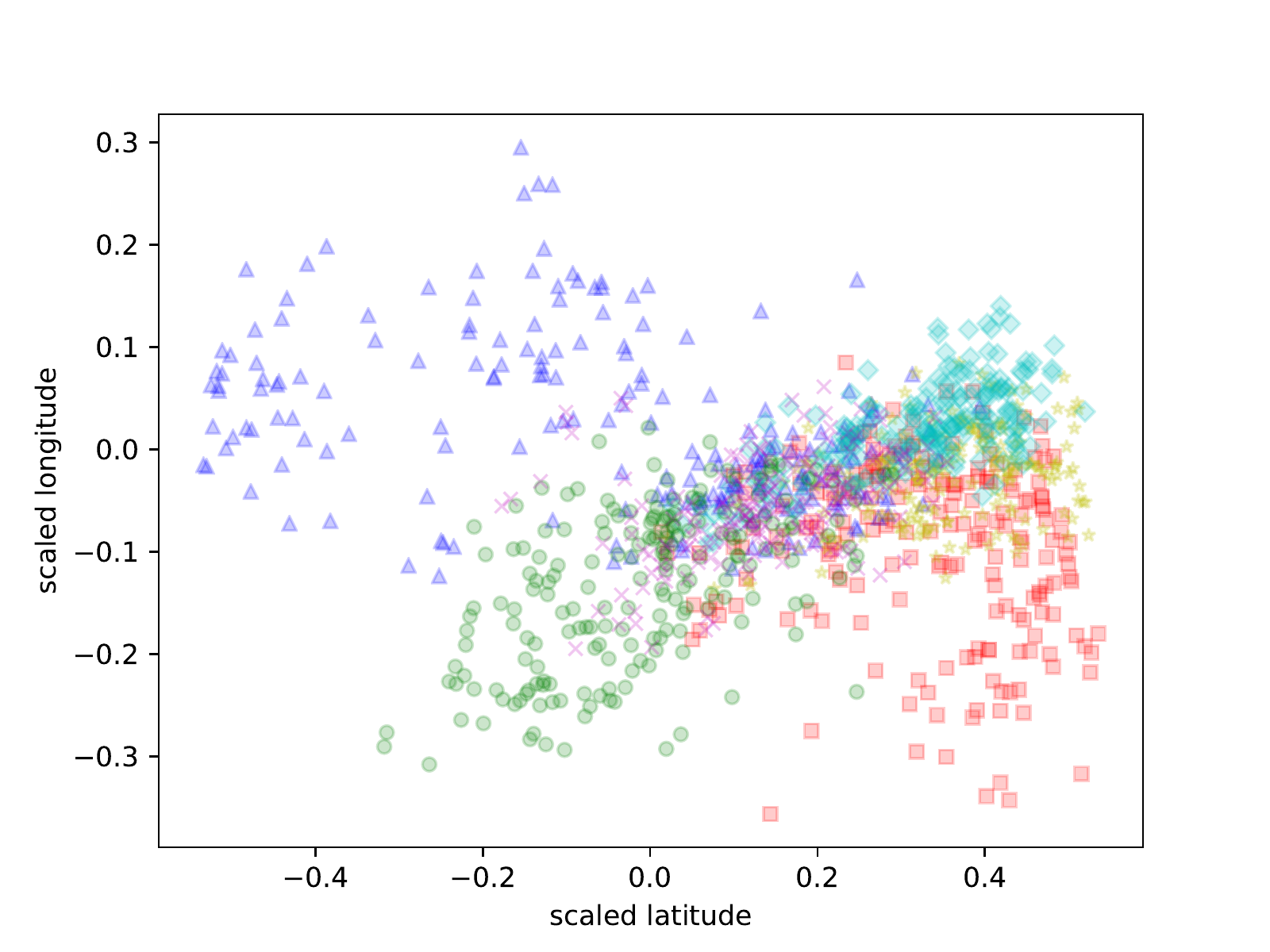}};
   %\draw [dashed] (-1.5, 0.56) -- (3.96,1.55);
   %\draw [dashed] (-1.5,-0.87) -- (3.96,-1.63);
   %\draw [dashed] (1.1, 0.56) -- (8.2,1.55);
   %\draw [dashed] (1.1,-0.87) -- (8.2,-1.63);

      \draw [dashed] (-1.5, 0.56)  -- (3.8,1.7);
      \draw [dashed] (-1.5,-0.87)  -- (3.85,-1.75);
      \draw [dashed] (1.1, 0.56) -- (8.4,1.7);
      \draw [dashed] (1.1,-0.87) -- (8.4,-1.73);

      \draw [dashed] (-1.5, 0.56)  -- (-1.5,-0.87);
      \draw [dashed] (-1.5,-0.87)  -- (1.1,-0.87);
      \draw [dashed] (1.1,-0.87) -- (1.1, 0.56);
      \draw [dashed] (1.1, 0.56) -- (-1.5, 0.56);

   \draw [dashed](3.86,1.27) -- (9.8,1.7);
   \draw [dashed](3.86,-1.42) -- (9.8,-1.75);
   \draw [dashed](8.25,1.27) -- (14.4,1.68);
   \draw[dashed] (8.25,-1.42) -- (14.4,-1.73);

   \draw [dashed](3.86,1.27) -- (3.86,-1.42);
   \draw [dashed](3.86,-1.42) -- (8.25,-1.42) ;
   \draw[dashed] (8.25,-1.42) -- (8.25,1.27);
   \draw [dashed](8.25,1.27) -- (3.86,1.27);

   %\node[right=0pt of img3] (img4) {\includegraphics[height=1cm]{example-image}};
   %\node[right=0pt of img4] (img5) {\includegraphics[height=2cm]{example-image}};
   \end{tikzpicture}
   \caption{\footnotesize  Gowalla testing  data.  From left to right: Laplace noise, no noise, our noise produced using mutual information. $L=518$m.}
   \label{fig:ts_real_new_exp}
\end{figure*}

\subsection{Experiment 3: relaxed  utility constraint}\label{sec:gr}
\label{real_laplace_noise}
In this experiment we study the case of a large upper bound on the utility loss, 
which would potentially allow to achieve the maximum utility. 
We set $L$ and the privacy parameter  (and consequently $\loss[Z\mid W] $) of the planar Laplace as follows:
\begin{equation}
L = 1150\mbox{m} \qquad
\epsilon = \frac{\ln 2}{400}\qquad
\loss[Z\mid W] \approx 1154.15\mbox{m} 
\end{equation}
The results for the Laplace and our method  are illustrated in Fig.\ref{fig:ts_real_data}. %and~\ref{fig:tr_real_data} (training data).  
As we can see, the utility constraint is relaxed enough to allow our method to achieve the maximum privacy. As reported in Fig.\ref{real_bayes}, indeed, the Bayes error is close to that of random guess, namely $1-\nicefrac{1}{6}\approx 0.83$. Fig.~\ref{real_bayes_Gowalla} shows the part of the Bayes error due to the discretization of the domain $\calz$. 
The planar Laplace, on the other hand, confirms the relatively limited level of privacy as observed in the synthetic data.

\subsection{Experiment 4: stricter  utility constraint}\label{sec:gr}
\label{real_laplace_injection_loss_constr}

We consider now a much tighter utility constraint, and we set the parameters of the planar Laplace  as follows:  
\begin{equation}
L = 518\mbox{m} \qquad
\epsilon = \frac{\ln2}{180}\qquad
\loss[Z\mid W] \approx 519.37\mbox{m}  
\end{equation}

The results of the application of the Laplace and of our method to the testing data are shown in
Fig.~\ref{fig:ts_real_new_exp}, 
%(testing data) and~\ref{fig:tr_real_new_exp}
%(training data)
and the Bayes error is reported in
Fig.~\ref{tab:bayes_ts_real_new}.
\begin{figure*}[t]%
   {\footnotesize
   %\centering
   \begin{subfigure}{1\columnwidth}\centering
   %\begin{table}[h]
   %\begin{center}
   \begin{tabular}{|c|c|c|c|c|c|c|c|c|}
   \cline{2-9}
   \multicolumn{1}{c|}{} & \multicolumn{8}{c|}{\fix{Number of cells}} \\ 
   \cline{2-9}
   \multicolumn{1}{c|}{} & \multicolumn{2}{c|}{$13\times13$} & \multicolumn{2}{c|}{$65\times65$} & \multicolumn{2}{c|}{$130\times130$}& \multicolumn{2}{c|}{$260\times260$}\\%\cline{3-10}
   %\hhline{|~|-|-|-|-|-|-|-|-} 
   \hline
   \fix{Obf} &  \cellcolor{lightgray}Lap & Our &  \cellcolor{lightgray}Lap & Our&  \cellcolor{lightgray}Lap & Our&  \cellcolor{lightgray}Lap & Our\\
   \hline
   %\multirow{4}{*}{\begin{sideways} \fix{Obf.}~ \end{sideways}} 
   $10$ &  $\cellcolor{lightgray}0.38$ & $0.50$ &  $\cellcolor{lightgray}0.29$ & $0.42$ & $\cellcolor{lightgray}0.20$ & $0.37$ & $\cellcolor{lightgray}0.27$ & $0.08$  \\  %[5pt]
   \hhline{|-|-|-|-|-|-|-|-|-} 
   $100$ & $\cellcolor{lightgray}0.39$ & $0.51$ & $\cellcolor{lightgray}0.36$ & $0.44$& $\cellcolor{lightgray}0.34$ & $0.43$& $\cellcolor{lightgray}0.27$ & $0.40$   \\  %[5pt]
   \hhline{|-|-|-|-|-|-|-|-|-} 
   $200$ & $\cellcolor{lightgray}0.39$ & $0.51$  & $\cellcolor{lightgray}0.36$& $0.44$& $\cellcolor{lightgray}0.35$& $0.43$&$\cellcolor{lightgray}0.31$ & $0.41$   \\  %[5pt]
   \hhline{|-|-|-|-|-|-|-|-|-} 
   $500$ &  $\cellcolor{lightgray}0.38$ & $0.51$  &  $\cellcolor{lightgray}0.37$ & $0.44$ & $\cellcolor{lightgray}0.36$ & $0.43$ & $\cellcolor{lightgray}0.34$ & $0.42$   \\  %[5pt]
   \hline
   \end{tabular}
   %\end{center}
   \caption{Training data.}
   \label{tab:bayes_tr_real_new}
   \end{subfigure}\hfill%
   %\end{table}
   %\begin{table}[h]
   \begin{subfigure}{1\columnwidth}\centering
   %\begin{center}
   \begin{tabular}{|c|c|c|c|c|c|c|c|c|}
   \cline{2-9}
   \multicolumn{1}{c|}{} & \multicolumn{8}{c|}{\fix{Number of cells}} \\ 
   \cline{2-9}
   \multicolumn{1}{c|}{} & \multicolumn{2}{c|}{$13\times13$} & \multicolumn{2}{c|}{$65\times65$} & \multicolumn{2}{c|}{$130\times130$}& \multicolumn{2}{c|}{$260\times260$}\\%\cline{3-10}
   %\hhline{~|-|-|-|-|-|-|-|-} 
   \hline
   \fix{Obf} &  \cellcolor{lightgray}Lap & Our &  \cellcolor{lightgray}Lap & Our&  \cellcolor{lightgray}Lap & Our&  \cellcolor{lightgray}Lap & Our\\
   \hline
   %\multirow{4}{*}{\begin{sideways} \fix{Obf.}~ \end{sideways}} 
   $10$ &  $\cellcolor{lightgray}0.34$ & $0.47$  &  $\cellcolor{lightgray}0.20$ & $0.36$ & $\cellcolor{lightgray}0.08$ & $0.35$ &$\cellcolor{lightgray}0.03$ &$0.12$   \\  %[5pt]
   \hhline{|-|-|-|-|-|-|-|-|-} 
   $100$ & $\cellcolor{lightgray}0.37$ & $0.49$ & $\cellcolor{lightgray}0.32$ & $0.41$& $\cellcolor{lightgray}0.25$ & $0.38$ &$\cellcolor{lightgray}0.15$& $0.32$   \\  %[5pt]
   \hhline{|-|-|-|-|-|-|-|-|-}  
   $200$ &  $\cellcolor{lightgray}0.37$& $0.48$ &  $\cellcolor{lightgray}0.33$ & $0.41$  & $\cellcolor{lightgray}0.30$ & $0.39$&$\cellcolor{lightgray}0.21$ & $0.35$  \\  %[5pt]
   \hhline{|-|-|-|-|-|-|-|-|-}  
   $500$ &  $\cellcolor{lightgray}0.37$ & $0.49$& $\cellcolor{lightgray}0.35$ & $0.42$ & $\cellcolor{lightgray}0.32$ & $0.40$& $\cellcolor{lightgray}0.28$ & $0.38$   \\  %[5pt]
   \hline
   \end{tabular}		
   %\end{center}
   \caption{Testing data.}%Experiments on the obfuscated versions of the data in RealTS.}
   %\end{table}
   \end{subfigure}%
   }		
   \caption{\footnotesize Estimation of $B(X\mid Z)$ on the Gowalla data for the Laplace and for our mechanisms, with $L=518$m.
   The utility loss 
  fort training and testing data 
is
$\approx 523.40\mbox{m}-535.21$m respectively for the Laplace and 
$\approx 487.34 \mbox{m}-502.89$m for ours.
   We could not compute the optimal mechanism due to the high complexity of the linear program.}
   \label{tab:bayes_ts_real_new}
\end{figure*}
\balance
As the grid becomes finer, both the planar
Laplace and our method become more sensitive to the number of samples, in the
sense that the (approximation of) the Bayes error grows considerably as the
number of samples increases. 
%We think that this is a phenomenon due to the
%way the Bayes error is estimated
This is not surprising: when the cells are small they tend to have
a limited number of hits. Therefore the number of hits whose class is in minority 
(in a given cell), and hence not selected as the best guess, is limited. Note that these
minority hits are those that contribute to the Bayes error.	
% !TEX root = main.tex

\section{Conclusion and future work}				
We have proposed a method based on adversarial nets to generate obfuscation mechanisms with a good tradeoff between privacy and utility. 
The crucial feature of our approach is that the target function to minimize is the mutual information rather than the cross entropy. 
We have applied  our method to the case of location privacy, and experimented with a set of synthetic data and with data from Gowalla. 
We have compared the mechanism obtained via our approach with the planar Laplace, the typical mechanism used for geo-indistinguishability, obtaining favorable results. 

\changed{Although the experiments here were limited to the case of location privacy, 
our setting is very general and can in principle be applied to any 
kind of sensitive and public data with finite domain.  
The same holds for the notion of utility: in this paper we have considered the distortion, i.e. the expected distance between 
the original value $w$ and its corresponding noisy version $z$, but our framework can accommodate any notion  of loss
on which the gradient descent is applicable.
In the future we plan to explore the validity of our approach to other privacy scenarios and  
other loss utility functions. 
We  also plan to study the estimation of mutual information by means of other functions 
which are more suitable for neural networks training in order to reduce the computational burden.
}

\changed{		
We also plan to explore the possibility of using other notions of privacy. 
In particular, we are considering using directly the  Bayes error $B(X,Z)$ in the objective function. 
The main challenge is when the domain of $Z$ is too large, as it makes unfeasible to estimate accurately 
 $B(X,Z)$.   We are currently exploring an approach based on partitioning the domain of $Z$, so to reduce its cardinality. 
We are also interested in considering  the   notion $g$-vulnerability~\cite{Alvim:12:CSF} which generalizes (the converse of) the Bayes error  
to the case in which the adversary's attack is rewarded by a  generic gain functions $g$. 
Our main challenge, however, is to extend our framework to deal with worst-case notions of privacy, such as differential privacy and geo-indistinguishability. 
To this end, we  plan to start with a variant of differential privacy called R\'enyi differential privacy~\cite{Mironov:17:CSF}, which is formulated in terms of divergence, and 
explore the applicability of the learning-based method for estimating  $f$-divergences  proposed in \cite{Rubenstein:19:NIPS}.}

Moreover we plan to enhance the flexibility of the constraint on distortion
(in the loss function), by requiring it to be \emph{per user} rather than
global. More specifically, we aim at producing obfuscation mechanisms that
satisfy constraints stating that the expected displacement \emph{for each
user} is at most up to a certain threshold. The motivation is that different
users may have different requirements. Another potential application is to
encompass a notion of \emph{fairness}, that can be obtained by requiring that
the threshold is the same for everybody.

\subsection{Acknowledgement}{This research was supported by DATAIA Convergence Institute as part of the ``Programme d'Investissement d'Avenir'' (ANR-17-CONV-0003), operated by Inria and CentraleSup\'elec. It was also supported by the ANR project REPAS, and by the Inria/DRI project LOGIS. The work of Palamidessi was supported by the project HYPATIA, funded by  the European Research Council (ERC) under the European Union's Horizon 2020 research and innovation program, grant agreement n. 835294.}

\bibliographystyle{IEEEtran}
\bibliography{../short}

% !TEX root = main.tex
\section{Appendix}

\subsection{Proofs of the results in the paper}
\ConvexityI*
\begin{proof}
Let us recall that 
\begin{equation}
X \leftrightarrow W \leftrightarrow Z \leftrightarrow Y.
\end{equation}
represents a Markov chain where:
\begin{itemize}
\item the relation between the two random variables $X$ and $W$ is defined by the
data distribution $P_{X,W}$,
\item the relation between $Z$ and $Y$ depends only on the chosen classifier according to $P_{Y|Z}$, 
\item the relation between $Z$ and $W$ can be described by the variable $P_{Z|W}$ 
\end{itemize}
If we consider $X$ as the secret input and $Y$ as the observable output of a stochastic channel,
the mutual information between the two random variable can be expressed as
\begin{equation}
I(X;Y) = g(P_{Y|X}).
\end{equation} 
We know from \cite{Cover:06:BOOK} that $g(\cdot)$ is a convex function wrt $P_{Y|X}$. 
We can express $P_{Y|X}$ as:
\begin{equation}
P_{Y|X}(y|x) = \frac{\sum_{zw}P_{X,W}(x, w)P_{Z|W}(z| w)P_{Y|Z}(y|z)}{\sum_{w}P_{X,W}(x, w)}.
\label{eq:channel_matrix}
\end{equation}
Eq.~\eqref{eq:channel_matrix} represents a linear function of the  variable $P_{Z|W}$ (all the other probabilities are constant). Hence $f(P_{Z|W})=g(h(P_{Z|W}))$ where $g(\cdot)$ is convex and 
$h(\cdot)$ is linear. The composition of a convex function with a linear one is a convex function and this concludes the proof.
\end{proof}

\equilibriumI*
\begin{proof}
%Let us prove that $I(X;Z)$ is an upper bound for $\max_CI(X;Y)$. 
Given that~\cref{eq:markov}
represents a Markov chain, 
$
X \leftrightarrow Z \leftrightarrow Y 
$
represents one as well.
From the data processing inequality it follows that: 
\begin{equation}
I(X;Y)\leq I(X;Z)\, .
\end{equation}
Hence we have:
\begin{equation}
\max_CI(X;Y)\leq  I(X;Z)\, ,
\end{equation}
and therefore:
\begin{equation}
\min_G\max_CI(X;Y)\leq \min_GI(X;Z)\, .
\end{equation}
%thus concluding the proof.
\end{proof}

\vspace{0.5cm}
\LowerBound*
\renewcommand*{\arraystretch}{1.8}
\begin{proof}
%\[
%\begin{array}{rcl}
%B(X\mid Z)&=& \sum_z P_Z(z) (1-\max_x P_{X|Z} (x\mid z))\\ 
%&=& 1-\sum_z P_Z(z)  \max_x P_{X|Z} (x\mid z)\\
%&=& 1-\sum_z P_Z(z)   \sum_y P_{Y|Z}(y|z) \max_x P_{X|Z} (x\mid z)\\
%&=& 1-\sum_y P_Y(y)   \sum_z P_{Z|Y}(z|y) \max_x P_{X|Z} (x\mid z)\\
%&\leq& 1-\sum_y P_Y(y)   \max_x  \sum_z P_{Z|Y}(z|y)P_{X|Z} (x\mid z)\\
%&=& 1-\sum_y P_Y(y)   \max_x  P_{X|Y} (x\mid y)\\
%&=& B(X\mid Y)
%\end{array}
\small
\begin{align*}
&B(X\mid Z)=\\
& \sum_z P_Z(z) (1-\max_x P_{X|Z} (x\mid z))\\ 
&= 1-\sum_z P_Z(z)  \max_x P_{X|Z} (x\mid z)\\
%&= 1-\sum_z P_Z(z)   \sum_y P_{Y|Z}(y|z) \max_x P_{X|Z} (x\mid z)\\
&= 1-\sum_y P_Y(y)   \sum_z P_{Z|Y}(z|y) \max_x P_{X|Z} (x\mid z)\\
&\leq 1-\sum_y P_Y(y)   \max_x  \sum_z P_{Z|Y}(z|y)P_{X|Z} (x\mid z)\\
&= 1-\sum_y P_Y(y)   \max_x  P_{X|Y} (x\mid y)\\
&= B(X\mid Y)\\
\end{align*}\\[-5ex]
%\]
\end{proof}
\renewcommand*{\arraystretch}{1}

\balance

\vspace{0.5cm}
\propMICE*
\begin{proof}
$P_{XW}$ is fixed, and   therefore $H(X)$ is fixed as well. 
Hence the goal of $C$ of maximizing $I(X;Y)$ reduces to maximizing $-H(X|Y)$. 
Consider two mechanisms, $P_{Z_1| W}$  and $P_{Z_2| W}$, and the distributions induced on  $X$ 
by $Z_1$ and $Z_2$ respectively, namely $P_{X| Z_1}$ and $P_{X| Z_2}$.
Consider the predictions $P_{Y_1| Z_1}$ and $P_{Y_1| Z_2}$ that $C$ obtains by 
minimizing  the cross entropy with  $P_{X| Z_1}$ and $P_{X| Z_2}$ respectively. 

It is well known that $\argmin_Q\CE(P,Q) = P$, hence we have $P_{Y_1| Z_1} =  P_{X| Z_1}$ and 
$P_{Y_2| Z_2} =  P_{X| Z_2}$. (Note that $X$, $Y_1$ and $Y_2$ all have the same domain $\calx$.)
Hence, taking into account that $X\leftrightarrow Z_1 \leftrightarrow Y_1$ and $X\leftrightarrow Z_2 \leftrightarrow Y_2$ (i.e., they are Markov chains),  we have:
\begin{equation*}\label{eq:conditional_ce}
\scriptsize
\begin{array}{c}
- H(X|Y_1)  \leq   - H(X|Y_2) \\ 
\\[-1ex]
%\end{array}
%\end{equation*}
%\begin{equation*} 
%\scriptsize
%\begin{array}{c}
 \mbox{iff}\\ 
\\[-1ex]
  \sum_z P_{Z_1}(z) \sum_{xy} \, P_{X|Z_1=z}(x|z)P_{Y_1|Z_1=z}(y|z)  \log P_{Y_1|Z_1=z}(y|z) \\ 
 \leq \\
 \sum_z P_{Z_2}(z) \sum_{xy} \, P_{X|Z_2=z}(x|z)P_{Y_2|Z_2=z}(y|z) \log P_{Y_2|Z_2=z}(y|z)\\ 
\\[-1ex]
 \mbox{iff}\\ 
\\[-1ex]
 \sum_z P_{Z_1}(z) \sum_{xy} \, P_{X|Z_1=z}(x|z)P_{X|Z_1=z}(y|z) \log P_{X|Z_1=z}(x|z)\\  
   \leq \\
  \sum_z P_{Z_2}(z) \sum_{xy} \, P_{X|Z_2=z}(x|z)P_{X|Z_2=z}(y|z)  \log P_{X|Z_2=z}(x|z)\\ 
\\[-1ex]
\mbox{iff}\\ 
\\[-1ex]
- H(X|Z_1 )  \leq   - H(X|Z_2).   
\end{array}
\end{equation*}
Finally, observe that 
\begin{equation}\label{eq:conditional_ce}
\scriptsize
-H(X|Z_1 ) \,\leq \, -H(X|Z_2) \quad \mbox{implies} \quad \max_{P_{Y_1|Z}} I(X; Y_1) \, \leq \, \max_{P_{Y_2|Z}}  I(X; Y_2) 
\end{equation}
and recall that $P_{Y_i|Z}$ is the prediction produced by $C$. 
\end{proof}

\changed{
\propCEB*
\begin{proof}
Let $P_{Y\mid Z} = \argmin_\mathit{C} \; \CE(X,Y)$ and let $f^*$ be defined as in \eqref{eqn:bestprediction}. We note that, for every $z\in\calz$:
\renewcommand*{\arraystretch}{1.8}
\[
\begin{array}{l}
\sum_{x} P_{X\mid Z}(x|z) \overline{\mathbb{1}}_{f^*}(x,z)\\ 
\begin{array}{cl}
=& {\displaystyle \sum_{x\neq f^*(z)}  P_{X|Z} (x,z)}\\  
=& {\displaystyle \sum_{x\neq\argmax_y P_{Y |Z} (y|z)}  P_{X|Z} (x,z)}\\  
=& 1- {\displaystyle \sum_{x=\argmax_y P_{Y |Z} (y|z)}  P_{X|Z} (x,z)}\\  
%=&  1-P_{X|Z} (\argmax_y P_{Y |Z} (y|z)\mid z))\\  
=&  1-P_{X|Z} (\argmax_x P_{X |Z} (x|z)\mid z))\\  
=& 1- \max_x P_{X|Z} (x\mid z))\\ 
\end{array}
\end{array}
\]
where the first equality is due to the definition of $\overline{\mathbb{1}}_{f^*}$, the second one is due to the definition of $f^*$, and the last but one follows from the fact that 
$P_{Y\mid Z} = \argmin_\mathit{C} \; \CE(X,Y)$ and therefore, by \autoref{prop:propMICE}, $P_{Y\mid Z} = P_{X\mid Z}$. 
Hence, we have:
{\small
\[
\begin{array}{l}
{R(f^*)}\\ 
\begin{array}{cl}
=& \sum_{xz} P_{X,Z}(x,z) \overline{\mathbb{1}}_f^*(x,z)\\ 
=& \sum_{xz} P_Z(z)   P_{X|Z}(x|z) \overline{\mathbb{1}}_f^*(x,z)\\  
=& \sum_{z} P_Z(z)   \sum_x P_{X|Z}(x|z) \overline{\mathbb{1}}_f^*(x,z)\\  
=& \sum_z P_Z(z)   (1-\max_x P_{X|Z} (x\mid z))\\
=& B(X\mid Z)  \qquad \mbox{(cfr. \autoref{def:BayesError})}
\end{array}
\end{array}
\]
}
\renewcommand*{\arraystretch}{1}
\end{proof}
}

\balance

\clearpage
\MakeEndNotes
\onecolumn
\MarkupsHowto % prints markups help page

\end{document}